\documentclass{article}

\PassOptionsToPackage{numbers, compress}{natbib}

\usepackage[final]{neurips_2024}

\usepackage[utf8]{inputenc} 
\usepackage[T1]{fontenc}    
\usepackage{hyperref}       
\usepackage{url}            
\usepackage{booktabs}       
\usepackage{amsfonts}       
\usepackage{amsmath,mathtools,amsthm}	
\usepackage{amssymb,dsfont,bbm}	
\usepackage{nicefrac}       
\usepackage{microtype}      
\usepackage{xcolor}         
\usepackage{tikz-cd}        

\usepackage{xcolor}
\usepackage{srcltx}
\usepackage{etoolbox}
\usepackage{nameref}
\usepackage{comment}
\usepackage{mathrsfs}
\usepackage{enumerate}
\usepackage[shortlabels]{enumitem}
\usepackage{amsfonts} 
\usepackage{nicefrac} 
\usepackage{mathtools}
\usepackage{dsfont,mathrsfs}
\usepackage{cleveref}
\usepackage{xargs}
\usepackage{multirow}
\usepackage{bm}
\usepackage{subfigure}

\usepackage{autonum}
\usepackage{upgreek}

\newtheorem{assum}{A\hspace{-2pt}}

\newtheorem{theorem}{Theorem}
\crefname{theorem}{theorem}{Theorems}
\Crefname{theorem}{Theorem}{Theorems}

\newtheorem{lemma}{Lemma}
\crefname{lemma}{lemma}{lemmas}
\Crefname{lemma}{Lemma}{Lemmas}

\newtheorem{remark}{Remark}
\crefname{remark}{remark}{remarks}
\Crefname{remark}{Remark}{Remarks}

\newtheorem{corollary}{Corollary}
\crefname{corollary}{corollary}{corollaries}
\Crefname{corollary}{Corollary}{Corollaries}

\newtheorem{proposition}{Proposition}
\crefname{proposition}{proposition}{propositions}
\Crefname{proposition}{Proposition}{Propositions}

\crefname{definition}{definition}{definitions}
\Crefname{Definition}{Definition}{Definitions}

\crefname{example}{example}{examples}
\Crefname{Example}{Example}{Examples}

\crefname{figure}{figure}{figures}
\Crefname{Figure}{Figure}{Figures}

\crefname{table}{table}{tables}
\Crefname{Table}{Table}{Tables}

\crefname{assum}{A\hspace{-2pt}}{A\hspace{-2pt}}
\crefname{assumb}{B\hspace{-2pt}}{B\hspace{-2pt}}
\crefname{assumUGE}{UGE\hspace{-1pt}}{UGE\hspace{-1pt}}
\crefname{assumID}{IND\hspace{-1pt}}{IND\hspace{-1pt}}
\crefname{assumUE}{UE\hspace{-1pt}}{UE\hspace{-1pt}}
\crefname{assumSUP}{M\hspace{-1pt}}{M\hspace{-1pt}}

\newlist{renumerate}{enumerate}{3}
\setlist[renumerate]{wide, labelwidth=!, labelindent=0pt,label=(\roman*)}

\newlist{aenumerate}{enumerate}{3}
\setlist[aenumerate]{wide, labelwidth=!, labelindent=0pt,label=(\arabic*)}

\newlist{aaenumerate}{enumerate}{3}
\setlist[aaenumerate]{wide, labelwidth=!, labelindent=0pt,label=(\alph*)}

\newlist{aenumerateSpace}{enumerate}{3}
\setlist[aenumerateSpace]{wide, labelwidth=!,label=(\arabic*)}

\newlist{benumerate}{enumerate}{3}
\setlist[benumerate]{wide, labelwidth=!, labelindent=0pt,label=$\bullet$}
\def\supconsteps{\supnorm{\funnoisew}}

\newcommand{\PE}{\mathbb{E}}
\newcommand{\var}{\operatorname{Var}}
\newcommand{\PP}{\mathbb{P}}
\newcommandx{\genericb}[1][1=]{b_{#1}}
\newcommandx{\Constros}[1][1=]{\operatorname{C}_{\operatorname{Ros},#1}}
\newcommandx{\Constburk}[1][1=]{\operatorname{C}_{\operatorname{Burk}}}
\newcommandx{\driftW}[1][1=]{W_{#1}}

\newcommandx{\metricd}[1][1=]{\mathsf{d}_{#1}}

\newcommandx\invmeasure[1][1=]{\Pi_{#1}}
\newcommandx{\PPjoint}[1][1=]{\PP^{\MKjoint[#1]}}
\newcommandx{\PEjoint}[1][1=]{\PE^{\MKjoint[#1]}}
\newcommandx{\PEMID}[1][1=\alpha]{\PE^{\MK[#1]}}
\newcommandx{\PPMID}[1][1=\alpha]{\PP^{\MK[#1]}}

\newcommand{\supnorm}[1]{\norm{ #1 }[\infty]}
\newcommandx{\MKjoint}[1][1=]{\bar{\operatorname{P}}_{#1}}
\newcommandx\costw[1][1=]{\mathsf{c}_{#1}}

\newcommandx\Intergrdist[1][1=]{\mathbb{M}_{1}(#1)}

\newcommandx{\mmarkov}[1][1=0]{m^{(\Markov)}_{#1}}

\def\Conv{\operatorname{Conv}}
\def\seta{\mathcal{A}}

\def\H{\mathcal{H}}

\def\Zset{\mathsf{Z}}

\def\rset{\mathbb{R}}

\def\nset{\ensuremath{\mathbb{N}}}
\def\nsets{\ensuremath{\mathbb{N}^*}}

\newcommand{\msi}{\mathsf{I}}

\newcommand{\Mat}[1]{{\bf{#1}}}
\def\MatB{B}

\def\covfeat{\Sigma_\varphi}
\renewcommand{\S}{\mathcal{S}}
\newcommand{\A}{\mathcal{A}}

\def\PMDP{\MKQ}

\newcommand{\bConst}[1]{\operatorname{C}_{{\bf #1}}}

\newcommandx\sequence[4][2=,3=,4=]
{\ifthenelse{\equal{#3}{}}{\ensuremath{\{ #1_{#2 #4}\}}}{\ensuremath{\{ #1_{#2 #4} \}_{#2 \in #3}}}}

\newcommandx\sequenceD[2][2=]
{\ifthenelse{\equal{#2}{}}{\ensuremath{\{ #1\}}}{\ensuremath{\{ #1\!~:\!~#2  \}}}}
\newcommandx\sequenceDouble[4][3=,4=]
{\ifthenelse{\equal{#3}{}}{\ensuremath{\{ (#1_{#3},#2_{#3})\}}}{\ensuremath{\{ (#1_{#3},#2_{#3})\}_{#3 \in #4}}}}

\newcommandx{\sequencen}[2][2=n\in\nset]{\ensuremath{\{ #1, \eqsp #2 \}}}
\newcommandx\sequencens[2][2=n]
{\ensuremath{\{ #1_{#2} \!~:\!~#2\in\nsets\}}}
\newcommandx\sequencet[4]
{\ensuremath{\{ #1{#2_{#3}} \, : \, \eqsp #3 \in #4 \}}}
\def\PE{\mathbb{E}}
\def\P{\mathbb{P}}
\def\ProdB{\Gamma}

\newcommandx{\PVar}[1][1=]{\ensuremath{\operatorname{Var}_{#1}}}
\newcommandx\conststab[1][1=p]{\varkappa_{#1}}

\newcommand{\ConstPR}[1]{\mathsf{C}_{#1}}

\def\noisecov{\Sigma_\varepsilon}

\def\metrics{\mathsf{d}_{\S}}

\newcommandx{\MK}[1][1=\alpha]{\mathrm{P}_{#1}}
\newcommandx\MKK[1][1=\alpha]{\mathrm{K}_{#1}}
\def\MKQ{\mathrm{P}}

\newcommandx{\PEtilde}[1][1=]{\PE^{\mathrm{K}_{#1}}}
\newcommandx{\PPtilde}[1][1=]{\PP^{\mathrm{K}_{#1}}}

\newcommandx{\norm}[2][2=]{\Vert#1 \Vert_{{#2}}}
\newcommandx{\normLigne}[2][2=]{\Vert#1 \Vert_{{#2}}}
\newcommandx{\normLine}[2][2=]{\Vert#1 \Vert_{{#2}}}
\newcommandx{\normop}[2][2=]{\Vert{#1}\Vert_{{#2}}}
\newcommandx{\normopLigne}[2][2=]{\Vert{#1}\Vert_{{#2}}}
\newcommandx{\normopLine}[2][2=]{\Vert{#1}\Vert_{{#2}}}
\newcommandx{\osc}[2][1=]{\mathrm{osc}_{#1}(#2)}

\newcommandx{\normlip}[2][2=\operatorname{Lip}]{\Vert#1 \Vert_{{#2}}}
\newcommand{\lip}{\operatorname{L}}
\newcommandx{\lipspace}[1]{\lip_{#1}}

\newcommandx{\CPP}[3][1=]
{\ifthenelse{\equal{#1}{}}{{\mathbb P}\left(\left. #2 \, \right| #3 \right)}{{\mathbb P}_{#1}\left(\left. #2 \, \right | #3 \right)}}
\newcommandx{\CPPtilde}[3][1=]
{\ifthenelse{\equal{#1}{}}{{\tilde{\mathbb P}}\left(\left. #2 \, \right| #3 \right)}{{\tilde{\mathbb P}}_{#1}\left(\left. #2 \, \right | #3 \right)}}

\def\iid{i.i.d.}

\newcommandx{\as}[1][1=\PP]{\ensuremath{#1\, -\mathrm{a.s.}}}

\newcommand{\ie}{i.e.}

\newcommand{\eqsp}{\;}

\newcommand{\Id}{\mathrm{I}}
\def\prtheta{\bar{\theta}}

\def\utheta{\tilde{\theta}^{\sf (tr)}}
\def\vtheta{\tilde{\theta}^{\sf (fl)}}

\newcommandx{\boundmetric}[1][1=]{\kappa_{\MKK[#1]}}

\newcommand{\Jnalpha}[2]{J_{#1}^{(#2)}}
\newcommand{\Hnalpha}[2]{H_{#1}^{(#2)}}

\newcommandx{\Nnorm}[2][1=V]{[ #2]_{#1}}
\newcommandx{\lipnorm}[2][1=g]{[ #1]_{#2}}

\newcommandx{\CPE}[3][1=]{{\mathbb E}^{#3}_{#1}\left[#2\right]}
\newcommandx{\CPEext}[3][1=]{\tilde{\mathbb E}^{#3}_{#1}\left[#2\right]}
\newcommandx{\CPEtilde}[3][1=]{{\tilde{\mathbb E}}^{#3}_{#1}\left[#2\right]}
\newcommandx{\CPEs}[3][1=]{{\mathbb E}^{#3}_{#1}[#2]}

\def\thetalim{\theta^\star}

\def\trace{\operatorname{Tr}}

\newcommand{\rme}{\mathrm{e}}
\newcommand{\rmd}{\mathrm{d}}
\def\funcAw{\mathbf{A}}

\newcommand{\funcA}[1]{\funcAw(#1)}

\def\funcbw{\mathbf{b}}
\newcommand{\funcb}[1]{\funcbw(#1)}
\newcommandx{\zmfuncA}[2][1=]{\tilde{\funcAw}^{#1}(#2)}
\newcommandx{\zmfuncAw}[1][1=]{\tilde{\funcAw}_{#1}}
\newcommandx{\zmfuncb}[2][1=]{\tilde{\funcbw}^{#1}(#2)}
\def\funnoisew{\varepsilon}
\newcommand{\funcnoise}[1]{\funnoisew(#1)}

\newcommandx{\funcct}[2][1=]{\funcctilde^{#1}(#2)}

\def\qcond{\kappa_{Q}}

\def\State{Z}

\newcommand{\1}{\boldsymbol{1}}

\newcommandx{\CovC}[1][1=u]{\operatorname{C}_{#1}}

\usepackage[disable]{todonotes}

\def\msz{\mathsf{Z}}
\def\mcz{\mathcal{Z}}
\newcommand\borel[1]{\mathcal{B}(#1)}

\def\plusinfty{+\infty}

\DeclareMathAlphabet{\mathpzc}{OT1}{pzc}{m}{it}

\def\lyapW{\mathpzc{W}}

\newcommandx{\bias}[1][1=\alpha]{\operatorname{B}_{#1}}

\newcommandx\probaMarkovTilde[2][2=]
{\ifthenelse{\equal{#2}{}}{{\widetilde{\mathbb{P}}_{#1}}}{\widetilde{\mathbb{P}}_{#1}\left[ #2\right]}}

\def\mcf{\mathcal{F}}

\newcommand{\indi}[1]{\1_{#1}}

\def\bA{\bar{\mathbf{A}}}

\def\X{{\bf X}}
\def\Y{{\bf Y}}

\def\thetas{\thetalim}

\def\Am{{\bf A}}
\def\bm{{\bf b}}
\def\funcctilde{\tilde{c}_u}
\def\Remainder{\Delta}

\def\barb{\bar{\mathbf{b}}}
\newcommandx{\driftb}[1][1=p]{\bar{b}_{#1}}

\def\barA{\bar{A}}

\def\Zbf{\mathbf{Z}}

\def\eps{\varepsilon}

\newcommandx{\boldb}[1][1={q}]{\mathsf{b}_{#1}}

\newcommandx{\ConstGW}[1][1={n,\lyapW}]{\operatorname{G}_{#1}}

\newcommandx{\ConstMW}[1][1={n,\lyapW}]{\operatorname{M}_{#1}}

\newtheorem{assumTD}{\textbf{TD}\hspace{-1pt}}
\Crefname{assumTD}{\textbf{TD}\hspace{-1pt}}{\textbf{TD}\hspace{-1pt}}
\crefname{assumTD}{\textbf{TD}}{\textbf{TD}}

\Crefname{assumptionC}{\textbf{C}\hspace{-1pt}}{\textbf{C}\hspace{-1pt}}
\crefname{assumptionC}{\textbf{C}}{\textbf{C}}

\Crefname{assumptionM}{\textbf{UGE}\hspace{-1pt}}{\textbf{UGE}\hspace{-1pt}}
\crefname{assumptionM}{\textbf{UGE}}{\textbf{UGE}}

\def\distance{\mathsf{d}}

\newcommandx{\vartconstwas}[1][1=V]{c_{#1}}

\newcommandx{\deltawas}[1][1=*]{\delta_{#1}}

\newcommandx{\wasser}[4][1=\distance,4=]{\mathbf{W}_{#1}^{#4}\left(#2,#3\right)}
\newcommandx{\covcoeff}[2]{\rho_{#1}^{(#2)}}

\newcommand{\dobrush}{\mathsf{\Delta}}
\newcommandx{\dobru}[3][1=,3=]{\dobrush_{#1}^{#3}( #2)}  

\def\qexponent{q}
\def\ppexponent{p}
\def\Markov{\mathrm{M}}

\newcommandx{\dlim}[1]{\ensuremath{\stackrel{#1}{\Longrightarrow}}}

\def\boot{\mathsf{b}}
\newcommand{\PPb}{\mathbb{P}^\boot}
\newcommand{\PEb}{\mathbb{E}^\boot}

\def\kolmogorov{\rho_n^{\Conv}}

\title{Gaussian Approximation and Multiplier Bootstrap for Polyak-Ruppert Averaged Linear Stochastic Approximation with Applications to TD Learning}

\author{%
  Sergey Samsonov\\
  HSE University \\
  svsamsonov@hse.ru
  \And
  Eric Moulines\\
  Ecole Polytechnique,\\
  MBUZAI \\
  \And
  Qi-Man Shao\\ 
  Department of Statistics and Data Science, \\
  Shenzhen International Center of Mathematics,\\
  Southern University of Science and Technology \\
  \And 
  Zhuo-Song Zhang\\
  Department of Statistics and Data Science, \\
  Shenzhen International Center of Mathematics,\\
  Southern University of Science and Technology \\
  \And
  Alexey Naumov\\
  HSE University,\\
  Steklov Mathematical Institute \\
  of Russian Academy of Sciences \\
}

\begin{document}

\maketitle

\begin{abstract}
In this paper, we obtain the Berry-Esseen bound for multivariate normal approximation for the Polyak-Ruppert averaged iterates of the linear stochastic approximation (LSA) algorithm with decreasing step size. 
Moreover, we prove the non-asymptotic validity of the confidence intervals for parameter estimation with LSA based on multiplier bootstrap. This procedure updates the LSA estimate together with a set of randomly perturbed LSA estimates upon the arrival of subsequent observations. We illustrate our findings in the setting of temporal difference learning with linear function approximation.
\end{abstract}

\section{Introduction}
\label{sec:intro}
Stochastic approximation (SA) methods are a central component for solving various optimization problems that arise in machine learning \cite{kingma2014adam,GoodBengCour16}, empirical risk minimization \cite{vapnik2013nature} and reinforcement learning \cite{mnih2015,sutton:book:2018}. There is a vast number of contributions in the literature, which cover both asymptotic \cite{nemirovskij1983problem,polyak1992acceleration} and non-asymptotic \cite{moulines2011non,duchi2012ergodic,lan2012optimal} properties of the SA estimates. The primarily important property among the asymptotic ones of the SA estimates is their asymptotic normality \cite{polyak1992acceleration}, which is important due to its role in constructing (asymptotic) confidence intervals and hypothesis testing \cite{van1996weak}. However, a natural question of the rate of convergence in the appropriate central limit theorems (CLT) is not well addressed in literature even in the relatively simple setting of the linear stochastic approximation (LSA) \cite{eweda:macchi:1983}, \cite{kushner2003stochastic}, \cite{borkar:sa:2008}.
\par 
Alternatively, confidence sets for SA algorithms can be constructed in a non-asymptotic manner based on concentration inequalities \cite{auer2002finite}. These bounds are often regarded as loose \cite{russo2014learning}, yielding suboptimal performance of the statistical procedures based on the latter estimates \cite{hao2019_bootstrap_ucb}. In contrast, for statistical inference procedures based on independent and identically distributed (i.i.d.) observations, such as $M$-estimators \cite{van1996weak}, there is a machinery of non-parametric methods for constructing confidence sets with the bootstrap \cite{efron1992bootstrap,rubin1981bayesian}. This approach is accompanied with theoretical guarantees, showing the non-asymptotic validity of the bootstrap-based confidence intervals for parameters in linear regression \cite{spokoiny2015} and statistical tests \cite{Chernozhukov2013}. Extending theoretical guarantees to a non-classical situation with online learning algorithms encounters serious problems, essentially related to the problem of obtaining rate of convergence in the corresponding CLTs. At the same time, many phenomena arising in the analysis of nonlinear SA algorithms already appear in the analysis of LSA problems. 

\par 
The LSA procedure aims to find an approximate solution for the linear system $\bA \thetalim = \barb$ with a unique solution $\thetalim$ based on a sequence of observations $\{( \funcA{Z_k}, \funcb{Z_k})\}_{k \in \nset}$. Here $\Am: \msz \to \rset^{d \times d}$ and $\bm: \msz \to \rset^d$ are measurable functions and $(Z_k)_{k \in \nset}$ is a sequence of noise variables taking values in some measurable space $(\msz,\mcz)$ with a distribution $\pi$ satisfying $\PE [ \funcA{Z_k} ] = \bA$ and $\PE [ \funcb{Z_k} ] = \barb$. We focus on the setting of independent and identically distributed (i.i.d.) observations $\{\State_k\}_{k \in \nset}$. With a sequence of decreasing step sizes $(\alpha_k)_{k \in \nset}$ and the starting point $\theta_0 \in \rset^{d}$, we consider the estimates $\{ \prtheta_{n} \}_{n \in \nset}$ given by
\begin{equation}
\label{eq:lsa}
\theta_{k} = \theta_{k-1} - \alpha_{k} \{ \funcA{Z_k} \theta_{k-1} - \funcb{Z_k} \} \eqsp,~~ k \geq 1, \quad \prtheta_{n} = n^{-1} \sum_{k=n}^{2n-1} \theta_k \eqsp, ~~n \geq 1 \eqsp.
\end{equation}
Here, we have fixed the size of the \emph{burn-in} period (see, e.g., \cite{durmus2022finite, mou2021optimal}) to $n_0 = n$. Provided that $n$ is large enough, the burn-in size affects only a constant factor in the subsequent bounds. The sequence $\{\theta_k\}_{k \in \nset}$ corresponds to the standard LSA iterates, while $\{ \prtheta_{n} \}_{n \in \nset}$ corresponds to the Polyak-Ruppert (PR) averaged iterates \cite{ruppert1988efficient, polyak1992acceleration}. It is known that $\prtheta_{n}$ is asymptotically normal with a minimax-optimal covariance matrix (see \cite{polyak1992acceleration} and \cite{fort:clt:markov:2015} for discussion). Specifically, under appropriate technical conditions on the step sizes $\{\alpha_k\}$ and noisy observations $\{\funcA{Z_k}\}$, it holds that
\begin{equation}
\label{eq:CLT_fort_prelim} 
\sqrt{n}(\bar{\theta}_{n} - \thetas) \overset{d}{\rightarrow} \mathcal{N}(0,\Sigma_{\infty})\eqsp, 
\end{equation}
where $\Sigma_{\infty}$ is the asymptotic covariance matrix defined later in \Cref{sec:clt_lsa_pr}. There is a long list of contributions to the non-asymptotic analysis of $\prtheta_{n}$, particularly \cite{mou2020linear} and \cite{durmus2022finite}, which study moment and Bernstein-type concentration bounds for $\sqrt{n}(\bar{\theta}_{n} - \thetas)$. Unfortunately, such bounds do not imply Berry-Esseen type inequalities for $\sqrt{n}(\bar{\theta}_{n} - \thetas)$, that is, they do not allow us to control the quantity
\begin{equation}
\label{eq:berry-esseen} 
\kolmogorov = \sup_{B \in \Conv(\rset^{d})}\left|\P\bigl(\sqrt{n}(\bar{\theta}_{n} - \thetas) \in B\bigr) - \P(\Sigma_{\infty}^{1/2}\eta \in B)\right|\eqsp,
\end{equation}
where $\Conv(\rset^{d})$ refers to the set of convex sets in $\rset^{d}$. While the Berry-Esseen bounds are a popular subject of study in probability theory, starting from the classical work \cite{esseen1945}, most results are obtained for sums of random variables or martingale difference sequences \cite{petrov1975sums, bolthausen1982}. We can only mention a few results for SA algorithms, see \Cref{sec:related-work} for more details. This paper aims to provide the latter bounds for the specific setting of the LSA procedure. Our primary contribution is twofold:

\begin{itemize}[noitemsep,topsep=0pt]
    \item We establish a Berry–Esseen bound for accuracy of normal approximation of the distribution of Polyak-Ruppert averaged LSA iterates with a polynomially decreasing step size. Our results suggest that the best rate of normal approximation, in the sense of \eqref{eq:berry-esseen}, is of order $n^{-1/4}$ up to logarithmic factors in $n$, where $n$ denotes the number of samples. Interestingly, this rate is achieved with an aggressive step size, $\alpha_{k} = c_0 / \sqrt{k}$. Our proof technique follows the Berry-Esseen bounds for nonlinear statistics provided in \cite{shao2022berry}.

    \item We provide non-asymptotic confidence bounds for the distribution of the PR-averaged statistic $\sqrt{n}(\prtheta_{n} - \thetas)$ using the multiplier bootstrap procedure. In particular, our bounds imply that the quantiles of the exact distribution of $\sqrt{n}(\prtheta_{n} - \thetas)$ can be approximated at a rate of $n^{-1/4}$, where $n$ is the number of samples used in the procedure, provided that $n$ is sufficiently large (see \Cref{assum:step-size-bootstrap} for exact conditions). To the best of our knowledge, this is the first non-asymptotic bound on the accuracy of bootstrap approximation in SA algorithms. We apply the proposed methodology to the temporal difference learning (TD) algorithm for policy evaluation in reinforcement learning.
\end{itemize}
The rest of the paper is organized as follows. In \Cref{sec:related-work}, we provide a literature review on the non-asymptotic analysis of the LSA algorithm and bootstrap methods. Next, in \Cref{sec:independent_case}, we analyze the convergence rate of Polyak-Ruppert averaged LSA iterates to the normal distribution. In \Cref{sec:bootstrap}, we discuss the multiplier bootstrap approach for LSA and establish bounds on the accuracy of approximating the quantiles of the true distribution. Finally, we apply our findings to TD learning and present numerical illustrations in \Cref{sec:experiments}.
\par 
\textbf{Notations.} For matrix $A \in \rset^{d \times d}$ we denote by $\norm{A}$ its operator norm. For symmetric matrix $Q = Q^\top \succ 0\eqsp, \eqsp Q \in \rset^{d \times d}$ and $x \in \rset^{d}$ we define the corresponding norm $\|x\|_Q = \sqrt{x^\top Q x}$, and define the respective matrix $Q$-norm of the matrix $B \in \rset^{d \times d}$ by $\normop{B}[Q] = \sup_{x \neq 0} \norm{Bx}[Q]/\norm{x}[Q]$. For sequences $a_n$ and $b_n$, we write $a_n \lesssim b_n$ if there exist a constant $c > 0$ such that $a_n \leq c b_n$ for $ c > 0$. For simplicity we state the main results of the paper up to constant factors.

\vspace{-3mm}
\section{Related works}
\label{sec:related-work}
Among contributions to the analysis of the LSA algorithm, we should mention the papers \cite{polyak1992acceleration, kushner2003stochastic, borkar:sa:2008, benveniste2012adaptive}. These works investigate the asymptotic properties of the LSA estimates (such as asymptotic normality and almost sure convergence) under i.i.d. and Markov noise. Non-asymptotic results for the LSA and PR-averaged LSA estimates were obtained in \cite{rakhlin2012making, nemirovski2009robust, bhandari2018finite, lakshminarayanan2018linear, mou2021optimal}, where MSE bounds were established, and in \cite{mou2020linear, durmus2021tight, durmus2022finite}, which provided high-probability error bounds. The latter results enable the construction of Bernstein-type confidence intervals for the error $\prtheta_{n}-\thetas$. Unfortunately, the corresponding bounds typically depend on unknown problem properties of \eqref{eq:lsa}, related to the design matrix $\bA$ and the noise variables $\funcA{Z_k}$, $\funcb{Z_k}$. For this reason, applying these error bounds in practice is complicated. Furthermore, concentration bounds for the LSA error \cite{mou2020linear, durmus2021tight, durmus2022finite} do not imply convergence rates of the rescaled error $\sqrt{n}(\prtheta_{n} - \thetas)$ to the normal distribution in Wasserstein or Kolmogorov distance. Non-asymptotic convergence rates were previously studied in \cite{pmlr-v99-anastasiou19a} using the Stein method, but the resulting rate corresponds to a smoothed Wasserstein distance. Recent work \cite{srikant2024rates} investigates convergence rates to the normal distribution in Wasserstein distance for LSA with Markovian observations. Both papers yield bounds that are less tight with respect to their dependence on trajectory length $n$ than those presented in the present work, see a detailed comparison after \Cref{th:shao2022_berry}. 
\par 
A popular method for constructing confidence intervals in the context of parametric estimation is based on the bootstrap approach (\cite{efron1992bootstrap}). Its analysis has attracted many contributions, in particular a series of papers \cite{Chernozhukov2013} and \cite{Chernozhukov2015} that validate a bootstrap procedure for a test based on the maximum of a large number of statistics. Their study shows a close relationship between bootstrap validity results, Gaussian comparison and anticoncentration bounds for rectangular sets. The papers \cite{spokoiny2015} and \cite{Bernolli2019} investigate the applicability of likelihood-based statistics for finite samples and large parameter dimensions under possible model misspecification. The important step in proving bootstrap validity is again based on Gaussian comparison and anticoncentration bounds, but now for spherical sets. The bootstrap procedure for spectral projectors of covariance matrices is discussed in \cite{PTRF2019} and \cite{jirak2022quantitative}. The authors follow the same steps to prove the validity of the bootstrap. 
\par 
Extending the classical bootstrap approach to online learning algorithms is a challenge. For example, the iterates $\{\theta_k\}_{k \in \nset}$ determined by \eqref{eq:lsa} are not necessarily stored in memory, which makes the classical bootstrap inapplicable. This problem can be solved by performing randomly perturbed updates of the online procedure, as proposed in \cite{JMLR:v19:17-370} for the iterates of the Stochastic Gradient Descent (SGD) algorithm. The authors in \cite{JASA2023} used the same procedure for the case of Markov noise and policy evaluation algorithms in reinforcement learning, but in both papers the authors only consider the asymptotic validity. In our paper we use the same multiplier bootstrap approach (see \Cref{sec:bootstrap}), but we provide an explicit error bound for the bootstrap approximation of the distribution of the statistics $\sqrt{n}(\prtheta_{n} - \thetas)$. 
\par 
In addition to the bootstrap approach, one can also use the pivotal statistics \cite{LEE2024105673,li2023online,li2023statistical} or various estimates of the asymptotic covariance matrix \cite{zhu2023online_cov_matr} to construct the confidence intervals for $\thetas$. The latter approach can be based on the plug-in estimators \cite{pmlr-v178-li22b}, batch mean estimators \cite{chen2020aos} or in combination with the multiplier bootstrap approach \cite{zhong2023online}. However, the theoretical guarantees for mentioned methods remain purely asymptotic.

\vspace{-3mm}
\section{Accuracy of normal approximation for LSA}
\label{sec:independent_case}
We first study the rate of normal approximation for the tail-averaged LSA procedure. When there is no risk of ambiguity, we use simply the notations $\funcAw_k = \funcA{\State_k}$ and $\funcbw_k = \funcb{\State_k}$. Starting from the definition \eqref{eq:lsa}, we get with elementary transformations that
\begin{equation}
\label{eq:main_recurrence_1_step}
\theta_{k} - \thetas = (\Id - \alpha_{k} \funcAw_k)(\theta_{k-1} - \thetas) - \alpha_{k} \funnoisew_{k}\eqsp,
\end{equation}
where we have set $\funnoisew_k= \funcnoise{\State_k}$ with 
\begin{equation}
\label{eq:def_center_version_and_noise}
\textstyle
\funcnoise{z} =  \zmfuncA{z} \thetas - \zmfuncb{z}\eqsp, \quad \zmfuncA{z}  = \funcA{z} - \bA \eqsp, \quad   \zmfuncb{z} = \funcb{z} - \barb \eqsp \eqsp.
\end{equation}
Here the random variable $\funcnoise{\State_k}$ can be viewed as a noise, measured at the optimal point $\thetas$. We now assume the following technical conditions:
\begin{assum}
\label{assum:iid}
Sequence $\{\State_k\}_{k \in \nset}$ is a sequence of \iid\ random variables defined on a probability space $(\Omega,\mcf,\PP)$ with distribution $\pi$.
\end{assum}
\begin{assum}
\label{assum:noise-level}
$\int_{\Zset}\funcA{z}\rmd \pi(z) = \bA$ and $\int_{\Zset}\funcb{z}\rmd \pi(z) = \barb$, with the matrix $-\bA$ being Hurwitz. Moreover, $\supconsteps = \sup_{z \in \msz}\normop{\funcnoise{z}} < \plusinfty$, and the mapping $z \to \funcA{z}$ is bounded, that is, 
\begin{equation}
\label{eq:a_matr_bounded}
\bConst{A} = \sup_{z \in \msz} \normop{\funcA{z}} \vee \sup_{z \in \msz} \normop{\zmfuncA{z}} < \infty\eqsp.
\end{equation}
Moreover, for the noise covariance matrix
\begin{equation}
\label{eq:def_noise_cov}
\textstyle \noisecov = \int_{\Zset} \funcnoise{z}\funcnoise{z}^\top \rmd \pi(z)
\end{equation}
it holds that its smallest eigenvalue is bounded away from $0$, that is,
\begin{equation}
\label{eq:eig_sigma_eps}
\textstyle \lambda_{\min}:= \lambda_{\min}(\noisecov) > 0\eqsp.
\end{equation}
\end{assum}
It is possible to change \eqref{eq:a_matr_bounded} to the moment-type bound as it was previously considered in \cite{mou2020linear} and \cite{durmus2022finite}, see the detailed discussion after \Cref{th:shao2022_berry}. The fact that the matrix $-\bA$ is Hurwitz implies that the linear system $\bA \theta = \barb$ has a unique solution $\thetalim$. Moreover, this fact is sufficient to show that the matrix $\Id - \alpha \bA$ is a contraction in an appropriate matrix $Q$-norm for small enough $\alpha > 0$. Precisely, the following result holds:
\begin{proposition}
\label{prop:hurwitz_stability}
Let $-\bA$ be a Hurwitz matrix. Then for any $P = P^{\top} \succ \Id$, there exists a unique matrix $Q = Q^{\top} \succ \Id$, satisfying the Lyapunov equation $\bA^\top Q + Q \bA = P$. Moreover, setting
\begin{equation}
\label{eq:alpha_infty_def}
\textstyle 
a = \frac{\lambda_{\min}(P)}{2\normop{Q}}\eqsp, \quad
\text{and} \quad \alpha_\infty = \frac{\lambda_{\min}(P)}{2\qcond \normop{\bA}[Q]^{2}}\wedge \frac{\normop{Q}}{\lambda_{\min}(P)} \eqsp,
\end{equation}
where $\qcond = \lambda_{\max}(Q)/\lambda_{\min}(Q)$, it holds for any $\alpha \in [0, \alpha_{\infty}]$ that $\alpha a \leq 1/2$, and
\begin{equation}
\label{eq:contractin_q_norm}
\normop{\Id - \alpha \bA}[Q]^2 \leq 1 - \alpha a\eqsp.    
\end{equation}
\end{proposition}
The proof of \Cref{prop:hurwitz_stability} is provided in \Cref{proof:hurwitz_stability}. Note that it is possible to set $P = \Id$ as in \cite{durmus2021stability}, yet it is possible that other choices of $P$ could be more beneficial for particular applications. Now consider an assumption on the step sizes $\alpha_{k}$ and number of observations $n$:
\begin{assum}
\label{assum:step-size}
The step sizes $\{\alpha_{k}\}_{k \in \nset}$ has a form $\alpha_{k} = c_{0} / k^{\gamma}$, where $\gamma \in [1/2;1)$ and $c_{0} \in (0;\alpha_{\infty} \wedge a \wedge (1-\gamma)]$. Moreover, we assume that $n \geq d$, and 
\begin{equation}
\label{eq:sample_size_bound}
\begin{cases}
\frac{\sqrt{n}}{(1+\log{n})\log{n}} \geq \frac{c_{0} \qcond \bConst{A}^2}{a(1-\sqrt{2}/2)} \vee \frac{4}{a c_{0} (1-\sqrt{2}/2)}\eqsp, \text{ if } \gamma = 1/2\eqsp, \\
\frac{n^{1-\gamma}}{\log{n}} \geq \frac{2 c_0 \qcond \bConst{A}^2 }{a(2\gamma - 1)(1-(1/2)^{1-\gamma})} \vee \frac{8 \gamma (1-\gamma)}{a c_0 (1 - (1/2)^{1-\gamma}}\eqsp, \text{ if } 1/2 < \gamma < 1\eqsp. 
\end{cases}
\end{equation}
\end{assum}

The main aim of lower bounding $n$ is to ensure that the number of observations is large enough in order that the LSA error related to the choice of initial condition $\theta_0 - \thetas$ becomes small.

\subsection{Central limit theorem for Polyak-Ruppert averaged LSA iterates.}
\label{sec:clt_lsa_pr}
It is known that the assumptions \Cref{assum:iid}-\Cref{assum:step-size} guarantee that the CLT applies to the iterates of $\prtheta_{n}$, namely, 
\begin{equation}
\label{eq:CLT_fort} 
\sqrt{n}(\bar{\theta}_{n} - \thetas) \overset{d}{\rightarrow} \mathcal{N}(0,\Sigma_{\infty})\eqsp, 
\end{equation}
where the asymptotic covariance matrix $\Sigma_{\infty}$ has a form 
\begin{equation}
\label{eq:asympt_cov_matr} 
\Sigma_{\infty} = \bA^{-1} \noisecov \bA^{-\top},
\end{equation}
and $\noisecov$ is defined in \eqref{eq:def_noise_cov}. This result can be found for example in \cite{polyak1992acceleration} and \cite{fort:clt:markov:2015}. We are interested in the Berry-Esseen type bound for the rate of convergence in \eqref{eq:CLT_fort}, that is, we aim to bound $\kolmogorov$ defined in \eqref{eq:berry-esseen} w.r.t. the available sample size $n$. We control $\kolmogorov$ using a method from \cite{shao2022berry} based on randomized multivariate concentration inequality. Below we briefly state its setting and required definitions.
Let $X_1, \ldots, X_n$ be independent random variables  taking values in $\mathcal X$ and $T = T(X_1, \ldots, X_n)$ be a general $d$-dimensional statistics such that $T = W + D$, where 
\begin{equation}
\label{eq:W-D-decomposition} 
W = \sum_{\ell = 1}^n \xi_\ell, \quad D: = D(X_1, \ldots, X_n) = T - W,
\end{equation}
$\xi_\ell = h_\ell(X_\ell)$ and $h_\ell: \mathcal X \to \rset^d$ is a Borel measurable function. Here the statistics $D$ can be non-linear and is treated as an error term, which is "small" compared to $W$ in an appropriate sense. Assume that $\PE[\xi_\ell] = 0$ and $\sum_{\ell=1}^n \PE[\xi_\ell \xi_\ell^\top] = \Id_d$. Let $\Upsilon = \Upsilon_n = \sum_{\ell=1}^n \PE[\|\xi_\ell\|^3]$. Then, with $\eta \sim \mathcal{N}(0,\Id_d)$, 
\begin{equation}
\label{eq:shao_zhang_bound}
\sup_{B \in \Conv(\rset^d)} | \PP(T \in A) - \PP(\eta \in A)| \le 259 d^{1/2} \Upsilon + 2 \PE[\|W\| \|D\|] + 2 \sum_{\ell=1}^n \PE[\|\xi_\ell\| \|D - D^{(\ell)}\|],
\end{equation}
where $D^{(\ell)} = D(X_1, \ldots, X_{\ell-1}, X_{\ell}^{\prime}, X_{\ell+1}, \ldots, X_n)$ and $X_\ell^{\prime}$ is an independent copy of $X_\ell$. This result is due to \cite[Theorem~2.1]{shao2022berry}. One can modify the bound \eqref{eq:shao_zhang_bound} for the setting when $\sum_{\ell=1}^n \PE[\xi_\ell \xi_\ell^\top] = \Sigma \succ 0$. This result due to \cite[Corollary~2.3]{shao2022berry}. Following the construction \eqref{eq:W-D-decomposition}, we set $T = \sqrt{n}\bA(\bar{\theta}_{n} - \thetas)$ and consider it as a nonlinear statistic of i.i.d. random variables $\State_1, \ldots, \State_{2n}$, which drive the LSA dynamics \eqref{eq:lsa}. We can exactly represent $T$ as a sum of linear ($W$) and non-linear parts ($D$), where  
\begin{multline}
W  = - \frac{1}{\sqrt{n}} \sum_{k=n}^{2n-1}\funnoisew_{k+1}, \quad 
D  =   \frac{1}{\sqrt{n}}\frac{\theta_{n}-\thetas}{\alpha_{n}} - \frac{1}{\sqrt{n}}\frac{\theta_{2n}-\thetas}{\alpha_{2n}}
-\frac{1}{\sqrt{n}} \sum_{k=n+1}^{2n}(\funcAw_k - \bA)(\theta_{k-1} - \thetas) \\
+\frac{1}{\sqrt{n}}\sum_{k=n+1}^{2n}\bigl(\theta_{k-1} - \thetas\bigr)\left(\frac{1}{\alpha_k} - \frac{1}{\alpha_{k-1}}\right).
\end{multline}    
The proof of this result can be bound in \Cref{prop: expansion}. To obtain a bound for the approximation accuracy in
\eqref{eq:berry-esseen} using the bound \eqref{eq:shao_zhang_bound}, we need to upper bound $\PE^{1/2}[\norm{D(\State_1,\ldots,\State_{2n})}^2]$ and $\PE[\norm{D - D^{(i)}}]$. The first result below provides a second moment bound on $D$:
\begin{theorem}
\label{th:theo_1_iid}
Assume \Cref{assum:iid}, \Cref{assum:noise-level}, and \Cref{assum:step-size}. Then we obtain the following error bound:
\begin{equation}
\label{eq:MSE_2nd_moment_decreasing}
\begin{split}
\PE^{1/2}\left[\norm{D(\State_1,\ldots,\State_{2n})}^{2}\right] 
&\lesssim \frac{\sqrt{\qcond} \supconsteps}{\sqrt{a c_0}}\left(\frac{1}{n^{(1-\gamma)/2}} + \frac{c_0 \bConst{A}}{\sqrt{1-\gamma} n^{\gamma/2}}\right)  \\ 
& \qquad  + \sqrt{\qcond} \Remainder_1 \exp\biggl\{-\frac{c_0 a n^{1-\gamma}}{2(1-\gamma)}\biggr\} \norm{\theta_{0}-\thetas}\eqsp,
\end{split}
\end{equation}
where $\lesssim$ stands for inequality up to an absolute constant, and $\Delta_1 = 
\Remainder_{1}(n,a,\bConst{A},c_0)$ is a polynomial function defined in \Cref{sec:proof_theo_1_iid}, eq. \eqref{eq:remainders_theo_1_iid}.
\end{theorem}
The proof of \Cref{th:theo_1_iid} is provided in \Cref{sec:proof_theo_1_iid}. Now it remains to upper bound the term $\PE[\norm{D - D^{(i)}}]$, which is done in \Cref{lem:aux-lemmas-normal} using the synchronous coupling methods \cite{brosse2018pitfalls}. Combining these bounds, we obtain the following theorem:
\begin{theorem}
\label{th:shao2022_berry}
Assume \Cref{assum:iid}, \Cref{assum:noise-level}, and \Cref{assum:step-size}. Then the following bound holds:
\begin{multline}
\label{eq:kolmogorov_bound_non_optimized}
\kolmogorov \lesssim \frac{d^{1/2}  \supconsteps^3}{\lambda_{\min}^{3/2} \sqrt{n}} + \frac{1}{\lambda_{\min}}\left(\frac{\ConstPR{1}}{n^{(1-\gamma)/2}} + \frac{\ConstPR{2}}{n^{\gamma/2}}\right) + \frac{\Remainder_{2}}{\lambda_{\min}} \exp\biggl\{-\frac{c_0 a n^{1-\gamma}}{2(1-\gamma)}\biggr\} \norm{\theta_0 - \thetas}\eqsp,
\end{multline}
where $\Remainder_{2} = \Remainder_{2}(n,a,\bConst{A},\trace{\noisecov},c_0)$ is a polynomial function defined in \eqref{eq:remainders_theorem_shao}, and constants $\ConstPR{1},\ConstPR{2}$, depending upon $a,\bConst{A},\qcond,\trace{\noisecov},c_0$, are defined in \Cref{sec:proof_shao2022_berry}, eq. \eqref{eq:const_def_th_2}.
\end{theorem}
The proof of \Cref{th:shao2022_berry} is provided in \Cref{sec:proof_shao2022_berry}. Note that the assumption \Cref{assum:noise-level} requires that $\funcnoise{\State_1}$ is almost sure bounded. It is a strong assumption, but it can be partially relaxed. Following the stability of matrix products technique, used in \cite[Proposition 3]{durmus2021tight}, it is possible to consider the setting when the random variable $\norm{\zmfuncA{Z_1}}$ has only finite number of moments. In particular, we expect that assuming finite third moment of $\norm{\zmfuncA{Z_1}}$ and  $\norm{\funcnoise{\State_1}}$ is sufficient to obtain a counterpart to \Cref{th:theo_1_iid}. However, this generalization requires non-trivial technical work on generalizing the stability of matrix products result (see \Cref{cor:exp_bound_decay} in \Cref{appendix:tehnical} ).  
\par 
Note that the bound of \Cref{th:shao2022_berry} predicts the optimal error of normal approximation for Polyak-Ruppert averaged estimates of order $n^{-1/4}$, which is achieved with the aggressive step size $\alpha_{k} = c_{0}/\sqrt{k}$, that is, when setting $\gamma = 1/2$ in \eqref{eq:kolmogorov_bound_non_optimized}. In this case we obtain the optimized bound 
\begin{equation}
\label{eq:kolmogorov_bound_optimized}
\kolmogorov \lesssim \frac{\ConstPR{3}}{\lambda_{\min} n^{1/4}} + \frac{d^{1/2}  \supconsteps^3}{\lambda_{\min}^{3/2} \sqrt{n}} + \frac{\Remainder_{1} \exp\bigl\{-c_0 a \sqrt{n}\bigr\}}{\lambda_{\min}}  \norm{\theta_0 - \thetas}\eqsp,
\end{equation}
where $\ConstPR{3} = \ConstPR{3}(a,\bConst{A},\qcond,\trace{\noisecov},\supconsteps)$ is provided in \eqref{eq:const_def_th_2}.

\paragraph{Discussion.} Our proof technique of \Cref{th:shao2022_berry} reveals an interesting feature: fastest rate of convergence in the convex distance $\kolmogorov$ corresponds to the learning rate schedule that admits the fastest decay of the second-order term in the MSE bound for remainder statistics $D$ (see \Cref{th:theo_1_iid}). Results similar to the one of \Cref{th:shao2022_berry} have been recently obtained in the literature in \cite{srikant2024rates} and \cite{pmlr-v99-anastasiou19a}. The author in \cite{srikant2024rates} considers the LSA problem specified to the temporal-difference learning (see \Cref{sec:experiments}) with Markov noise and obtains convergence rate in Wasserstein distance of order $n^{-1/4}$, which corresponds to the "optimal" step size schedule $\alpha_{k} = c_{0}/k^{3/4}$. Using the bound of \cite[eq. (3)]{nourdin2022multivariate} (see also section~2 in \cite{ross2011stein}), this result yield a suboptimal bound of order $n^{-1/8}$ for the convex distance $\kolmogorov$. Such an upper bound may be loose for some classes of distributions, but it is not clear if in particular setting of LSA the bound of \cite{srikant2024rates} could imply scaling of order $n^{-1/4}$ for $\kolmogorov$. At the same time, in case of $X_1,\ldots,X_n$ forming a Markov chain in \eqref{eq:W-D-decomposition} there is no available counterpart of the bound \eqref{eq:shao_zhang_bound}. Generalizing \eqref{eq:shao_zhang_bound} is an interesting research direction that would allow to obtain a counterpart of \Cref{th:shao2022_berry} in case of Markovian dynamics. Similarly, the result of  \cite{pmlr-v99-anastasiou19a} holds for much stronger metrics, which controls the convergence of moments of twice differentiable functions. We provide additional details about connections between this metric and $\kolmogorov$ in \Cref{sec:smooth_wasserstein_kolmogorov}. At the same time, the authors in \cite{pmlr-v99-anastasiou19a} cover the non-linear setting of PR-averaged iterates of stochastic gradient descent algorithm under strong convexity. 

\begin{remark}
\label{rem:projected_iterates}
The leading (with respect to $n$) terms of the bound from \Cref{th:theo_1_iid} have an implicit dependence on the problem dimension $d$ due to the presence of $\lambda_{\min}$. Yet the result of \Cref{th:theo_1_iid} can be improved in a sense of dependence in dimension if one is interested not in the rates of convergence for $\sqrt{n}(\prtheta_{n} - \thetas)$, but in the projected iterated $\sqrt{n}\Pi^\top(\prtheta_{n} - \thetas)$ for some $\Pi \in \rset^{d \times m}$, $m \leq d$. If this is the case, one may apply \eqref{eq:shao_zhang_bound} for the class $\Conv_m = \Conv(\rset^m)$ of convex sets in $\rset^m$ and obtain, setting step size $\alpha_{k} = c_0/\sqrt{k}$, and $\noisecov^{(\Pi)} = \Pi \noisecov \Pi^{\top}$, that
\begin{equation}
\label{eq:kolmogorov_bound_proj_optimized}
\kolmogorov \lesssim \frac{\ConstPR{4}}{\lambda_{\min} n^{1/4}} + \frac{m^{1/2}  \supconsteps^3}{\lambda_{\min}^{3/2} \sqrt{n}} + \frac{\Remainder_{2} \rme^{-c_0 a \sqrt{n}}}{\lambda_{\min}}  \norm{\theta_0 - \thetas}\eqsp,
\end{equation}
and the constant $\ConstPR{4} = \ConstPR{4}(a,\bConst{A},\qcond,\trace{\noisecov^{(\Pi)}},\supconsteps)$ is provided in \eqref{eq:const_def_th_2}.
\end{remark}

\begin{remark}
\label{rem:last_iterate_bound}
Results similar to \Cref{th:theo_1_iid} can be obtained not only for the Polyak-Ruppert averaged estimator $\bar{\theta}_{n}$, but also for the last iterate $\theta_n$. In particular, it is known (see e.g. \cite{fort:clt:markov:2015}), that the last iterate error $\theta_n - \thetas$ is also asymptotically normal:
\[
\textstyle 
\frac{\theta_n - \theta^*}{\sqrt{\alpha_n}} \to \mathcal{N}(0,\Sigma_{\operatorname{last}})\eqsp,
\]
where the covariance matrix $\Sigma_{\operatorname{last}}$ is different from $\Sigma_{\infty}$. In such a case $\Sigma_{\operatorname{last}}$ can be found as a solution to appropriate Lyapunov equation, see \cite{fort:clt:markov:2015}. Then, we expect that it is possible to use the perturbation-expansion approach from \cite{aguech2000perturbation} together with randomized concentration inequalities \cite{shao2022berry} (see \eqref{eq:shao_zhang_bound}), in order to obtain the Berry-Esseen bound 
\[
\textstyle 
\sup_{B \in \Conv(\rset^{d})}\left|\P\bigl(\frac{\theta_n - \theta^*}{\sqrt{\alpha_n}} \in B\bigr) - \P(\Sigma_{\operatorname{last}}^{1/2} \eta \in B)\right| \lesssim \sqrt{\alpha_n}\eqsp.
\]
We leave the detailed derivation for future work.
\end{remark}

\vspace{-3mm}
\section{Multiplier bootstrap for LSA}
\label{sec:bootstrap}
In order to perform statistical inference with the Polyak-Ruppert estimator $\bar{\theta}_{n}$, we propose an online bootstrap resampling procedure, which recursively updates the LSA estimate as well as a large number of randomly perturbed LSA estimates, upon the arrival of each data point. The suggested procedure follows the one outlined in \cite{JMLR:v19:17-370}. It has the following advantages: it does not rely on the asymptotic distribution of the error $\sqrt{n}(\bar{\theta}_{n} - \thetas)$, does not require to know the moments of $\sqrt{n}(\bar{\theta}_{n} - \thetas)$ or its asymptotic covariance matrix $\Sigma_{\infty}$, and does not involve any data splitting.
\par 
We state the suggested procedure as follows. Let $\mathcal W^{2n} = \{W_\ell\}_{1 \leq \ell \leq 2n}$ be a set of i.i.d. random variables, independent of $\mathcal Z^{2n} = \{Z_\ell\}_{1 \leq \ell \leq 2n}$, with $\PE[W_1] = 1$ and $\var[W_1]=1$. We write, respectively, $\PPb = \PP(\cdot | \mathcal Z^{2n})$ and $\PEb = \PE(\cdot | \mathcal Z^{2n})$ for the corresponding conditional probability and expectation. In parallel with procedure \eqref{eq:lsa} that generates $\{\theta_{k}\}_{1 \leq k \leq 2n}$ and $\prtheta_{n}$, we generate $M$ independent samples $(w_{n}^{\ell},\ldots,w_{2n}^{\ell})$, $1 \leq \ell \leq M$ distributed as $\mathcal W^{2n}$, and recursively update $M$ randomly perturbed LSA estimates, that is, 
\begin{equation}
\label{eq:lsa_bootstrap}
\begin{split}
\textstyle \theta_{k}^{\boot,\ell} 
&= \textstyle \theta_{k-1}^{\boot,\ell} - \alpha_{k} w_k^{\ell}\{ \funcA{Z_k} \theta_{k-1}^{\boot,\ell} - \funcb{Z_k} \} \eqsp,~~ k \geq n+1 \eqsp, ~~ \theta_{n}^{\boot,\ell} = \theta_{n} \eqsp, \\
\textstyle \prtheta_{n}^{\boot,\ell} 
&= \textstyle n^{-1} \sum_{k=n}^{2n-1} \theta_k^{\boot,\ell} \eqsp, ~~n \geq 1 \eqsp.
\end{split}
\end{equation}
We use a short notation $\bar{\theta}_{n}^\boot$ for $\bar{\theta}_{n}^{\boot,1}$. The key idea of the procedure \eqref{eq:lsa_bootstrap} is that the "Bootstrap-world" distribution (that is, the one conditional on $\mathcal Z^{2n}$) of the perturbed samples $\sqrt n (\bar{\theta}_{n}^\boot - \bar{\theta}_n)$ is close to the distribution of the quantity of interest, that is, $\sqrt n (\bar{\theta}_n - \thetas)$. Precisely, the main result of this section will show that the quantity
\begin{equation}
\label{eq:boot_validity_supremum}
\sup_{B \in \Conv(\rset^{d})} |\PPb(\sqrt n (\bar{\theta}_{n}^\boot - \bar{\theta}_n) \in B ) - \PP(\sqrt n (\bar{\theta}_n - \thetas) \in B)|
\end{equation}
is small. Although an analytic expression for $\PPb(\sqrt n (\bar{\theta}_{n}^\boot - \bar{\theta}_n) \in B )$ is not available, one can approximate it from numerical simulations according to \eqref{eq:lsa_bootstrap} by generating sufficiently large number $M$ of perturbed trajectories. Standard arguments, see e.g. \cite[Section~5.1]{shao2003mathematical} suggest that the accuracy of Monte-Carlo approximation is of order $M^{-1/2}$. To analyze the suggested procedure, we shall impose an additional assumption on the trajectory length $n$:
\begin{assum}
\label{assum:step-size-bootstrap}
Assumption \Cref{assum:step-size} holds with $\gamma = 1/2$, and $c_0 \leq 1/(\bConst{A}^2 \qcond \rme)$. Moreover, setting 
\begin{equation}
\label{eq:block_size_constraint}
\textstyle 
h(n) = \biggl\lceil \biggl(\frac{4\bConst{A}\qcond^{1/2}}{(\sqrt{2}-1)a}\biggr)^{2}(1+2\log{(2n^4)})^2 \biggr\rceil\eqsp,
\end{equation}
it holds that 
\begin{equation}
\label{eq:sample_size_bound_part_2}
\textstyle
\frac{\sqrt{n}}{h(n)} \geq \frac{2}{a(\sqrt{2}-1)} \vee \frac{c_{0}}{\alpha_{\infty}}\eqsp, \text{ and } \frac{\sqrt{n}}{\log^{2}{n}} \geq \frac{c_{0}(1 \vee \bConst{A}^2)}{a} \vee c_{0} a \bConst{A}^2  \vee \frac{4}{a c_{0}}
\end{equation}
Moreover, we assume that for $\lambda_{\min}$ defined in \eqref{eq:eig_sigma_eps} it holds that 
\begin{equation}
\label{eq:lambda_min_bound_boot}
\lambda_{\min} \geq 8  \supconsteps \sqrt{\frac{\|\noisecov\| \log n}{n}} + \frac{8 (\|\noisecov\| + \supconsteps^2 )\log n}{n}
\end{equation}
\end{assum}
Note that the new bound \eqref{eq:sample_size_bound_part_2} simply states that $\sqrt{n}/\log^{2}(n)$ is sufficiently large, since $h(n)$ scales as $\log^2{n}$. We discuss the assumption \Cref{assum:step-size-bootstrap} in more details in the proof scheme. Now we formulate the main result of this section. We analyze only the setting of polynomially decaying step size with $\gamma = 1/2$, since decay rate of \eqref{eq:boot_validity_supremum} essentially depends on the approximation rate of \Cref{th:shao2022_berry}, with the fastest rate achieved when $\gamma = 1/2$. For other learning rates the decay rate of right-hand side in \Cref{th:bootstrap_validity} will be slower. For simplicity, we do not trace the dependence of the bound below on the parameter $c_0$.

\begin{theorem}
\label{th:bootstrap_validity}
Assume \Cref{assum:iid}, \Cref{assum:noise-level}, \Cref{assum:step-size} with $\gamma = 1/2$, and \Cref{assum:step-size-bootstrap}. Then with $\PP$ -- probability at least $1 - 6/n$ it holds that
\begin{align}
&\sup_{B \in \Conv(\rset^{d})} |\PPb(\sqrt n (\bar{\theta}_{n}^\boot - \bar{\theta}_n) \in B ) - \PP(\sqrt n (\bar{\theta}_n - \thetas) \in B)| \lesssim \frac{\qcond^{2}(\bConst{A}^4 \vee 1) (1 + \supconsteps^2) \log{n}}{a^{5/2} \lambda_{\min} n^{1/4}} \\
&\qquad + \frac{\sqrt{d}}{\sqrt{n}}\left(\frac{\supconsteps^3}{\lambda_{\min}^{3/2}} + \qcond \supconsteps \frac{\sqrt{\log{n}}}{\sqrt{\lambda_{\min}}} + \frac{\qcond (1 + \supconsteps^2 / \lambda_{\min}) \log{n}}{\sqrt{n}} \right) + \frac{\Remainder_{3}\rme^{-(c_0/2) a \sqrt{n}}}{\lambda_{\min}}\norm{\theta_0-\thetas}\eqsp,
\end{align}
where $\Remainder_{3} = \Remainder_{3}(n,a,\bConst{A},\supconsteps)$ is a polynomial function defined in \Cref{appendix:bootstrap}, eq. \eqref{eq:remainders_theorem_bootstrap}.
\end{theorem}

The proof of \Cref{th:bootstrap_validity} is based on the Gaussian approximation performed both in the "real" world and bootstrap world together with an appropriate Gaussian comparison inequality. The main steps of the proof are illustrated by the following scheme:

\begin{tikzcd}[column sep = 120pt]
\text{Real world: \quad\quad} \sqrt n \bA(\bar{\theta}_{n} - \thetas) \arrow[<->]{r}{\text{Gaussian approximation, Th.}~\ref{th:shao2022_berry}}  &
  \xi \sim \mathcal N(0,  \noisecov )  \arrow[<->]{d}{\text{Gaussian comparison, Theorem }\ref{prop: gaussian comparison}} 
\\
\text{Bootstrap world: } \sqrt n \bA (\bar{\theta}_{n}^\boot - \bar{\theta}_n) \arrow[<->]{r}{\text{Gaussian approx. in Bootstrap world, Th.}~\ref{CLT in the bootstrap world}} &
\xi^\boot \sim \mathcal N(0,  \noisecov^\boot) 
\end{tikzcd}

In the above scheme we have denoted by $ \noisecov^\boot = n^{-1} \sum_{\ell = n}^{2n-1} \funnoisew_\ell \funnoisew_\ell^{\top}$ the sample covariance matrix approximating $\noisecov$. 
Gaussian approximation for the true distribution of $\sqrt n \bA(\bar{\theta}_{n} - \thetas)$ follows from \Cref{th:shao2022_berry}. Proof of Gaussian approximation in the Bootstrap world \Cref{CLT in the bootstrap world} is also based on the inequality \eqref{eq:shao_zhang_bound}, but is more complicated and involves the expansion analysis of the LSA error from \cite{aguech2000perturbation}. This technique allows to separate the LSA error into different scales with respect to the step sizes $\{\alpha_{k}\}$, see \Cref{sec:error_decomspotion_perturbed} for details. However, this technique requires to impose additional assumption \Cref{assum:step-size-bootstrap} - eq. \eqref{eq:sample_size_bound_part_2}. Proof of the Gaussian comparison part of \Cref{prop: gaussian comparison} is based on Pinsker's inequality and matrix Bernstein inequality. The latter result requires that $n$ is large enough to ensure that minimal eigenvalue of $\noisecov^\boot$ is close to $\lambda_{\min}$, justifying the assumption \Cref{assum:step-size-bootstrap} - eq. \eqref{eq:lambda_min_bound_boot}. Detailed proof if provided in \Cref{appendix:bootstrap}.

\textbf{Discussion.} We emphasize that the Gaussian approximation result of \Cref{th:shao2022_berry} (with Bootstrap world generalization in \Cref{CLT in the bootstrap world}) is a key result to prove the above bootstrap validity. This argument was missing in the earlier works studying confidence intervals for stochastic optimization algorithms \cite{chen2020aos,zhu2023online_cov_matr,zhong2023online}, where the authors considered procedures to estimate $\Sigma_{\infty}$ in \eqref{eq:asympt_cov_matr}. They combine \emph{non-asymptotic} bounds on the accuracy of recovering $\Sigma_{\infty}$ with only \emph{asymptotic} validity of the resulting confidence intervals. We expect that our proof technique for \Cref{th:shao2022_berry} can be used to provide similar non-asymptotic validity results for outlined approaches for constructing confidence intervals based on the estimation of the asymptotic covariance matrix.

\begin{corollary}(Set of Euclidean balls or ellipsoids) Suppose that we are interested in estimating quantile of a given order $\alpha \in (0,1)$ and some matrix $B \in \rset^{d \times d}$, that is, the quantity
\begin{equation}
\label{eq: real quantile}
t_\alpha = \inf\{ t > 0: \PP(\sqrt n \|B(\bar{\theta}_{n} - \thetas)\| \geq t) \leq \alpha \}.
\end{equation}
We define its counterpart in the Bootstrap world, $t_\alpha^\boot$, as
\begin{equation}
\label{eq: boot quantile}
t_\alpha^\boot = \inf\{ t > 0: \PPb(\sqrt n \|B(\bar{\theta}_{n}^\boot - \bar{\theta}_n)\| \geq t) \leq \alpha \}. 
\end{equation}
Note that $t_\alpha^\boot$ is defines with respect to the bootstrap measure, therefore, it depends on the data $\mathcal Z^{2n}$. This bootstrap critical value $t_\alpha^\boot$ is applied in the Bootstrap world to build the confidence set 
\begin{equation}
\label{confidence set}
\mathcal E(\alpha) = \{\theta \in \rset^d: \sqrt n \| B(\theta - \bar{\theta}_n)\| \leq t_\alpha^\boot\}\eqsp.
\end{equation}
\Cref{th:bootstrap_validity} justifies this construction and evaluate the coverage probability of the true value $\thetas$ by this set. It states that 
$$
\PP(\thetas \notin \mathcal E(\alpha)) = \PP(\sqrt n \|B(\bar{\theta}_n - \thetas)\| > t_\alpha^\boot) \approx \alpha\eqsp,
$$
with the error of order $n^{-1/4}$ in the right-hand side. Although an analytic expression for $t_\alpha^\boot$ is not available, one can approximate it by generating a large number $M$ of independent samples of $\mathcal W_n$ and computing from them the empirical distribution function of $\sqrt n \|B(\bar{\theta}_{n}^\boot - \bar{\theta}_n)\|$, following \eqref{eq:lsa_bootstrap}.
\end{corollary}

\begin{remark}
\label{rem:non-linear}
A natural question that arises after \Cref{th:bootstrap_validity} is whether it is possible to prove similar bounds for the iterates of first-order stochastic optimization algorithms. There are several MSE bounds for corresponding algorithms with explicit dependence on the step size $\alpha_{k}$; see, for example, \cite{moulines2011non, bach:moulines:2013}. Therefore, we expect that it is possible to obtain a counterpart to \Cref{th:shao2022_berry}. At the same time, for general first-order stochastic optimization algorithms, unlike LSA, there are no counterparts to the precise error expansions of \cite{aguech2000perturbation}. Thus, proving the counterpart of \Cref{th:bootstrap_validity} in this setting is more challenging. Similarly, we emphasize that generalizations of the procedure \eqref{eq:lsa_bootstrap} to cases where \( \{\State_k\}_{k \in \nset} \) are dependent, for example, form a Markov chain, are complicated. The approach of \cite{shao2022berry} is not directly applicable in this setting, and appropriate generalization of \eqref{eq:shao_zhang_bound} is a separate and challenging research direction.
\end{remark}

\vspace{-3mm}
\section{Applications to the TD learning and numerical results}
\label{sec:experiments}
We illustrate our findings for the setting of temporal difference (TD) learning algorithm \cite{sutton1988learning,sutton:book:2018} for policy evaluation in RL. Non-asymptotic error bounds for this algorithm attracted lot of contributions \cite{mou2020linear, durmus2022finite, huo2023bias, patil2023finite,li2023sharp}. At the same time, confidence intervals for TD were studied in \cite{JMLR:v19:17-370,JASA2023} only in terms of their asymptotic validity. In the TD algorithm we consider a discounted MDP (Markov Decision Process) given by a tuple $(\S,\A,\PMDP,r,\gamma)$. Here $\S$ and $\A$ stand for state and action spaces, and $\gamma \in (0,1)$ is a discount factor. Assume that $\S$ is a complete metric space with metric $\metrics$ and Borel $\sigma$-algebra $\borel{\S}$. $\PMDP$ stands for the transition kernel $\PMDP(B | s,a)$, which determines the probability of moving from state $s$ to a set $B \in \borel{\S}$ when action $a$ is performed. Reward function $r\colon \S \times \A \to [0,1]$ is assumed to be deterministic. \emph{Policy} $\pi(\cdot|s)$ is the distribution over action space $\A$ corresponding to agent's action preferences in state $s \in \S$. We aim to estimate \emph{value function}
\[
V^{\pi}(s) = \PE\bigl[\sum_{k=0}^{\infty}\gamma^{k}r(s_k,a_k)|s_0 = s\bigr]\eqsp,
\]
where $a_{k} \sim \pi(\cdot | s_k)$, and $s_{k+1} \sim \PMDP(\cdot | s_{k}, a_{k})$ for any $k \in \nset$. Define the transition kernel under $\pi$,
\begin{equation}
\label{eq:transition_matrix_P_pi}
\textstyle \PMDP_{\pi}(B | s) = \int_{\A} \PMDP(B | s, a)\pi(\rmd a|s)\eqsp,
\end{equation}
which corresponds to the $1$-step transition probability from state $s$ to a set $B \in \borel{\S}$. The state space $\S$ here can be arbitrary. It is a common option to consider the \emph{linear function approximation} for $V^{\pi}(s)$, defined for $s \in \S$, $\theta \in \rset^{d}$, and a feature mapping $\varphi\colon \S \to \rset^{d}$ as $V_{\theta}^{\pi}(s) =  \varphi^\top(s) \theta$. Here $d$ is the dimension of feature space. Our goal is to find a parameter $\thetas$ which is defined as a unique solution to the projected Bellman equation, see \cite{tsitsiklis:td:1997}. We denote by $\mu$ the invariant distribution over the state space $\S$ induced by $\PMDP^{\pi}(\cdot | s)$ in \eqref{eq:transition_matrix_P_pi}. We define the \emph{design matrix} $\covfeat$ as
\begin{equation}
\label{eq:covfeat_matr_def}
\textstyle
\covfeat = \PE_{\mu}[\varphi(s)\varphi(s)^{\top}] \in \rset^{d \times d}\eqsp.
\end{equation}
Consider the following assumptions on the generative mechanism and on the feature mapping $\varphi(\cdot)$:
\begin{assumTD}
\label{assum:generative_model}
Tuples $(s,a,s')$ are generated \iid with $s \sim \mu$, $a \sim \pi(\cdot|s)$, $s' \sim \PMDP(\cdot|s,a)$\eqsp.
\end{assumTD}
\begin{assumTD}
\label{assum:feature_design}
Matrix $\covfeat$ is non-degenerate with the minimal eigenvalue $\lambda_{\min}(\covfeat) > 0$. Moreover, the feature mapping $\varphi(\cdot)$ satisfies 
$\sup_{s \in \S} \|\varphi(s) \| \leq 1$.
\end{assumTD}
In the setting of linear function approximation the estimation of $V^{\pi}(s)$ reduces to estimating $\thetas \in \rset^{d}$, which can be done via the LSA procedure. Here, the $k$-th step randomness is given by the tuple $\State_k= (s_k, a_k, s'_{k})$. Then, the corresponding LSA update can be written as
\begin{equation}
\label{eq:LSA_procedure_TD}
\textstyle 
\theta_{k} = \theta_{k-1} - \alpha_{k} (\funcAw_{k} \theta_{k-1} - \funcbw_{k})\eqsp,
\end{equation}
where $\funcAw_{k}$ and $\funcbw_{k}$ are given, respectively, by
\begin{equation}
\label{eq:matr_A_def}
\textstyle \funcAw_{k} =  \varphi(s_k)\{\varphi(s_k) - \gamma \varphi(s'_k)\}^{\top}\eqsp, \quad \funcbw_{k} = \textstyle \varphi(s_k) r(s_k,a_k)\eqsp.
\end{equation}
We provide the expressions for the corresponding system matrix \( \bA = \mathbb{E}[\funcAw_{k}] \) and the right-hand side \( \barb \) in \Cref{appendix:td_learning}. We verify that assumption \Cref{assum:noise-level} holds and, furthermore, we provide a tighter counterpart to the result of \Cref{prop:hurwitz_stability}. This result closely follows \cite{patil2023finite} and \cite{samsonov2023finite}.
\begin{proposition}
\label{prop:assumption_check_TD}
Let $\{\theta\}_{k \in \nset}$ be a sequence of TD updates generated by \eqref{eq:LSA_procedure_TD} under \Cref{assum:generative_model} and \Cref{assum:feature_design}. Then this update scheme satisfies assumption \Cref{assum:noise-level} with
\begin{align}
\label{eq:condition-check-1}\bConst{A} = 2 (1+\gamma)\eqsp, \quad \supconsteps = 2 (1+\gamma)(\norm{\thetas} + 1)\eqsp,
\end{align}
moreover, one can check that $\normop{\Id - \alpha \barA}^2 \leq 1 - \alpha a$ with 
\begin{equation}
\label{eq:a_alpha_infty_td}
a = (1-\gamma)\lambda_{\min}(\covfeat) \eqsp, \quad \alpha_{\infty} = (1-\gamma)/(1+\gamma)^2\eqsp,
\end{equation}
that is, \Cref{prop:hurwitz_stability} holds with $Q = \Id$.
\end{proposition}
Proof of \Cref{prop:assumption_check_TD} is provided in \Cref{appendix:td_learning}. Since all the assumptions in \Cref{assum:noise-level} are fulfilled, we can verify tightness of the bound \Cref{th:shao2022_berry} for different learning rate schedules $\alpha_{k}$ in \eqref{eq:LSA_procedure_TD}.
\vspace{-2mm}

\paragraph{Numerical results.} Efficiency of the multiplier bootstrap approach \eqref{eq:lsa_bootstrap} to the problems of constructing confidence sets in online algorithms has been demonstrated in the works \cite{JMLR:v19:17-370} and \cite{JASA2023}. We aim to illustrate the tightness of our bounds for normal approximation outlined in \Cref{th:shao2022_berry} in the setting of TD learning with linear function approximation. To this end, we consider the classical Garnet problem \cite{archibald1995generation}, in the simplified version proposed by \cite{geist2014off}. This problem is characterized by the number of states $N_s$, number of actions $a$, and branching factor $b$ (\ie\ the number of neighbors of each state in the MDP). We set these values to $N_s=10$, $a=2$ and $b=3$, and aim to evaluate the value function of the randomly generated policy $\pi(\cdot|s)$. Details on the way the policy $\pi$ is set can be found in \Cref{appendix:numeric_details}. We consider the problem of policy evaluation in this MDP using the TD learning algorithm with identity feature mapping, that is, $\phi(s) = e_{s}$ (that is, $s$-th coordinate vector) for $s \in \{1,\ldots,N_s\}$. We run the procedure \eqref{eq:LSA_procedure_TD} with the learning rates $\alpha_{k} = c_{0}/k^{\gamma}$ and different powers $\gamma \in \{0.5,0.65,0.7\}$. For each of the experiments we aim to estimate the supremum
\begin{equation}
\label{eq:approx_supremum_experiment}
\textstyle 
\Delta_{n} := \sup_{x \in \rset} \bigl|\P(\sqrt{n}\norm{\bar{\theta}_{n} - \thetas} \leq x) - \P(\norm{\Sigma_{\infty}^{1/2}\eta} \leq x)\bigr|\eqsp,
\end{equation}
$\eta \sim \mathcal{N}(0,\Id_{N_s})$, and show that this supremum scales as $n^{-1/4}$ when $\gamma = 1/2$ and admits slower decay for other powers of $\gamma$. We approximate true probability $\P(\norm{\Sigma_{\infty}^{1/2}\eta} \leq x)$ by the corresponding empirical probabilities based on sample of size $M \gg n$. Second, for $n \in \{1600, \ldots, 1638400\}$, where next sample size is twice larger than the previous one, we generate $N = 6553600$ trajectories of TD algorithm and approximate the distribution of $\sqrt{n}\norm{\bar{\theta}_{n} - \thetas}$ based on the corresponding empirical distribution. We report our results in \Cref{fig:results-expe}, showing that the smallest values of $\Delta_{n}$ correspond to the step size schedule $\gamma = 1/2$, moreover, the decay rate $n^{-1/4}$ seems to be tight, otherwise one should expect further decay of $\Delta_{n} n^{1/4}$. Additional simulations are provided in \Cref{appendix:numeric_details}. 

\begin{figure}[t!]
\centering
\subfigure{\includegraphics[width=1.0\linewidth]{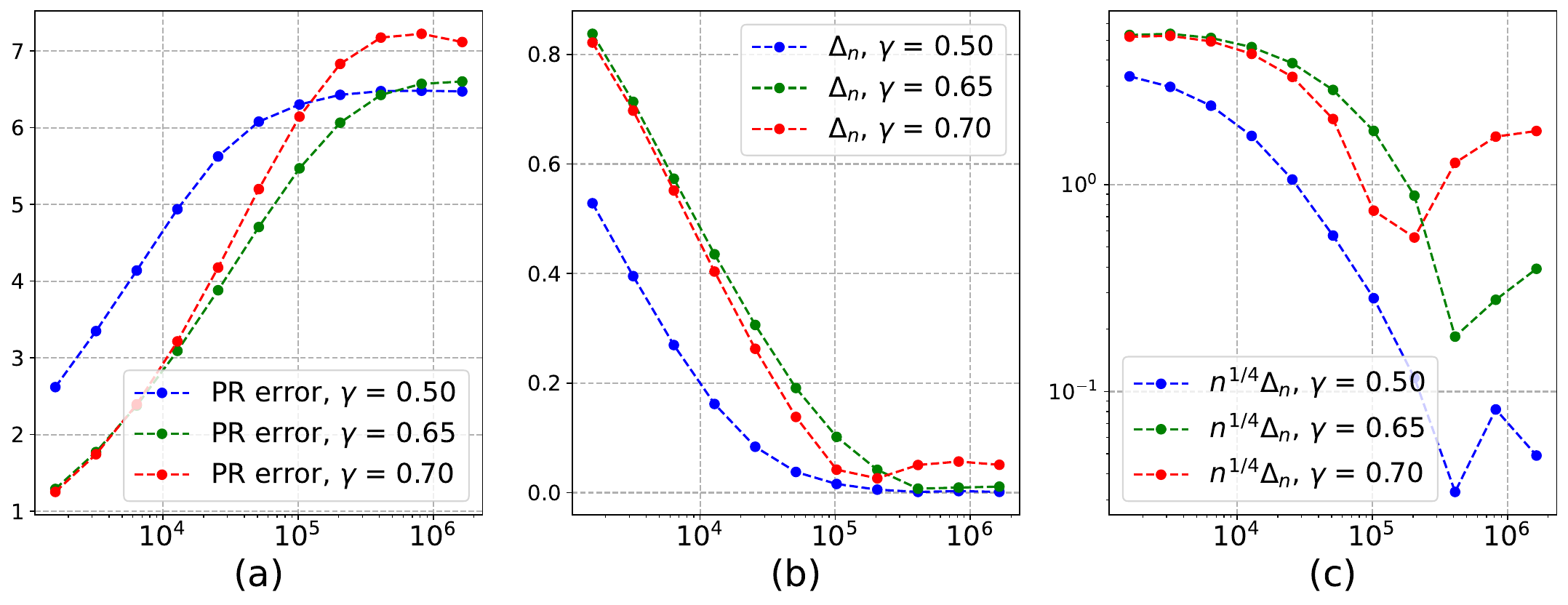}}%
\vspace{-3mm}
\caption{Subfigure (a): Rescaled error $\sqrt{n}\norm{\bar{\theta}_{n} - \thetas}$, averaged over $N$ independent TD trajectories for different trajectory lengths $n$. Subfigure (b): approximate quantity $\Delta_n$ from \eqref{eq:approx_supremum_experiment} for different powers $\gamma$ and $n$. Subfigure (c): $\Delta_n$, rescaled by a factor $n^{1/4}$, predicted by \Cref{th:shao2022_berry}.}
\label{fig:results-expe}
\vspace{-3mm}
\end{figure}

\vspace{-3mm}
\section{Conclusion}
\label{sec:conclusion}
In this paper, we have established, to the best of our knowledge, the first fully non-asymptotic confidence bounds for parameter estimation in the LSA algorithm using the multiplier bootstrap. This result is based on a novel Berry-Esseen bound for the Polyak-Ruppert averaged LSA iterates, which is of independent interest. Our paper suggests several interesting directions for further research. First, our Berry-Esseen bounds are obtained using the randomized concentration inequality \cite{shao2022berry}, and it would be valuable to generalize this approach to the setting of Markov chains. Second, it is natural to extend our results to the first-order gradient methods, both for stochastic optimization and variational inequalities. Third, it becomes possible to prove the fully non-asymptotic validity of confidence intervals obtained with plug-in techniques or other estimators of the asymptotic covariance matrix of $\prtheta_{n}$. These could then be compared with the multiplier bootstrap confidence intervals in terms of their dependence on problem dimension $d$ and other instance-dependent quantities. 

\section*{Acknowledgement}

The work of S. Samsonov and A. Naumov was prepared within the framework of the HSE University Basic Research Program. The work of E. Moulines has been partly funded by the European Union (ERC-2022-SYG-OCEAN-101071601). Views and opinions expressed are however those of the author(s) only and do not necessarily reflect those of the European Union or the European Research Council Executive Agency. Neither the European Union nor the granting authority can be held responsible for them. The work of Q.-M. Shao is partially supported by National Nature Science Foundation of China NSFC 12031005 and Shenzhen Outstanding Talents Training Fund, China. The work of Z.-S. Zhang is partially supported by National Nature Science Foundation of China NSFC 12301183 and National Nature Science Found for Excellent Young Scientists Fund. This research was supported in part through computational resources of HPC facilities at HSE University \cite{kostenetskiy2021hpc}.

\clearpage
\newpage

\bibliography{references}

\begin{thebibliography}{10}

\bibitem{aguech2000perturbation}
Rafik Aguech, Eric Moulines, and Pierre Priouret.
\newblock On a perturbation approach for the analysis of stochastic tracking
  algorithms.
\newblock {\em SIAM Journal on Control and Optimization}, 39(3):872--899, 2000.

\bibitem{pmlr-v99-anastasiou19a}
Andreas Anastasiou, Krishnakumar Balasubramanian, and Murat~A. Erdogdu.
\newblock Normal approximation for stochastic gradient descent via
  non-asymptotic rates of martingale {CLT}.
\newblock In Alina Beygelzimer and Daniel Hsu, editors, {\em Proceedings of the
  Thirty-Second Conference on Learning Theory}, volume~99 of {\em Proceedings
  of Machine Learning Research}, pages 115--137. PMLR, 25--28 Jun 2019.

\bibitem{archibald1995generation}
TW~Archibald, KIM McKinnon, and LC~Thomas.
\newblock On the generation of markov decision processes.
\newblock {\em Journal of the Operational Research Society}, 46(3):354--361,
  1995.

\bibitem{auer2002finite}
Peter Auer, Nicolo Cesa-Bianchi, and Paul Fischer.
\newblock Finite-time analysis of the multiarmed bandit problem.
\newblock {\em Machine learning}, 47:235--256, 2002.

\bibitem{bach:moulines:2013}
F.~Bach and E.~Moulines.
\newblock Non-strongly-convex smooth stochastic approximation with convergence
  rate o(1/n).
\newblock In C.~J.~C. Burges, L.~Bottou, M.~Welling, Z.~Ghahramani, and K.~Q.
  Weinberger, editors, {\em Advances in Neural Information Processing Systems},
  volume~26. Curran Associates, Inc., 2013.

\bibitem{benveniste2012adaptive}
A.~Benveniste, M.~M{\'e}tivier, and P.~Priouret.
\newblock {\em Adaptive algorithms and stochastic approximations}, volume~22.
\newblock Springer Science \& Business Media, 2012.

\bibitem{bhandari2018finite}
J.~Bhandari, D.~Russo, and R.~Singal.
\newblock A finite time analysis of temporal difference learning with linear
  function approximation.
\newblock In {\em Conference On Learning Theory}, pages 1691--1692, 2018.

\bibitem{bolthausen1982}
E.~Bolthausen.
\newblock {Exact Convergence Rates in Some Martingale Central Limit Theorems}.
\newblock {\em The Annals of Probability}, 10(3):672 -- 688, 1982.

\bibitem{borkar:sa:2008}
Vivek~S Borkar.
\newblock {\em Stochastic Approximation: A Dynamical Systems Viewpoint}.
\newblock Cambridge University Press, 2008.

\bibitem{brosse2018pitfalls}
Nicolas Brosse, Alain Durmus, and Eric Moulines.
\newblock The promises and pitfalls of stochastic gradient langevin dynamics.
\newblock In S.~Bengio, H.~Wallach, H.~Larochelle, K.~Grauman, N.~Cesa-Bianchi,
  and R.~Garnett, editors, {\em Advances in Neural Information Processing
  Systems}, volume~31. Curran Associates, Inc., 2018.

\bibitem{chen2020aos}
Xi~Chen, Jason~D. Lee, Xin~T. Tong, and Yichen Zhang.
\newblock {Statistical inference for model parameters in stochastic gradient
  descent}.
\newblock {\em The Annals of Statistics}, 48(1):251 -- 273, 2020.

\bibitem{Chernozhukov2013}
Victor Chernozhukov, Denis Chetverikov, and Kengo Kato.
\newblock Gaussian approximations and multiplier bootstrap for maxima of sums
  of high-dimensional random vectors.
\newblock {\em Ann. Statist.}, 41(6):2786--2819, 2013.

\bibitem{Chernozhukov2015}
Victor Chernozhukov, Denis Chetverikov, and Kengo Kato.
\newblock Central limit theorems and bootstrap in high dimensions.
\newblock {\em Ann. Probab.}, 45(4):2309--2352, 2017.

\bibitem{dalal2019tale}
G.~Dalal, Balazs Szorenyi, and G.~Thoppe.
\newblock A tale of two-timescale reinforcement learning with the tightest
  finite-time bound.
\newblock {\em arXiv preprint arXiv:1911.09157}, 2019.

\bibitem{duchi2012ergodic}
John~C Duchi, Alekh Agarwal, Mikael Johansson, and Michael~I Jordan.
\newblock Ergodic mirror descent.
\newblock {\em SIAM Journal on Optimization}, 22(4):1549--1578, 2012.

\bibitem{durmus2022finite}
Alain Durmus, Eric Moulines, Alexey Naumov, and Sergey Samsonov.
\newblock Finite-time high-probability bounds for {P}olyak-{R}uppert averaged
  iterates of linear stochastic approximation.
\newblock {\em Mathematics of Operations Research}, 2024.

\bibitem{durmus2021tight}
Alain Durmus, Eric Moulines, Alexey Naumov, Sergey Samsonov, Kevin Scaman, and
  Hoi-To Wai.
\newblock Tight high probability bounds for linear stochastic approximation
  with fixed stepsize.
\newblock In M.~Ranzato, A.~Beygelzimer, K.~Nguyen, P.~S. Liang, J.~W. Vaughan,
  and Y.~Dauphin, editors, {\em Advances in Neural Information Processing
  Systems}, volume~34, pages 30063--30074. Curran Associates, Inc., 2021.

\bibitem{durmus2021stability}
Alain Durmus, Eric Moulines, Alexey Naumov, Sergey Samsonov, and Hoi-To Wai.
\newblock On the stability of random matrix product with markovian noise:
  Application to linear stochastic approximation and td learning.
\newblock In Mikhail Belkin and Samory Kpotufe, editors, {\em Proceedings of
  Thirty Fourth Conference on Learning Theory}, volume 134 of {\em Proceedings
  of Machine Learning Research}, pages 1711--1752. PMLR, 15--19 Aug 2021.

\bibitem{efron1992bootstrap}
Bradley Efron.
\newblock Bootstrap methods: another look at the jackknife.
\newblock In {\em Breakthroughs in statistics: Methodology and distribution},
  pages 569--593. Springer, 1992.

\bibitem{esseen1945}
Carl-Gustav Esseen.
\newblock {Fourier analysis of distribution functions. A mathematical study of
  the Laplace-Gaussian law}.
\newblock {\em Acta Mathematica}, 77(none):1 -- 125, 1945.

\bibitem{eweda:macchi:1983}
E.~Eweda and O.~Macchi.
\newblock Quadratic mean and almost-sure convergence of unbounded stochastic
  approximation algorithms with correlated observations.
\newblock {\em Ann. Inst. H. Poincar\'{e} Sect. B (N.S.)}, 19(3):235--255,
  1983.

\bibitem{JMLR:v19:17-370}
Yixin Fang, Jinfeng Xu, and Lei Yang.
\newblock Online bootstrap confidence intervals for the stochastic gradient
  descent estimator.
\newblock {\em Journal of Machine Learning Research}, 19(78):1--21, 2018.

\bibitem{fort:clt:markov:2015}
{G. Fort}.
\newblock Central limit theorems for stochastic approximation with controlled
  {M}arkov chain dynamics.
\newblock {\em ESAIM: PS}, 19:60--80, 2015.

\bibitem{gaunt2023bounding}
Robert~E Gaunt and Siqi Li.
\newblock Bounding {K}olmogorov distances through {W}asserstein and related
  integral probability metrics.
\newblock {\em Journal of Mathematical Analysis and Applications},
  522(1):126985, 2023.

\bibitem{geist2014off}
Matthieu Geist, Bruno Scherrer, et~al.
\newblock Off-policy learning with eligibility traces: a survey.
\newblock {\em J. Mach. Learn. Res.}, 15(1):289--333, 2014.

\bibitem{GoodBengCour16}
Ian~J. Goodfellow, Yoshua Bengio, and Aaron Courville.
\newblock {\em Deep Learning}.
\newblock MIT Press, Cambridge, MA, USA, 2016.
\newblock \url{http://www.deeplearningbook.org}.

\bibitem{Bernolli2019}
Friedrich G\"{o}tze, Alexey Naumov, Vladimir Spokoiny, and Vladimir Ulyanov.
\newblock Large ball probabilities, {G}aussian comparison and
  anti-concentration.
\newblock {\em Bernoulli}, 25(4A):2538--2563, 2019.

\bibitem{hao2019_bootstrap_ucb}
Botao Hao, Yasin Abbasi~Yadkori, Zheng Wen, and Guang Cheng.
\newblock Bootstrapping upper confidence bound.
\newblock In H.~Wallach, H.~Larochelle, A.~Beygelzimer, F.~d\textquotesingle
  Alch\'{e}-Buc, E.~Fox, and R.~Garnett, editors, {\em Advances in Neural
  Information Processing Systems}, volume~32. Curran Associates, Inc., 2019.

\bibitem{huang2020matrix}
De~Huang, Jonathan Niles-Weed, Joel~A Tropp, and Rachel Ward.
\newblock Matrix concentration for products.
\newblock {\em Foundations of Computational Mathematics}, pages 1--33, 2021.

\bibitem{huo2023bias}
Dongyan Huo, Yudong Chen, and Qiaomin Xie.
\newblock Bias and extrapolation in markovian linear stochastic approximation
  with constant stepsizes.
\newblock In {\em Abstract Proceedings of the 2023 ACM SIGMETRICS International
  Conference on Measurement and Modeling of Computer Systems}, pages 81--82,
  2023.

\bibitem{jirak2022quantitative}
Moritz Jirak and Martin Wahl.
\newblock Quantitative limit theorems and bootstrap approximations for
  empirical spectral projectors.
\newblock {\em Probability Theory and Related Fields}, 190(1):119--177, 2024.

\bibitem{kingma2014adam}
Diederik~P Kingma and Jimmy Ba.
\newblock Adam: A method for stochastic optimization.
\newblock {\em arXiv preprint arXiv:1412.6980}, 2014.

\bibitem{kostenetskiy2021hpc}
PS~Kostenetskiy, RA~Chulkevich, and VI~Kozyrev.
\newblock Hpc resources of the higher school of economics.
\newblock In {\em Journal of Physics: Conference Series}, volume 1740, page
  012050. IOP Publishing, 2021.

\bibitem{kushner2003stochastic}
Harold Kushner and G~George Yin.
\newblock {\em Stochastic approximation and recursive algorithms and
  applications}, volume~35.
\newblock Springer Science \& Business Media, 2003.

\bibitem{lakshminarayanan2018linear}
C.~Lakshminarayanan and C.~Szepesvari.
\newblock Linear stochastic approximation: How far does constant step-size and
  iterate averaging go?
\newblock In {\em International Conference on Artificial Intelligence and
  Statistics}, pages 1347--1355, 2018.

\bibitem{lan2012optimal}
Guanghui Lan.
\newblock An optimal method for stochastic composite optimization.
\newblock {\em Mathematical Programming}, 133(1-2):365--397, 2012.

\bibitem{LEE2024105673}
Sokbae Lee, Yuan Liao, Myung~Hwan Seo, and Youngki Shin.
\newblock Fast inference for quantile regression with tens of millions of
  observations.
\newblock {\em Journal of Econometrics}, page 105673, 2024.

\bibitem{li2023sharp}
Gen Li, Weichen Wu, Yuejie Chi, Cong Ma, Alessandro Rinaldo, and Yuting Wei.
\newblock High-probability sample complexities for policy evaluation with
  linear function approximation.
\newblock {\em IEEE Transactions on Information Theory}, 70(8):5969--5999,
  2024.

\bibitem{pmlr-v178-li22b}
Xiang Li, Jiadong Liang, Xiangyu Chang, and Zhihua Zhang.
\newblock Statistical estimation and online inference via local sgd.
\newblock In Po-Ling Loh and Maxim Raginsky, editors, {\em Proceedings of
  Thirty Fifth Conference on Learning Theory}, volume 178 of {\em Proceedings
  of Machine Learning Research}, pages 1613--1661. PMLR, 02--05 Jul 2022.

\bibitem{li2023online}
Xiang Li, Jiadong Liang, and Zhihua Zhang.
\newblock Online statistical inference for nonlinear stochastic approximation
  with {M}arkovian data.
\newblock {\em arXiv preprint arXiv:2302.07690}, 2023.

\bibitem{li2023statistical}
Xiang Li, Wenhao Yang, Jiadong Liang, Zhihua Zhang, and Michael~I Jordan.
\newblock A statistical analysis of {P}olyak-{R}uppert averaged {Q}-learning.
\newblock In {\em International Conference on Artificial Intelligence and
  Statistics}, pages 2207--2261. PMLR, 2023.

\bibitem{mnih2015}
Volodymyr Mnih, Koray Kavukcuoglu, David Silver, Andrei~A. Rusu, Joel Veness,
  Marc~G. Bellemare, Alex Graves, Martin Riedmiller, Andreas~K. Fidjeland,
  Georg Ostrovski, Stig Petersen, Charles Beattie, Amir Sadik, Ioannis
  Antonoglou, Helen King, Dharshan Kumaran, Daan Wierstra, Shane Legg, and
  Demis Hassabis.
\newblock Human-level control through deep reinforcement learning.
\newblock {\em Nature}, 518(7540):529--533, 2015.

\bibitem{mou2020linear}
Wenlong Mou, Chris~Junchi Li, Martin~J Wainwright, Peter~L Bartlett, and
  Michael~I Jordan.
\newblock On linear stochastic approximation: {F}ine-grained {P}olyak-{R}uppert
  and non-asymptotic concentration.
\newblock In {\em Conference on Learning Theory}, pages 2947--2997. PMLR, 2020.

\bibitem{mou2021optimal}
Wenlong Mou, Ashwin Pananjady, Martin~J Wainwright, and Peter~L Bartlett.
\newblock Optimal and instance-dependent guarantees for markovian linear
  stochastic approximation.
\newblock {\em Mathematical Statistics and Learning}, 7(1):41--153, 2024.

\bibitem{moulines2011non}
Eric Moulines and Francis Bach.
\newblock Non-asymptotic analysis of stochastic approximation algorithms for
  machine learning.
\newblock {\em Advances in neural information processing systems}, 24:451--459,
  2011.

\bibitem{PTRF2019}
Alexey Naumov, Vladimir Spokoiny, and Vladimir Ulyanov.
\newblock Bootstrap confidence sets for spectral projectors of sample
  covariance.
\newblock {\em Probab. Theory Related Fields}, 174(3-4):1091--1132, 2019.

\bibitem{nemirovski2009robust}
Arkadi Nemirovski, Anatoli Juditsky, Guanghui Lan, and Alexander Shapiro.
\newblock Robust stochastic approximation approach to stochastic programming.
\newblock {\em SIAM Journal on optimization}, 19(4):1574--1609, 2009.

\bibitem{nemirovskij1983problem}
Arkadij~Semenovi{\v{c}} Nemirovskij and David~Borisovich Yudin.
\newblock Problem complexity and method efficiency in optimization.
\newblock 1983.

\bibitem{nourdin2022multivariate}
Ivan Nourdin, Giovanni Peccati, and Xiaochuan Yang.
\newblock Multivariate normal approximation on the wiener space: new bounds in
  the convex distance.
\newblock {\em Journal of Theoretical Probability}, 35(3):2020--2037, 2022.

\bibitem{osekowski:2012}
A.~Osekowski.
\newblock {\em Sharp Martingale and Semimartingale Inequalities}.
\newblock Monografie Matematyczne 72. Birkhäuser Basel, 1 edition, 2012.

\bibitem{patil2023finite}
Gandharv Patil, LA~Prashanth, Dheeraj Nagaraj, and Doina Precup.
\newblock Finite time analysis of temporal difference learning with linear
  function approximation: Tail averaging and regularisation.
\newblock In {\em International Conference on Artificial Intelligence and
  Statistics}, pages 5438--5448. PMLR, 2023.

\bibitem{petrov1975sums}
V.~Petrov.
\newblock {\em Sums of Independent Random Variables}.
\newblock Ergebnisse der Mathematik und ihrer Grenzgebiete. 2. Folge. Springer
  Berlin Heidelberg, 1975.

\bibitem{polyak1992acceleration}
Boris~T Polyak and Anatoli~B Juditsky.
\newblock Acceleration of stochastic approximation by averaging.
\newblock {\em SIAM journal on control and optimization}, 30(4):838--855, 1992.

\bibitem{poznyak:control}
A.~S. Poznyak.
\newblock {\em Advanced Mathematical Tools for Automatic Control Engineers:
  Deterministic Techniques}.
\newblock Elsevier, Oxford, 2008.

\bibitem{rakhlin2012making}
Alexander Rakhlin, Ohad Shamir, and Karthik Sridharan.
\newblock Making gradient descent optimal for strongly convex stochastic
  optimization.
\newblock In {\em Proceedings of the 29th International Coference on
  International Conference on Machine Learning}, pages 1571--1578, 2012.

\bibitem{JASA2023}
Pratik Ramprasad, Yuantong Li, Zhuoran Yang, Zhaoran Wang, Will~Wei Sun, and
  Guang Cheng.
\newblock Online bootstrap inference for policy evaluation in reinforcement
  learning.
\newblock {\em J. Amer. Statist. Assoc.}, 118(544):2901--2914, 2023.

\bibitem{ross2011stein}
Nathan Ross.
\newblock {Fundamentals of Stein's method}.
\newblock {\em Probability Surveys}, 8(none):210 -- 293, 2011.

\bibitem{rubin1981bayesian}
Donald~B Rubin.
\newblock The bayesian bootstrap.
\newblock {\em The annals of statistics}, pages 130--134, 1981.

\bibitem{ruppert1988efficient}
David Ruppert.
\newblock Efficient estimations from a slowly convergent robbins-monro process.
\newblock Technical report, Cornell University Operations Research and
  Industrial Engineering, 1988.

\bibitem{russo2014learning}
Daniel Russo and Benjamin Van~Roy.
\newblock Learning to optimize via posterior sampling.
\newblock {\em Mathematics of Operations Research}, 39(4):1221--1243, 2014.

\bibitem{samsonov2023finite}
Sergey Samsonov, Daniil Tiapkin, Alexey Naumov, and Eric Moulines.
\newblock Improved {H}igh-{P}robability {B}ounds for the {T}emporal
  {D}ifference {L}earning {A}lgorithm via {E}xponential {S}tability.
\newblock In Shipra Agrawal and Aaron Roth, editors, {\em Proceedings of Thirty
  Seventh Conference on Learning Theory}, volume 247 of {\em Proceedings of
  Machine Learning Research}, pages 4511--4547. PMLR, 30 Jun--03 Jul 2024.

\bibitem{shao2003mathematical}
Jun Shao.
\newblock {\em Mathematical statistics}.
\newblock Springer Science \& Business Media, 2003.

\bibitem{shao2022berry}
Qi-Man Shao and Zhuo-Song Zhang.
\newblock Berry--{E}sseen bounds for multivariate nonlinear statistics with
  applications to {M}-estimators and stochastic gradient descent algorithms.
\newblock {\em Bernoulli}, 28(3):1548--1576, 2022.

\bibitem{spokoiny2015}
Vladimir Spokoiny and Mayya Zhilova.
\newblock {Bootstrap confidence sets under model misspecification}.
\newblock {\em The Annals of Statistics}, 43(6):2653 -- 2675, 2015.

\bibitem{srikant2024rates}
R~Srikant.
\newblock Rates of convergence in the central limit theorem for markov chains,
  with an application to {TD} learning.
\newblock {\em arXiv preprint arXiv:2401.15719}, 2024.

\bibitem{sutton1988learning}
R.~S Sutton.
\newblock Learning to predict by the methods of temporal differences.
\newblock {\em Machine learning}, 3(1):9--44, 1988.

\bibitem{sutton:book:2018}
R.~S. Sutton and Andrew~G. Barto.
\newblock {\em Reinforcement Learning: An Introduction}.
\newblock The MIT Press, second edition, 2018.

\bibitem{MR2802042}
Joel~A. Tropp.
\newblock Freedman's inequality for matrix martingales.
\newblock {\em Electron. Commun. Probab.}, 16:262--270, 2011.

\bibitem{tropp2015introduction}
Joel~A Tropp et~al.
\newblock An introduction to matrix concentration inequalities.
\newblock {\em Foundations and Trends{\textregistered} in Machine Learning},
  8(1-2):1--230, 2015.

\bibitem{tsitsiklis:td:1997}
J.~N. {Tsitsiklis} and B.~{Van Roy}.
\newblock An analysis of temporal-difference learning with function
  approximation.
\newblock {\em IEEE Transactions on Automatic Control}, 42(5):674--690, May
  1997.

\bibitem{van1996weak}
A.~W. Van Der~Vaart and J.~A. Wellner.
\newblock {\em Weak convergence and empirical processes}.
\newblock Springer Series in Statistics, 1996.

\bibitem{vapnik2013nature}
Vladimir Vapnik.
\newblock {\em The nature of statistical learning theory}.
\newblock Springer science \& business media, 2013.

\bibitem{zhu2023online_cov_matr}
Xi~Chen Wanrong~Zhu and Wei~Biao Wu.
\newblock Online {C}ovariance {M}atrix {E}stimation in {S}tochastic {G}radient
  {D}escent.
\newblock {\em Journal of the American Statistical Association},
  118(541):393--404, 2023.

\bibitem{zhong2023online}
Yanjie Zhong, Todd Kuffner, and Soumendra Lahiri.
\newblock Online {B}ootstrap {I}nference with {N}onconvex {S}tochastic
  {G}radient {D}escent {E}stimator.
\newblock {\em arXiv preprint arXiv:2306.02205}, 2023.

\bibitem{zolotarev1984probability}
Vladimir~Mikhailovich Zolotarev.
\newblock Probability metrics.
\newblock {\em Theory of Probability \& Its Applications}, 28(2):278--302,
  1984.

\end{thebibliography}
\bibliographystyle{plain}

\clearpage
\newpage
\appendix

\newpage
\section{Proofs for accuracy of normal approximation}
\label{appendix:proofs}
\subsection{Expansion of the error of LSA equipped with the Polyak-Ruppert averaging}

\begin{proposition}
\label{prop: expansion} The following expansion holds:
\begin{multline}
\label{eq:error_decomposition}
\sqrt{n}\bA(\bar{\theta}_{n} - \thetas) = -\underbrace{\frac{1}{\sqrt{n}}\sum_{k=n+1}^{2n}\funnoisew_{k}}_{W} +  \underbrace{\frac{1}{\sqrt{n}}\frac{\theta_{n}-\thetas}{\alpha_{n}}}_{D_1} - \underbrace{\frac{1}{\sqrt{n}}\frac{\theta_{2n}-\thetas}{\alpha_{2n}}}_{D_2} \\
-\underbrace{\frac{1}{\sqrt{n}}\sum_{k=n+1}^{2n}(\funcAw_k - \bA)(\theta_{k-1} - \thetas)}_{D_3}
+\underbrace{\frac{1}{\sqrt{n}}\sum_{k=n+1}^{2n}\bigl(\theta_{k-1} - \thetas\bigr)\left(\frac{1}{\alpha_k} - \frac{1}{\alpha_{k-1}}\right)}_{D_4}
\end{multline}    
\end{proposition}
\begin{proof}
We use the recurrence \eqref{eq:main_recurrence_1_step} and rewrite it as
\begin{equation}
\label{eq: last iterate decomp}
\theta_{k} - \thetas = (\Id - \alpha_{k} \bA)(\theta_{k-1} - \thetas) - \alpha_{k}(\funcAw_k - \bA)(\theta_{k-1} - \thetas) - \alpha_{k} \funnoisew_{k}\eqsp.
\end{equation}
The previous equation implies, after algebraic manipulation and division by $\alpha_{k}$, that
\[
\bA(\theta_{k-1} - \thetas) = \frac{\theta_{k-1} - \thetas}{\alpha_k} - \frac{\theta_{k} - \thetas}{\alpha_k} - (\funcAw_k - \bA)(\theta_{k-1} - \thetas) - \funnoisew_{k}\eqsp.
\]
Taking average for $k$ from $n+1$ to $2n$ and multiplying by $\sqrt{n}$, we obtain \eqref{eq:error_decomposition}.
\end{proof}

\subsection{Bounding the error of the LSA algorithm last iterate}
\label{sec:last_moment_bound_lsa}

We begin with of technical lemma on the behavior of the last iterate $\theta_k$ of the LSA procedure given in \eqref{eq:lsa}. We aim to show that $\PE^{1/p}[\norm{\theta_k - \thetas}^p]$ scales as $\sqrt{\alpha_k}$, provided that $k$ is large enough.
This result is classical and appears in a number of papers, e.g.  \cite{bhandari2018finite,dalal2019tale,mou2020linear,durmus2021tight}. We provide the proof here for completeness. Our analysis of the bootstrap procedure and the last iterate error of LSA procedure is based on the error expansion technique from \cite{aguech2000perturbation}, see also \cite{durmus2022finite}. Namely, to perform the expansion, we decompose the LSA iterates $\theta_k$ defined in \eqref{eq:lsa} into a transient and fluctuation terms:
\[
\theta_{k} - \thetas = \utheta_{k} + \vtheta_{k}\eqsp,
\]
where we have defined the quantities
\begin{equation}
\label{eq:LSA_recursion_expanded}
\utheta_{k} = \ProdB_{1:k} \{ \theta_0 - \thetas \} \eqsp, \quad \vtheta_{k} = - \sum_{j=1}^{k} \alpha_{j} \ProdB_{j+1:k} \funnoisew_j\eqsp,
\end{equation}
setting 
\begin{equation}
\label{eq:prod_rand_matr}
\ProdB_{m:k} = \prod_{i=m}^{k} (\Id - \alpha_{i} \funcA{Z_i} ) \eqsp, \quad m,k \in\nset, m \leq k \eqsp, \text{ with the convention } \ProdB_{m:k} = \Id\eqsp, m > k\eqsp.
\end{equation}
The dependence of $\ProdB_{m:k}$ upon the stepsizes $(\alpha_j)$ is implicit in \eqref{eq:prod_rand_matr}. Here the quantity $\utheta_{k}$ is the transient component of the error, which determines the rate at which the initial error $\theta_0 - \thetas$ is forgotten. The term $\vtheta_{k}$ corresponds to the fluctuation component of the error and is determined by the oscillations of the last iterate $\theta_{k}$ around $\thetas$.

\begin{proposition}
\label{lem:last_moment_bound}
Assume \Cref{assum:iid}, \Cref{assum:noise-level}, and \Cref{assum:step-size}. Then for any $k \geq n$, where $n$ satisfies \eqref{eq:sample_size_bound}, it holds for $2 \leq p \leq \log{n^2}$, that 
\begin{equation}
\label{eq:last_iter_bound}
\PE^{1/p}[\norm{\theta_k - \thetas}^p] \leq \sqrt{\qcond} \rme \exp\bigl\{- (a/2) \sum_{\ell=1}^{k}\alpha_\ell\bigr\}\norm{\theta_0 - \thetas} + \frac{4\rme \sqrt{\qcond} \supconsteps p}{\sqrt{a}} \sqrt{\alpha_{k}}\eqsp.
\end{equation}
\end{proposition}
\begin{proof}
Expanding the decomposition \eqref{eq:LSA_recursion_expanded}, we obtain that 
\begin{equation}
\label{eq:p-norm-minkowski}
\PE^{1/p}[\norm{\theta_k - \thetas}^p] \leq \PE^{1/p}[\norm{\ProdB_{1:k} \{ \theta_0 - \thetas \}}^{p}] + \PE^{1/p}[\norm{\sum_{j=1}^k \alpha_{j} \ProdB_{j+1:k} \funnoisew_j}^{p}]\eqsp,
\end{equation}
and we bound both terms separately. Since the sample size $n$ satisfies \eqref{eq:sample_size_bound}, we get applying \Cref{cor:exp_bound_decay} (see equation \eqref{eq:concentration_iid_cor}), that for $2 \leq p \leq \log{n^2}$ it holds
\begin{equation}
\label{eq:transient_term_bound}
\PE^{1/p}[\norm{\ProdB_{1:k} \{ \theta_0 - \thetas \}}^{p}] \leq \sqrt{\qcond} \rme \exp\bigl\{- (a/2) \sum_{\ell=1}^{k}\alpha_\ell\bigr\}\norm{\theta_0 - \thetas}\eqsp.
\end{equation}
Now we proceed with the second term in \eqref{eq:p-norm-minkowski}. Applying Burholder's inequality \cite[Theorem 8.6]{osekowski:2012} and \Cref{lem:summ_alpha_k_squared} with $b = a/4$, we obtain that 
\begin{align}
\PE^{1/p}[\norm{\sum_{j=1}^k \alpha_{j} \ProdB_{j+1:k} \funnoisew_j}^{p}] &\leq p \left(\PE^{2/p}\left[\left(\sum\nolimits_{j=1}^{k}\alpha_{j}^2 \normop{\ProdB_{j+1:k} \funnoisew_j}^{2}\right)^{p/2}\right]\right)^{1/2} \\
& \leq  p \left(\sum\nolimits_{j=1}^{k}\alpha_{j}^{2}\PE^{2/p}\bigl[\normop{\ProdB_{j+1:k} \funnoisew_j}^{p} \bigr]\right)^{1/2} \\
& \leq p \sqrt{\qcond} \rme \supconsteps \biggl(\sum\nolimits_{j=1}^{k} \alpha_{j}^{2} \prod_{\ell=j+1}^{k}\bigl(1 - \frac{a \alpha_{\ell}}{4}\bigr) \biggr)^{1/2} \\
& \leq \frac{4\rme \sqrt{\qcond}\supconsteps p}{\sqrt{a}} \sqrt{\alpha_{k}}\eqsp.
\end{align}
\end{proof}

\begin{corollary}
\label{high prob last iterate}
Under assumptions of \Cref{lem:last_moment_bound}, it holds that
\begin{equation}
\label{eq:union_bound_last_iterate}
\PP \left( \exists k \in [n, 2n-1]: \norm{\theta_k - \thetas} \geq g(k,\norm{\theta_0-\thetas},n) \right) \leq \frac{1}{n}\eqsp,
\end{equation}
where we have defined 
\[
g(k,\norm{\theta_0-\thetas},n) = \sqrt{\qcond} \rme^{2} \exp\bigl\{- (a/2) \sum_{\ell=1}^{k}\alpha_\ell\bigr\}\norm{\theta_0 - \thetas} + \frac{8 \rme^{2}\sqrt{\qcond} \supconsteps \log n}{\sqrt{a}} \sqrt{\alpha_{k}}\eqsp.
\]
\end{corollary}
\begin{proof}
We first note that \Cref{lem:markov_inequality} implies, setting $\delta = 1/n^2$, that for every fixed $k \in [n;2n-1]$, 
\[
\PP \left(\norm{\theta_k - \thetas} \geq  \sqrt{\qcond} \rme^{2} \exp\bigl\{- (a/2) \sum_{\ell=1}^{k}\alpha_\ell\bigr\}\norm{\theta_0 - \thetas} + \frac{8 \rme^{2}\sqrt{\qcond} \supconsteps \log n}{\sqrt{a}} \sqrt{\alpha_{k}} \right) \leq \frac{1}{n^2}\eqsp.
\]
Application of the union bound concludes the proof.
\end{proof}

We conclude this part with a simple consequence of Markov's inequality.
\begin{lemma}
\label{lem:markov_inequality}
Fix $\delta \in (0,1/\rme^2)$ and let $Y$ be a positive random variable, such that 
\[
\PE^{1/p}[Y^{p}] \leq C_{1} + C_{2} p 
\]
for any $2 \leq p \leq \log{(1/\delta)}$. Then it holds with probability at least $1-\delta$, that 
\begin{equation}
\label{eq:markov_ineqality_deviation}
Y \leq \rme C_{1} + \rme C_{2} \log{(1/\delta)}\eqsp.
\end{equation}
\end{lemma}
\begin{proof}
Applying Markov's inequality, for any $t \geq 0$ we get that 
\begin{align}
\PP(Y \geq t) \leq \frac{\PE[Y^{p}]}{t^{p}} \leq \frac{(C_{1} + C_{2} p)^{p}}{t^{p}}\eqsp.
\end{align}
Now we set $p = \log{(1/\delta)}$, $t = \rme C_{1} + \rme C_{2} \log{(1/\delta)}$, and aim to check that 
\[
\frac{(C_{1} + C_{2} \log{(1/\delta)})^{\log{(1/\delta)}}}{(\rme C_{1} + \rme C_{2} \log{(1/\delta)})^{ \log{(1/\delta)}}} \leq \delta\eqsp.
\]
Taking logarithms from both sides, the latter inequality is equivalent to 
\[
\log{(1/\delta)} \log\frac{C_{1} + C_{2} \log{(1/\delta)}}{\rme(C_{1} + C_{2} \log{(1/\delta)})} \leq \log{\delta}\eqsp,
\]
which turns into exact equality.
\end{proof}

\subsection{Proof of \Cref{th:theo_1_iid}}
\label{sec:proof_theo_1_iid}
We first define explicitly the remainder term outlined in the statement of \Cref{th:theo_1_iid}:
\begin{equation}
\label{eq:remainders_theo_1_iid}
\Remainder_{1}(n,a,\bConst{A},c_0) = \frac{n^{\gamma-1/2}}{c_0} + \frac{\bConst{A}}{n^{(1-\gamma)/2}\sqrt{c_0 a}} + \frac{n^{2\gamma-3/2}}{ac_0^2}\eqsp.
\end{equation} 

\begin{proof}
Since both terms in the right-hand side of the error bound of \Cref{lem:last_moment_bound} scales linearly with $\sqrt{\qcond}$, for simplicity we do not trace it in the subsequent bounds (i.e. assume $\qcond = 1$), and then keep the required scaling with $\qcond$ only in the final bounds. The decomposition \eqref{eq:error_decomposition} is a key element of our proof and allows to treat different error sources $D_1 - D_4$ separately. For the last iterate we have, using \Cref{lem:last_moment_bound}, that
\begin{align}
\label{eq:D1-D2bounds}
\PE^{1/2}\bigl[\norm{\theta_{n}-\thetas}^2\bigr] &\lesssim \frac{\supconsteps}{\sqrt{a}} \sqrt{\alpha_{n}} + \exp\biggl\{- (a/2) \sum_{\ell=1}^{n}\alpha_\ell\biggr\}\norm{\theta_0 - \thetas} \\
\PE^{1/2}\bigl[\norm{\theta_{2n}-\thetas}^2\bigr] &\lesssim \frac{ \supconsteps}{\sqrt{a}} \sqrt{\alpha_{2n}} + \exp\biggl\{- (a/2) \sum_{\ell=1}^{2n}\alpha_\ell\biggr\}\norm{\theta_0 - \thetas}\eqsp.
\end{align}
Thus, using that $\sum_{k=1}^{n}\alpha_{k} \geq \frac{c_0 (n^{1-\gamma}-1)}{1-\gamma}$ and $c_0 \leq 1-\gamma$, we obtain that 
\begin{equation}
\label{eq:D1-bound}
\begin{split}
\PE^{1/2}\bigl[\norm{D_1}^2\bigr] 
&\lesssim \frac{ \supconsteps}{\sqrt{a c_0} n^{(1-\gamma)/2}} + \frac{n^{\gamma - 1/2}}{c_0} \exp\biggl\{-\frac{c_0 a n^{1-\gamma}}{2(1-\gamma)}\biggr\} \norm{\theta_{0}-\thetas} \eqsp, \\
\PE^{1/2}\bigl[\norm{D_2}^2\bigr] 
&\lesssim \frac{ \supconsteps}{\sqrt{a c_0} n^{(1-\gamma)/2}} +  \frac{n^{\gamma - 1/2}}{c_0} \exp\biggl\{-\frac{c_0 a (2n)^{1-\gamma}}{1-\gamma}\biggr\} \norm{\theta_{0}-\thetas}\eqsp.
\end{split}
\end{equation}
Now we proceed with $D_3$. Since it is a sum of a martingale-difference sequence w.r.t. $\mcf_{k} = \sigma(\State_{\ell}, \ell \leq k)$, we get using \Cref{lem:last_moment_bound}, that  
\begin{align}
\PE\bigl[\norm{D_3}^2\bigr] 
& \lesssim \frac{\bConst{A}^2}{n}\sum_{k=n}^{2n-1}\PE[\norm{\theta_{k}-\thetas}^2] \\
& \lesssim \frac{ \bConst{A}^2}{n} \sum_{k=n+1}^{2n}\frac{\supconsteps^2\alpha_{k}}{a} + \frac{ \bConst{A}^2}{n} \sum_{k=n+1}^{2n} \exp\biggl\{-a \sum_{\ell=1}^{k}\alpha_\ell\biggr\}\norm{\theta_0 - \thetas}^2 \\
& \lesssim \frac{ \bConst{A}^2}{n} \sum_{k=n+1}^{2n}\frac{\supconsteps^2\alpha_{k}}{a} + \frac{ \bConst{A}^2}{n \alpha_{2n}} \exp\biggl\{-a \sum_{\ell=1}^{n}\alpha_\ell\biggr\} \underbrace{\sum_{k=n+1}^{2n}\alpha_{k} \exp\biggl\{-a \sum_{\ell=n+1}^{k}\alpha_\ell\biggr\}}_{S_1}\norm{\theta_0 - \thetas}^2 \\
&\lesssim \frac{c_0 \bConst{A}^2 \supconsteps^2}{a (1-\gamma) n^{\gamma}} + \frac{ \bConst{A}^2}{n^{1-\gamma} c_0 a} \exp\biggl\{-\frac{c_0 a n^{1-\gamma}}{1-\gamma}\biggr\} \norm{\theta_{0}-\thetas}^{2}\eqsp,
\end{align}
where we additionally used that $S_1 \lesssim 1/a$ due to \Cref{lem:summ_alpha_k_squared}. Now it remains to bound the term $D_4$ from the representation \eqref{eq:error_decomposition}. Using Minkowski's inequality and \Cref{lem:last_moment_bound}, we get that 
\begin{align}
\PE^{1/2}\bigl[\norm{D_4}^2\bigr]  
&\lesssim \frac{1}{\sqrt{n}}\sum_{k=n}^{2n-1}\PE^{1/2}\bigl[\norm{\theta_{k} - \thetas}^2\bigr]\left(\frac{1}{\alpha_k} - \frac{1}{\alpha_{k-1}}\right) \\
&\lesssim \frac{1}{\sqrt{n}}\sum_{k=n}^{2n-1}\frac{\supconsteps (k^{\gamma} - (k-1)^{\gamma})}{c_0 \sqrt{a}}\sqrt{\alpha_{k}} \\
&\qquad \qquad \qquad + \frac{1}{c_0 \sqrt{n}}\sum_{k=n}^{2n-1} (k^{\gamma} - (k-1)^{\gamma}) \exp\biggl\{-(a/2)\sum_{\ell=1}^{k}\alpha_{\ell}\biggr\} \norm{\theta_{0}-\thetas}\\
&\overset{(a)}{\lesssim} \frac{\supconsteps}{\sqrt{a c_0}\sqrt{n}}\sum_{k=n}^{2n-1}\frac{1}{k^{1-\gamma/2}} + \frac{n^{2\gamma-3/2}}{a c_0^2}\exp\biggl\{-\frac{c_0 a n^{1-\gamma}}{2(1-\gamma)}\biggr\} \norm{\theta_{0}-\thetas} \\
&\lesssim \frac{\supconsteps}{\sqrt{a c_0}n^{(1-\gamma)/2}} + \frac{n^{2\gamma-3/2}}{a c_0^2}\exp\biggl\{-\frac{c_0 a n^{1-\gamma}}{2(1-\gamma)}\biggr\} \norm{\theta_{0}-\thetas}\eqsp.
\end{align}
Here in (a) we additionally used that $k^{\gamma} - (k-1)^{\gamma} \lesssim k^{1-\gamma}$ together with \Cref{lem:summ_alpha_k_squared}.
Combining the estimates above yields the result of \Cref{th:theo_1_iid}.
\end{proof}

We conclude this section with some technical lemmas. 
\begin{lemma}[Lemma~24 in \cite{durmus2021stability}]
\label{lem:summ_alpha_k}
Let $b > 0$ and $(\alpha_k)_{k \geq 0}$ be a non-increasing sequence such that $\alpha_0 \leq 1/b$. Then
\[
\sum_{j=1}^{n+1} \alpha_j \prod_{l=j+1}^{n+1} (1 - \alpha_l b) = \frac{1}{b} \left\{1  - \prod_{l=1}^{n+1} (1 - \alpha_l b) \right\}
\]
\end{lemma}
\begin{proof}
The proof of this statement is given in \cite{durmus2021stability}, we provide it here for completeness. Let us denote $u_{j:n+1} = \prod_{l = j}^{n+1} (1 - \alpha_l b)$. Then, for $j \in\{1,\dots,n+1\}$,
$u_{j+1:n+1} - u_{j:n+1} = b \alpha_j u_{j+1:n+1}$. Hence,
\[
\sum_{j=1}^{n+1} \alpha_j \prod_{l=j+1}^{n+1} (1 - \alpha_l b) = \frac{1}{b} \sum_{j=1}^{n+1} (u_{j+1:n+1} - u_{j:n+1}) = b^{-1} ( 1 - u_{1:n+1} )\eqsp,
\]
and the statement follows.
\end{proof}

\begin{lemma}[Modified Lemma~25 in \cite{durmus2021stability}]
\label{lem:summ_alpha_k_squared}
Let $b > 0$ and let $\alpha_{\ell} = c_{0}/\ell^{\gamma}$, $\gamma \in [1/2;1)$, such that $c_{0} \leq 1/b$. Then for any $n$ satisfying 
\begin{equation}
\label{eq:condition-on-n}
n \geq 2 + 2 \biggl(\frac{2 \gamma}{c_0 b}\biggr)^{1/(1-\gamma)}\eqsp, \quad \text{and} \quad \frac{n^{1-\gamma}}{1 + \log(n)} \geq \frac{2 \gamma (1-\gamma)}{c_0 b (1 - (1/2)^{1-\gamma}}\eqsp,
\end{equation}
and any $k \geq n$, it holds that 
\begin{equation}
\label{eq:sum_squares_bound}
\sum_{j=1}^{k+1}\alpha_{j}^{2}\prod_{\ell=j+1}^{k+1}(1-\alpha_{\ell} b) \leq (4/b) \alpha_{k+1}\eqsp.
\end{equation}
\end{lemma}
\begin{proof}
From elementary algebra, we obtain that 
\begin{align}
\label{eq:step_size_diff_bound}
\alpha_{\ell} - \alpha_{\ell+1} = \frac{c_0}{\ell^{\gamma}} - \frac{c_0}{(\ell+1)^{\gamma}} = \frac{c_{0}((1+1/\ell)^{\gamma}-1)}{(\ell+1)^{\gamma}} \leq \frac{c_0}{(\ell+1)^{\gamma}} \frac{\gamma}{\ell}\eqsp,
\end{align}
where we used the fact that $(1+x)^{\gamma} \leq 1 + \gamma x$ for $\gamma \in [1/2;1)$ and $x \in [0,1]$. Hence, 
\[
\frac{\alpha_{\ell}}{\alpha_{\ell+1}} \leq 1 + \frac{\gamma}{\ell}\eqsp.
\]
Thus we obtain that, since $k \geq n$,
\begin{align}
\sum_{j=1}^{k+1}&\alpha_{j}^{2}\prod_{\ell=j+1}^{k+1}(1-\alpha_{\ell} b) 
= \alpha_{k+1} \sum_{j=1}^{k+1} \alpha_{j}\prod_{\ell=j+1}^{k+1} \biggl(\frac{\alpha_{\ell-1}}{\alpha_{\ell}}\biggr) (1-\alpha_{\ell} b) \\
&\leq \alpha_{k+1} \sum_{j=1}^{k+1} \alpha_{j}\prod_{\ell=j+1}^{k+1}\left(1 + \frac{\gamma}{\ell-1}\right)(1-\alpha_{\ell} b) \\
&\leq \alpha_{k+1} \sum_{j=1}^{k+1} \alpha_{j} \exp\left\{\sum_{\ell=j+1}^{n}\frac{\gamma}{\ell-1}\right\}\exp\left\{-\sum_{\ell=j+1}^{n}\alpha_{\ell}b\right\}
\exp\left\{\sum_{\ell=n+1}^{k+1}\frac{\gamma}{\ell-1}\right\}\exp\left\{-\sum_{\ell=n+1}^{k+1}\alpha_{\ell}b\right\} \\
&\leq \alpha_{k+1} \sum_{j=1}^{k+1} \alpha_{j} \exp\left\{\sum_{\ell=j+1}^{n}\frac{\gamma}{\ell-1}\right\}\exp\left\{-\sum_{\ell=j+1}^{n}\alpha_{\ell}b\right\}
\exp\left\{-\frac{b}{2} \sum_{\ell=n+1}^{k+1}\alpha_{\ell}b\right\} \eqsp.
\end{align}
In the last identity we used the fact that, since $n$ satisfies \eqref{eq:condition-on-n}, it holds for $\ell \geq n/2$ that
\begin{equation}
\label{eq:key-inequality}
\frac{\gamma}{\ell-1} \leq \alpha_{\ell} b / 2\eqsp.
\end{equation}
We will now prove that for $j \leq n-1$, it holds 
\begin{equation}
\label{eq:key-condition}
\sum_{\ell=j+1}^{n}\frac{\gamma}{\ell-1} \leq (b/2) \sum_{\ell=j+1}^{n}\alpha_{\ell}\eqsp,
\end{equation}
For $j \geq n/2$, the bound \eqref{eq:key-condition} directly follows from \eqref{eq:key-inequality}. We now turn to the proof of \eqref{eq:key-condition} for $j \leq \lceil n/2 \rceil$. Note first that
 \[
 \sum_{\ell=j+1}^n \frac{\gamma}{\ell-1} \leq \sum_{\ell=2}^n \frac{\gamma}{\ell-1} \leq \gamma (\log(n) +1) \eqsp.
 \]
On the other hand, we get
\[
\sum_{\ell=j+1}^{n}\frac{1}{\ell^{\gamma}} \geq \int_{j+1}^{n+1}\frac{\rmd x}{x^{\gamma}} = \frac{\left((n+1)^{1-\gamma} - (j+1)^{1-\gamma}\right)}{1-\gamma}\eqsp.
\]
Comparing the above bounds, to ensure that \eqref{eq:key-condition} holds, it is enough to check that 
\begin{equation}
\label{eq:intermediate_bound_step_size}
\gamma (1 + \log{(n)}) \leq \frac{c_{0} b}{2(1-\gamma)}\bigl(n^{1-\gamma} - (n/2)^{1-\gamma}\bigr)\,.
\end{equation}
Note that \eqref{eq:intermediate_bound_step_size} is guaranteed by \eqref{eq:condition-on-n}. Using that $e^{-x} \leq 1-x/2$ for $x \in [0,1]$, we obtain that 
\begin{align}
\sum_{j=1}^{k+1}\alpha_{j}^{2}\prod_{\ell=j+1}^{k+1}(1-\alpha_{\ell} b) 
&\leq \alpha_{k+1} \sum_{j=1}^{k+1} \alpha_{j} \exp\biggl\{-(b/2) \sum_{\ell=j+1}^{k+1}\alpha_{\ell}\biggr\} \\
&\leq \alpha_{k+1} \sum_{j=1}^{k+1} \alpha_{j} \prod_{\ell=j+1}^{k+1}\bigl(1 - (b/4) \alpha_{\ell}\bigr) \\
&\leq (4/b)\, \alpha_{k+1}\eqsp,
\end{align}
where the last inequality follows from \Cref{lem:summ_alpha_k}.
\end{proof}

\section{Proof of \Cref{th:shao2022_berry}}
\label{sec:proof_shao2022_berry} 
We first define explicitly the remainder term outlined in the statement of \Cref{th:shao2022_berry}:
\begin{equation}
\label{eq:remainders_theorem_shao}
\Remainder_{2}(n,a,\bConst{A},\trace{\noisecov},c_0) = \qcond\left(\sqrt{\trace{\noisecov}} \Remainder_1 + \frac{\sqrt{\trace{\noisecov}} (\bConst{A} \vee 1)^2 n^{\gamma-1/2}}{ac_0}\right)\eqsp,
\end{equation} 
and constants $\ConstPR{1}, \ConstPR{2}, \ConstPR{3}, \ConstPR{4}$ from \Cref{th:shao2022_berry}, optimized bound \eqref{eq:kolmogorov_bound_optimized}, and \Cref{rem:projected_iterates}, respectively:
\begin{equation}
\label{eq:const_def_th_2}
\begin{split}
\ConstPR{1} &= \frac{\sqrt{\qcond} \supconsteps \sqrt{\trace{\noisecov}}}{\sqrt{a c_0}} + \frac{\qcond(\trace{\noisecov} + \bConst{A} \sqrt{\trace{\noisecov}} \supconsteps)}{a c_0} \eqsp, \\
\ConstPR{2} &= \frac{\sqrt{\qcond} \supconsteps \sqrt{\trace{\noisecov}} c_0 \bConst{A}}{\sqrt{a c_0 (1-\gamma)}} + \qcond \bConst{A} \sqrt{\trace{\noisecov}} (\supconsteps + \sqrt{\trace{\noisecov}} + \bConst{A} \supconsteps) \eqsp, \\
\ConstPR{3} &= \frac{\qcond (c_0 \bConst{A} \vee 1) \sqrt{\trace{\noisecov}}\bigl(\supconsteps + \sqrt{\trace{\noisecov}} + \bConst{A}\supconsteps\bigr)}{a c_0}\eqsp, \\
\ConstPR{4} &= \frac{\qcond (c_0 \bConst{A} \vee 1) \sqrt{\trace{\noisecov^{(\Pi)}}}\bigl(\supconsteps + \sqrt{\trace{\noisecov^{(\Pi)}}} + \bConst{A}\supconsteps\bigr)}{a c_0}\eqsp.
\end{split}
\end{equation}
To complete the proof we only need to combine \eqref{eq:shao_zhang_bound} with the bounds of \Cref{th:theo_1_iid}. Note that we apply \eqref{eq:shao_zhang_bound} with 
\[
\xi_{\ell} = \frac{\eps_{\ell}}{\sqrt{n}}\eqsp.
\]
Thus, for $\Upsilon_n$ defined in \eqref{eq:shao_zhang_bound} we have 
\[
\Upsilon_n \leq \frac{\supconsteps^3}{n^{1/2}}\eqsp.
\]
Applying the Cauchy-Schwartz inequality, we get 
\begin{equation}
\label{eq:interm_term_bound}
\begin{split}
\PE[\norm{D}\norm{W}] &\leq \PE^{1/2}[\norm{D}^2] \PE^{1/2}[\norm{W}^2] \lesssim \frac{\sqrt{\qcond} \supconsteps \sqrt{\trace{\noisecov}}}{\sqrt{a c_0}}\left(\frac{1}{n^{(1-\gamma)/2}} + \frac{c_0 \bConst{A}}{\sqrt{1-\gamma} n^{\gamma/2}}\right) \\
&\qquad + \sqrt{\qcond} \sqrt{\trace{\noisecov}} \Remainder_1 \exp\biggl\{-\frac{c_0 a n^{1-\gamma}}{2(1-\gamma)}\biggr\} \norm{\theta_{0}-\thetas}\eqsp.
\end{split}
\end{equation}
Now it remains to bound the last term in \eqref{eq:shao_zhang_bound}. Using the Cauchy-Schwartz inequality and \Cref{lem:D_i_bounds}, we obtain that 
\begin{align}
&n^{-1/2} \PE[\sum_{i=n}^{2n-1}\norm{\funnoisew_{i}}\norm{D - D^{(i)}}] \leq n^{-1/2} \PE^{1/2}[\norm{\funnoisew_{1}}^2]\sum_{i=n}^{2n-1} \PE^{1/2}[\norm{D - D^{(i)}}^2] \\
& \lesssim \frac{\qcond(\trace{\noisecov} + \bConst{A} \sqrt{\trace{\noisecov}} \supconsteps)}{a c_0 n^{1-\gamma}}  + \frac{\qcond \bConst{A} \sqrt{\trace{\noisecov}} (\supconsteps + \sqrt{\trace{\noisecov}} + \bConst{A} \supconsteps)}{ n^{\gamma/2}} \\
& + \frac{\qcond \sqrt{\trace{\noisecov}} (\bConst{A} \vee 1)^2 n^{\gamma-1/2}}{ac_0} \exp\biggl\{-\frac{c_0 a n^{1-\gamma}}{2(1-\gamma)}\biggr\} \norm{\theta_0 - \thetas}\eqsp,
\end{align}
and the statement follows from \cite[Corollary~2.3]{shao2022berry}.

\subsection{Proof of auxiliary lemmas for \Cref{th:shao2022_berry}.}
\label{lem:aux-lemmas-normal}
Our proof of \Cref{th:shao2022_berry} is based on the key lemma below, which allows us to bound $\PE^{1/2}[\norm{D - D^{(i)}}^2]$ for $i \in \{n+1,\ldots,2n\}$.
\begin{lemma}
\label{lem:D_i_bounds}
Assume \Cref{assum:iid}, \Cref{assum:noise-level}, and \Cref{assum:step-size}. Then
\begin{equation}
\begin{split}
\sum_{i=n+1}^{2n}\PE^{1/2}[\norm{D - D^{(i)}}^2] &\lesssim \frac{\qcond(\sqrt{\trace{\noisecov}}+\bConst{A} \supconsteps)}{a c_0} n^{\gamma-1/2} \\
&\qquad + \qcond \bConst{A} \left(\supconsteps + \sqrt{\trace{\noisecov}} + \bConst{A} \supconsteps \right) n^{\frac{1-\gamma}{2}} \\
&\qquad +  \frac{\qcond (\bConst{A} \vee 1)^2 n^{\gamma-1/2}}{ac_0} \exp\biggl\{-\frac{c_0 a n^{1-\gamma}}{2(1-\gamma)}\biggr\} \norm{\theta_0 - \thetas}\eqsp.
\end{split}
\end{equation}
\end{lemma}
\begin{proof}
Since both terms in the right-hand side of the error bound of \Cref{lem:last_moment_bound} scales linearly with $\sqrt{\qcond}$, for simplicity we do not trace it in the subsequent bounds (i.e. assume $\qcond = 1$), and then keep the required scaling with $\qcond$ only in the final bounds. Consider the sequences of noise variables 
\[
(\State_1,\ldots,\State_{i-1},\State_{i},\State_{i+1},\ldots,\State_{2n}) \text{ and } (\State_1,\ldots,\State_{i-1},\State_{i}^{\prime},\State_{i+1},\ldots,\State_{2n})\eqsp,
\]
which differ only in position $i$, $n+1 \leq i \leq 2n$, with $\State_{i}^{\prime}$ being an independent copy of $\State_{i}$. Consider the associated SA processes 
\begin{equation}
\label{eq:coupled_processes}
\begin{split}
\theta_{k} &= \theta_{k-1} - \alpha_{k} \{ \funcA{Z_k} \theta_{k-1} - \funcb{Z_k} \}\eqsp, \quad k \geq 1, \quad \theta_{0} = \theta_{0} \in \rset^{d} \\
\theta^{(i)}_{k} &= \theta^{(i)}_{k-1} - \alpha_{k} \{ \funcA{Y_k} \theta^{(i)}_{k-1} - \funcb{Y_k} \}\eqsp, \quad k \geq 1\eqsp, \quad \theta^{(i)}_{0} = \theta_{0} \in \rset^{d}\eqsp, 
\end{split}
\end{equation}
where $Y_k = Z_k$ for $k \neq i$ and $Y_i = Z_{i}^{\prime}$. From the above representations we easily observe that $\theta_{k} = \theta^{(i)}_{k}$ for $k < i$, moreover, 
\begin{equation}
\label{eq:bound_coupled_pair}
\begin{split}
\theta_{i} - \theta^{(i)}_{i} &= \alpha_{i}\bigl\{(\funcA{\State_i^{\prime}} - \funcA{\State_i})\theta_{i-1} - \funcb{\State_i^{\prime}} + \funcb{\State_i}\bigr\} \\
&= \alpha_{i}(\funcA{\State_i^{\prime}} - \funcA{\State_i})(\theta_{i-1} - \thetas) - \alpha_{i}(\funnoisew_{i} - \funnoisew_{i}^{\prime})\eqsp,
\end{split}
\end{equation}
where $\funnoisew_i= \funcnoise{\State_i}$ and $\funnoisew_{i}^{\prime}= \funcnoise{\State_i'}$. 
Representation \eqref{eq:bound_coupled_pair} implies, together with \Cref{lem:last_moment_bound} and $c_0 \leq a$, that
\begin{equation}
\label{eq:theta_i_diff_bound}
\begin{split}
\PE^{1/2}[\norm{\theta_{i} - \theta^{(i)}_{i}}^2] 
& \lesssim \alpha_{i}\sqrt{\trace{\noisecov}} + \frac{\bConst{A} \supconsteps \alpha_{i}^{3/2}}{\sqrt{a}} + \alpha_{i} \bConst{A} \exp\biggl\{-\frac{a}{2}\sum_{j=1}^{i-1}\alpha_{j}\biggr\} \norm{\theta_{0}-\thetas} \\
&\lesssim \alpha_i \bigl(\sqrt{\trace{\noisecov}} + \bConst{A} \supconsteps\bigr) + \alpha_{i} \bConst{A} \exp\biggl\{-\frac{a}{2}\sum_{j=1}^{i-1}\alpha_{j}\biggr\} \norm{\theta_{0}-\thetas}\eqsp.
\end{split}
\end{equation}
Moreover, for any $j > i$ one observes, expanding \eqref{eq:coupled_processes}, that
\begin{equation}
\label{eq:sync_coupled_chains}
\theta_{j} - \theta^{(i)}_{j} = \biggl\{\prod_{k=i+1}^{j}(\Id - \alpha_{k}\funcA{\State_k})\biggr\}(\theta_{i} - \theta^{(i)}_{i})\eqsp.
\end{equation}
We use the above representations to estimate $\PE^{1/2}[\norm{D - D^{(i)}}^2]$. Using Minkowski's inequality, 
\begin{equation}
\label{eq:D_D_i_bound}
\PE^{1/2}[\norm{D - D^{(i)}}^2] \leq \sum_{j=1}^{4}\PE^{1/2}[\norm{D_{j} - D_{j}^{(i)}}^2]\eqsp,
\end{equation}
and bound the respective differences separately. Recall that here $D_1-D_4$ are defined in \eqref{eq:error_decomposition}, and $D_1^{(i)}-D_4^{(i)}$ are their respective counterparts with $Z_i$ substituted with $Z_i^{\prime}$. First we note that the term $D_1 = D_{1}^{(i)}$ for any $n+1 \leq i \leq 2n$. Next, using \eqref{eq:sync_coupled_chains} and \eqref{eq:theta_i_diff_bound}, we get
\begin{align}
\PE^{1/2}[\norm{D_2 - D_{2}^{(i)}}^2] 
&= 
\frac{1}{\sqrt{n}\alpha_{2n}} \PE^{1/2}[\norm{\theta_{2n}-\theta_{2n}^{(i)}}^2] \\
&\leq \frac{1}{\sqrt{n}\alpha_{2n}} \PE^{1/2}\bigl[\norm{\prod_{k=i+1}^{2n}(\Id - \alpha_{k} \funcAw_k)}^2\bigr] \PE^{1/2}[\norm{\theta_{i}-\theta_{i}^{(i)}}^2] \\
&\overset{(a)}{\lesssim} \frac{\alpha_{i} (\sqrt{\trace{\noisecov}} + \bConst{A} \supconsteps)}{\sqrt{n}\alpha_{2n}}\exp\bigl\{-\frac{a}{2} \sum_{k=i+1}^{2n}\alpha_{k}\bigr\} \\
&\qquad \qquad + \frac{\alpha_{i} \bConst{A}}{\sqrt{n}\alpha_{2n}} \exp\bigl\{-\frac{a}{2}\sum_{k=1}^{2n}\alpha_{k}\bigr\} \norm{\theta_0 - \thetas}\eqsp. 
\end{align}
In the inequality (a) above we additionally used the stability of matrix product introduced from \Cref{cor:exp_bound_decay}. Summing the above inequality for $i = n+1$ to $2n$ and applying \Cref{lem:summ_alpha_k}, we get 
\begin{align}
&\sum_{i=n+1}^{2n}\PE^{1/2}[\norm{D_2 - D_{2}^{(i)}}^2] 
\lesssim \frac{\sqrt{\trace{\noisecov}} + \bConst{A} \supconsteps}{a\sqrt{n}\alpha_{2n}} + \frac{\bConst{A}}{a \sqrt{n}\alpha_{2n}} \exp\bigl\{-\frac{a}{2}\sum_{k=1}^{n}\alpha_{k}\bigr\} \norm{\theta_0 - \thetas} \nonumber \\
& \label{eq:D_2_diff_bound}
\qquad \qquad \lesssim \frac{\sqrt{\trace{\noisecov}} + \bConst{A} \supconsteps}{a c_0} n^{\gamma-1/2} + \frac{\bConst{A} n^{\gamma-1/2}}{a c_0}  \exp\biggl\{-\frac{c_0 a n^{1-\gamma}}{2(1-\gamma)}\biggr\} \norm{\theta_0 - \thetas} \eqsp. 
\end{align}
Now we proceed with the difference $D_{3} - D_{3}^{(i)}$. Using \eqref{eq:error_decomposition}, we get 
\begin{align}
D_3 - D_{3}^{(i)} = \frac{1}{\sqrt{n}}(\funcAw_i - \funcAw_i^{\prime})(\theta_{i-1} - \thetas) + \frac{1}{\sqrt{n}}\sum_{k=i+1}^{2n}(\funcAw_k - \bA)(\theta_{k-1} - \theta_{k-1}^{(i)})\eqsp. 
\end{align}
The expression above is a sum of martingale-difference terms w.r.t. filtration $\mcf_{k}^{\prime} = \sigma(\State_{i}^{\prime},\State_{\ell},\ell \leq k)$. Hence, we get, using \eqref{eq:sync_coupled_chains} and \Cref{lem:last_moment_bound}, that 
\begin{align}
\label{eq:D_3_diff_bound}
\PE[\norm{D_3 - D_{3}^{(i)}}^2] 
&\lesssim \frac{\bConst{A}^2}{n}\PE[\norm{\theta_{i-1} - \thetas}^2] + \frac{\bConst{A}^2}{n}\sum_{k=i+1}^{2n}\PE[\norm{\theta_{k-1} - \theta_{k-1}^{(i)}}^2] \\
&\lesssim \frac{\bConst{A}^2 \supconsteps^{2} \alpha_{i}}{n a} + \frac{\bConst{A}^2\norm{\theta_0 - \thetas}^2}{n}\exp\bigl\{-a\sum_{j=1}^{i-1}\alpha_{j}\bigr\} \\
&\qquad\qquad + \frac{\bConst{A}^2}{n}\PE[\norm{\theta_{i} - \theta^{(i)}_{i}}^2]\sum_{k=i+1}^{2n} \exp\bigl\{-a\sum_{j=i+1}^{k-1}\alpha_{j}\bigr\}\eqsp.
\end{align}
Using now the bound \eqref{eq:theta_i_diff_bound}, we obtain that 
\begin{align}
&\PE[\norm{\theta_{i} - \theta^{(i)}_{i}}^2]\sum_{k=i+1}^{2n} \exp\bigl\{-a\sum_{j=i+1}^{k-1}\alpha_{j}\bigr\} \\ 
& \quad \lesssim \alpha_{i}^2 \left(\trace{\noisecov} + \bConst{A}^2 \supconsteps^2\right) \sum_{k=i+1}^{2n} \exp\bigl\{-a\sum_{j=i+1}^{k-1}\alpha_{j}\bigr\} +  \alpha_{i}^{2} \bConst{A}^2 \norm{\theta_0 - \thetas}^2 \sum_{k=i+1}^{2n} \exp\bigl\{-a\sum_{j=1}^{k-1}\alpha_{i}\bigr\} \\
& \quad \lesssim \frac{\alpha_{i}^2}{\alpha_{2n}} \left(\trace{\noisecov} + \bConst{A}^2 \supconsteps^2\right) \sum_{k=i+1}^{2n} \alpha_{k} \exp\bigl\{-a\sum_{j=i+1}^{k-1}\alpha_{j}\bigr\} + \frac{\alpha_{i}^2 \bConst{A}^2}{\alpha_{2n}} \norm{\theta_0 - \thetas}^2 \sum_{k=i+1}^{2n} \alpha_{k} \exp\bigl\{-a\sum_{j=1}^{k-1}\alpha_{i}\bigr\} \\
& \quad \overset{(a)}{\lesssim} \frac{\alpha_{i}^2(\trace{\noisecov} + \bConst{A}^2\supconsteps^2)}{a \alpha_{2n}} + \frac{\alpha_{i}^2 \bConst{A}^2}{a \alpha_{2n}} \norm{\theta_0 - \thetas}^2 \exp\bigl\{-a \sum_{j=1}^{i-1}\alpha_{j}\bigr\}\eqsp. 
\end{align}
In the above formula in (a) we additionally used that, since $\alpha_i a \leq 1/2$,  
\begin{equation}
\label{eq:integral_bound}
\sum_{k=i+1}^{2n} \alpha_{k} \exp\bigl\{-a\sum_{j=i+1}^{k-1}\alpha_{j}\bigr\} \lesssim \int_{0}^{+\infty}\exp\{-ax\}\,dx = \frac{1}{a}\eqsp.
\end{equation}
Hence, combining everything in \eqref{eq:D_3_diff_bound}, and using additionally that $\alpha_i \leq a$,  we get
\begin{multline}
\PE^{1/2}[\norm{D_3 - D_{3}^{(i)}}^2] \lesssim \frac{ \bConst{A}}{\sqrt{na}}\left( \supconsteps\sqrt{\alpha_{i}} + \frac{\alpha_{i}(\sqrt{\trace{\noisecov}} + \bConst{A} \supconsteps)}{\sqrt{\alpha_{2n}}}\right) + \\
\frac{\bConst{A}}{\sqrt{n}} \left(1 + \frac{ \alpha_{i}\bConst{A}}{\sqrt{a \alpha_{2n}}} \right) \exp\bigl\{-\frac{a}{2}\sum_{j=1}^{i-1}\alpha_{i}\bigr\} \norm{\theta_0 - \thetas}\eqsp.
\end{multline}
Summing the above inequality for $i = n+1$ to $2n$, and using that $\alpha_{k} = c_{0}/k^{\gamma}$, we get
\begin{align}
\sum_{i=n+1}^{2n}\sqrt{\alpha_{i}} \lesssim \sqrt{c_0} n^{1-\gamma/2}\eqsp, \quad \sum_{i=n+1}^{2n}\frac{\alpha_{i}}{\sqrt{\alpha_{2n}}} \lesssim \sqrt{c_0} n^{1-\gamma/2}\eqsp,
\end{align}
and, hence, using again $\alpha_i \leq a$, we get
\begin{align}
&\sum_{i=n+1}^{2n}\PE^{1/2}[\norm{D_3 - D_{3}^{(i)}}^2] \\
& \lesssim \frac{\bConst{A} \sqrt{c_0}}{\sqrt{a}}\left(\supconsteps + \sqrt{\trace{\noisecov}} + \bConst{A} \supconsteps \right) n^{\frac{1-\gamma}{2}} + \frac{(\bConst{A} \vee 1)^2 \norm{\theta_{0}-\thetas}}{\sqrt{n} \alpha_{2n}} \sum_{i=n+1}^{2n} \alpha_{i} \exp\bigl\{-\frac{a}{2}\sum_{j=1}^{i-1}\alpha_{j}\bigr\} \\
& \lesssim \frac{\bConst{A} \sqrt{c_0}}{\sqrt{a}}\left(\supconsteps + \sqrt{\trace{\noisecov}} + \bConst{A} \supconsteps \right) n^{\frac{1-\gamma}{2}} + \frac{n^{\gamma-1/2} (\bConst{A} \vee 1)^2 \norm{\theta_{0}-\thetas}}{a c_0} \exp\biggl\{-\frac{c_0 a n^{1-\gamma}}{2(1-\gamma)}\biggr\}\eqsp,
\end{align}
where for the last identity we used the fact that $\alpha_{k} = c_{0}/k^{\gamma}$, and \eqref{eq:integral_bound}. It remains to upper bound the difference $D_{4} - D_{4}^{(i)}$. Note first that, proceeding as in \eqref{eq:step_size_diff_bound}, we get 
\begin{equation}
\label{eq:step_size_neighbors_bound}
\alpha_{k-1} - \alpha_{k} \leq \frac{\gamma}{k-1} \frac{1}{\alpha_{k-1}} \lesssim \frac{1}{(k-1)^{1-\gamma}}\eqsp.
\end{equation}
Using now the definition of $D_{4}$ in \eqref{eq:error_decomposition}, we have that
\begin{align}
\PE^{1/2}[\norm{D_4 - D_{4}^{(i)}}^2] 
&= \frac{1}{\sqrt{n}}
\PE^{1/2}[\norm{
\sum_{k=i+1}^{2n}\bigl(\theta_{k-1} - \theta^{(i)}_{k-1}\bigr)\left(\frac{1}{\alpha_k} - \frac{1}{\alpha_{k-1}}\right)}^2] \\
&\leq \frac{1}{\sqrt{n}} \PE^{1/2}[\norm{\theta_{i} - \theta^{(i)}_{i}}^2] \sum_{k=i+1}^{2n}\left(\frac{1}{\alpha_k} - \frac{1}{\alpha_{k-1}}\right)\exp\bigl\{-\frac{a}{2}\sum_{j=i+1}^{k-1}\alpha_{j}\bigr\}\eqsp.
\end{align}
Hence, using the bound \eqref{eq:theta_i_diff_bound} and taking sum for $i = n+1$ to $2n$, 
\begin{align}
\sum_{i=n+1}^{2n}\PE^{1/2}[\norm{D_4 - D_{4}^{(i)}}^2] 
&\lesssim \left(\frac{\sqrt{\trace{\noisecov}}+\bConst{A} \supconsteps}{\sqrt{n}}\right) \sum_{i=n+1}^{2n} \alpha_{i} \sum_{k=i+1}^{2n} \left(\frac{1}{\alpha_k} - \frac{1}{\alpha_{k-1}}\right) \exp\bigl\{-a\sum_{j=i+1}^{k-1}\alpha_{j}\bigr\} \\
& \quad + \frac{\bConst{A} \norm{\theta_{0}-\thetas}}{\sqrt{n}} \sum_{i=n+1}^{2n} \alpha_{i} \sum_{k=i+1}^{2n} \left(\frac{1}{\alpha_k} - \frac{1}{\alpha_{k-1}}\right) \exp\bigl\{-\frac{a}{2}\sum_{j=1}^{k-1}\alpha_{j}\bigr\}\eqsp.
\end{align}
Changing now the summation order, we obtain that 
\begin{align}
\sum_{i=n+1}^{2n} \alpha_{i} \sum_{k=i+1}^{2n} \left(\frac{1}{\alpha_k} - \frac{1}{\alpha_{k-1}}\right) \exp\bigl\{-a\sum_{j=i+1}^{k-1}\alpha_{j}\bigr\} 
&\lesssim \frac{1}{a}\sum_{k=n+2}^{2n} \left(\frac{1}{\alpha_k} - \frac{1}{\alpha_{k-1}}\right) \\
&= \frac{1}{a}\left(\frac{1}{\alpha_{n+1}} - \frac{1}{\alpha_{2n}}\right) \lesssim \frac{1}{a \alpha_{2n}}\eqsp.
\end{align}
Hence, combining the above bounds, we get
\begin{align}
\label{eq:D_4_diff_bound}
&\sum_{i=n+1}^{2n}\PE^{1/2}[\norm{D_4 - D_{4}^{(i)}}^2]
\lesssim \frac{\sqrt{\trace{\noisecov}}+\bConst{A}\supconsteps}{\sqrt{n} a \alpha_{2n}} +  
\sum_{i=n+1}^{2n}\PE^{1/2}[\norm{D_4 - D_{4}^{(i)}}^2] \\
&\qquad \lesssim \frac{\sqrt{\trace{\noisecov}}+\bConst{A}\supconsteps}{a c_0} n^{\gamma-1/2} + \frac{\bConst{A} n^{\gamma - 1/2}}{a c_0} \exp\biggl\{-\frac{c_0 a n^{1-\gamma}}{2(1-\gamma)}\biggr\} \norm{\theta_{0}-\thetas}\eqsp.
\end{align}
It remains now to combine \eqref{eq:D_2_diff_bound}, \eqref{eq:D_3_diff_bound}, and \eqref{eq:D_4_diff_bound} in \eqref{eq:D_D_i_bound} and use that $c_0 \leq a$. 
\end{proof}

\subsection{Relations between $\kolmogorov$ and integral probability metrics}
\label{sec:smooth_wasserstein_kolmogorov}
In this section we closely follow the exposition outlined in \cite{gaunt2023bounding}. Consider two $\rset^{d}$-valued random variables $X$ and $Y$. Then the integral probability metric \cite{zolotarev1984probability}, associated with the class of functions $\H = \{h: \rset^{d} \to \rset, \PE[|h(X)|] < \infty, \PE[|h(Y)|] < \infty\}$, is defined as 
\begin{equation}
\label{eq:integral_prob_metrics_def}
\metricd[\H](X,Y) = \sup_{h \in \H}\bigl| \PE[h(X)] - \PE[h(Y)] \bigr|\eqsp.
\end{equation}
Different choices of $\H$ induce different metrics, in particular, we mention the following:
\begin{align} 
\H_{K} &= \{\indi{x \leq u}, \quad u = (u_1,\ldots,u_d) \in \rset^{d}\} \\
\H_{Conv} &= \{\indi{x \in B}, \quad B \in \Conv(\rset^{d})\} \\
\H_{W} &= \{h: \rset^{d} \to \rset, \quad \norm{h}[\operatorname{Lip}] \leq 1\} \\
\H_{[m]} &= \{h: \rset^{d} \to \rset, \quad h^{m-1} \text{ is Lipschitz with } |h|_{j} \leq 1\eqsp, \quad 1 \leq j \leq m\}\eqsp,
\end{align}
where $\Conv(\rset^{d})$ refers to the set of convex sets in $\rset^{d}$, $\norm{h}[\operatorname{Lip}] = \sup_{x \neq y}\frac{\norm{h(x)-h(y)}}{\norm{x-y}}$, and the quantity $|h|_{j}$ is defined as
\[
|h|_{j} = \max_{i_1,\ldots,i_j \in \{1,\ldots,d\}} 
 \norm{\frac{\partial^{j} h(u)}{\partial u_{i_1} \ldots \partial u_{i_j}}}[\infty]\eqsp.
\]
In other words, for $m \in \nset$, the class $\H_{[m]}$ corresponds to the functions with bounded derivatives up to the $(m-1)$-th order. The class $\H_{K}$ induces the Kolmogorov distance between distributions \cite{zolotarev1984probability}, class $\H_{Conv}$ induces the metric $\kolmogorov$ defined in \eqref{eq:berry-esseen}, which is the main object of studies in the current paper. Class $\H_{W}$ induces the celebrated Wasserstein distance, and classes $\H_{[m]}$ induce smoothed Wasserstein distances. We will denote the respective metrics by $\metricd[K],\kolmogorov,\metricd[W]$, and $\metricd[{[m]}]$, respectively. Then, obviously, 
\[
\metricd[K](X,Y) \leq \kolmogorov(X,Y)
\]
for any random vectors $X$ and $Y$. Other relations are more involved. When $Y$ is a multivariate normal vector, it is known (see e.g. \cite{nourdin2022multivariate}) that
\begin{align}
\kolmogorov(X,Y) \leq C \sqrt{\metricd[W](X,Y)}\eqsp,
\end{align}
where the constant $C$ in the above inequality depends on the covariance matrix of vector $Y$. This inequality justifies comparison of our bounds of \Cref{th:shao2022_berry} with the result of \cite{srikant2024rates}. The authors in \cite{pmlr-v99-anastasiou19a} considered integral probability metric $\metricd[{[2]}]$ and obtained rate of convergence 
\[
\metricd[{[2]}](\sqrt{n}(\prtheta_{n} - \thetas),Y) \leq \frac{C_{1}}{\sqrt{n}}\eqsp,
\]
where $Y \sim \mathcal{N}(0,\Sigma_{\infty})$, and $C_{1}$ in the above inequality stands for a constant depending upon problem dimension $d$ and other instance-dependent parameters from \Cref{assum:noise-level}. Applying the result of \cite[Proposition~2.6]{gaunt2023bounding} yields 
\[
\metricd[K](\sqrt{n}(\prtheta_{n} - \thetas),Y) \lesssim \left(\metricd[{[2]}](\sqrt{n}(\prtheta_{n} - \thetas),Y)\right)^{1/3} \lesssim \frac{1}{n^{1/6}}\eqsp.
\]
Thus, the result of \cite{pmlr-v99-anastasiou19a} implies rate of convergence of $\sqrt{n}(\prtheta_{n}-\thetas)$ to normal law $\mathcal{N}(0,\Sigma_{\infty})$ of order $n^{-1/6}$ in a sense of Kolmogorov distance $\metricd[K]$. Our result of \Cref{th:shao2022_berry} implies the respective rate of order $n^{-1/4}$. At the same time, it is not clear if $\kolmogorov$ can be directly related to $\metricd[{[2]}]$.

\section{Bootstrap validity proof}
\label{appendix:bootstrap}
\subsection{Proof of \Cref{th:bootstrap_validity}}
\label{sec:bootstrap_validity}
We first define explicitly the remainder term $\Remainder_{3}$ outlined in the statement of \Cref{th:bootstrap_validity}, that is, 
\begin{equation}
\label{eq:remainders_theorem_bootstrap}
\Remainder_{3}(n,a,\bConst{A},\supconsteps) = 
\frac{\qcond^{3/2} (\bConst{A}^{3} \vee 1) \supconsteps n^{1/4} \sqrt{\log{n}}}{a^{3/2}}\eqsp.
\end{equation}
In the above bounds we do not trace the precise dependence on the constant $c_0$ from the definition of the step size. We now define the following sets, with the convention $\alpha_{\ell} = c_{0}/\sqrt{\ell}$:
\begin{align}
\label{eq:omegas_definition}
\Omega_{1} &= \left\{\forall k \in [n, 2n-1]: \norm{\theta_k - \thetas} \geq \sqrt{\qcond} \rme^{2} \exp\bigl\{- \frac{a}{2} \sum_{\ell=1}^{k}\alpha_\ell\bigr\}\norm{\theta_0 - \thetas} \right. \\
&\qquad \qquad \qquad \qquad \qquad \qquad \qquad \qquad \qquad \qquad \qquad \qquad \qquad + \left.\frac{8 \rme^{2}\sqrt{\qcond} \supconsteps \log n}{\sqrt{a}} \sqrt{\alpha_{k}}\right\}\eqsp, \\
\Omega_{2} &= \left\{ n+1 \le m \le k \le 2n: ~~ \|\ProdB_{m:k}\| \le  \sqrt{\qcond} \rme^2 \prod_{j = m}^{k} \bigl(1 - \frac{a \alpha_{j}}{4}\bigr) \right\}\eqsp,\\
\Omega_{3} &= \left\{ \|\noisecov^{-1/2} \noisecov^\boot  \noisecov ^{-1/2} - \Id \| \le 4  \|\funnoisew \|_{\infty} \sqrt{\frac{\log (n)}{ \sigma n}} + \frac{4 (1 + \|\funnoisew \|_{\infty}^2/\sigma^2 )\log (n)}{n} \right\}\eqsp,\\
\Omega_{4} &= \left\{\forall \ell \in [n, 2n-1]: ~~  \bigg \| \sum_{k = \ell+1}^{2n} (\funcAw_k - \bA) \ProdB_{\ell+1:k-1} \bigg \| \le  \frac{8 \bConst{A} \sqrt{\qcond}\rme^2 \sqrt{\log n}}{\sqrt{a \alpha_\ell}}   + 6\bConst{A} \sqrt{\qcond} \rme \log n  \right\}\eqsp, \\
\Omega_{5} &= \left\{\forall h \in [1;n]\eqsp, \, \forall m \in [n, 2n-h]: ~~ \norm{\sum_{\ell = m+1}^{m+h} \alpha_\ell (\Am_\ell - \bA)}[Q] \leq  2 \bConst{A} \sqrt{\qcond} \sqrt{\sum_{\ell=m+1}^{m+h}\alpha_\ell^{2}} \log(2n^4)\right\}\eqsp, \\
\Omega_{6} &= \left\{ \| \noisecov^\boot   - \noisecov \| \le 4  \supconsteps \sqrt{\frac{\|\noisecov\| \log (n)}{n}} + \frac{4 (\|\noisecov\| + \supconsteps^2 )\log (n)}{n} \right\}\eqsp,
\end{align}

Then, due to \Cref{high prob last iterate}, we have that $\P(\Omega_{1}) \geq 1 - \frac{1}{n}$. Similarly, due to \Cref{cor:matr_product_as_bound}, $\P(\Omega_{2}) \geq 1 - \frac{1}{n}$. The bounds on $\PP(\Omega_3)$ and $\PP(\Omega_4)$ follows from Lemma \ref{matrix bernstein} and Lemma \ref{matrix freedman}, respectively. Similarly, \Cref{prop:product_random_matrix_bootstrap} implies that $\P(\Omega_{5}) \geq 1 - \frac{1}{n}$. Hence, based on the sets above, we can construct 
\[
\Omega_0 = \Omega_{1} \cap \Omega_{2} \cap \Omega_{3} \cap \Omega_{4} \cap \Omega_{5} \cap \Omega_{6}\eqsp,
\]
such that $\PP(\Omega_0) \geq 1 - \frac{6}{n}$. All further on, we restrict ourselves to the event $\Omega_0$. Restricting to this event, we obtain that, with Minkowski's inequality, 
\begin{align}
& \sup_{B \in \Conv(\rset^{d})} |\PPb(\sqrt n (\bar{\theta}_{n}^\boot - \bar{\theta}_n) \in B ) - \PP(\sqrt n (\bar{\theta}_n - \thetas) \in B)| \\
&\qquad \qquad \leq \sup_{B \in \Conv(\rset^{d})} \big| \PPb(\sqrt n (\bar{\theta}_{n}^\boot - \bar{\theta}_n) \in B) - \PPb(\xi^\boot \in B)\big| \\
&\qquad \qquad + \sup_{B \in \Conv(\rset^{d})} \big| \PP(\xi \in B) - \PPb(\xi^\boot \in B) \big| \\
&\qquad\qquad + \sup_{B \in \Conv(\rset^{d})} \big| \P\bigl(\sqrt{n}(\bar{\theta}_{n} - \thetas) \in B\bigr) - \P(\xi \in B) \big|\eqsp,
\end{align}
where we set $\xi^\boot \sim \mathcal N(0, \bA^{-1} \noisecov^\boot\bA^{-\top})$, $ \noisecov^\boot = n^{-1} \sum_{\ell = 1}^n \funnoisew_\ell \funnoisew_\ell^{\top}$, and $\xi \sim \mathcal N(0,\Sigma_{\infty})$, where 
$\Sigma_{\infty} = \bA^{-1} \noisecov \bA^{-\top}$. Now we control the first supremum using \Cref{CLT in the bootstrap world}, second one using \Cref{prop: gaussian comparison}, and third with \Cref{th:shao2022_berry}.

\begin{lemma}
\label{matrix bernstein}
Assume \Cref{assum:iid}, \Cref{assum:noise-level}, \Cref{assum:step-size} with $\gamma = 1/2$, and \Cref{assum:step-size-bootstrap}. Then
$$
  \PP(\Omega_3) \geq 1 - 1/n. 
$$ 
\end{lemma}
\begin{proof}
    The proof follows directly from the matrix Bernstein inequality, e.g. \cite{tropp2015introduction}. We note that 
    $$
    \|\noisecov ^{-1/2} \funnoisew_{\ell} \funnoisew_{\ell}^\top \noisecov ^{-1/2}  - \Id \| \le 1 + \|\noisecov ^{-1/2} \funnoisew \|_{\infty}^2 \eqsp.   
    $$
    and 
    $$
\| \sum_{k = n+1}^{2n} \PE[(\noisecov ^{-1/2} \funnoisew_{\ell} \funnoisew_{\ell}^\top \noisecov ^{-1/2}  - \Id)^2] \| \le n \PE \|\noisecov ^{-1/2} \funnoisew \|_{\infty}^2 \eqsp. 
    $$
\end{proof}

\begin{lemma} 
\label{matrix freedman}
Assume \Cref{assum:iid}, \Cref{assum:noise-level}, \Cref{assum:step-size} with $\gamma = 1/2$, and \Cref{assum:step-size-bootstrap}. Then
$$
  \PP \left( \Omega_4 \cap \Omega_2 \right) \geq 1-  \frac{1}{n}. 
$$ 
\end{lemma}
\begin{proof}
   Denote
   $$
   X_k =  (\funcAw_k - \bA) \ProdB_{\ell+1:k-1}. 
   $$
   and let $\mathcal F_{k, l+1} = \sigma\{\State_j, \ell+1  \le j \le k \}$, $\ell+1 \le k \le 2n$. Then $\PE[X_k | \mathcal F_{k-1, \ell+1}] = 0$. Let $S_\ell = \sum_{k = \ell+1}^n X_k$. Note that on $\Omega_2$, quadratic variation of $S_\ell$ can be controlled as
   \begin{align}
   \operatorname{Var}^2 &:= \max ( \|\sum_{k=\ell+1}^{2n} \PE [X_k X_k^\top | \mathcal F_{k-1, l+1} ] \|, \|\sum_{k=\ell+1}^{2n} \PE[X_k^\top X_k | \mathcal F_{k-1, l+1} ]\|) \\ 
& \le \qcond \rme^4 \bConst{A}^2 \sum_{k = \ell+1}^{2n} \prod_{j = \ell+1}^{k-1} (1 - a \alpha_j/4)^2 \le \frac{4 \qcond \rme^4 \bConst{A}^2}{a \alpha_\ell}
   \end{align}
Furthermore, on $\Omega_2$
$$
\|X_k\| \le \sqrt{\qcond} \rme^2 \bConst{A} \prod_{j = \ell+1}^{k-1} (1 - a \alpha_j/4) \le \sqrt{\qcond} \rme^2 \bConst{A}.    
$$
It remains to apply the Freedman inequality for matrix-values martingales  \cite{MR2802042} and use the union bound over $\ell \in [n, 2n - 1]$. 
\end{proof}

\begin{lemma}
\label{matrix hoeffding}
Assume \Cref{assum:iid}, \Cref{assum:noise-level}, \Cref{assum:step-size} with $\gamma = 1/2$, and \Cref{assum:step-size-bootstrap}. Then
$$
\PP(\Omega_5) \geq 1- 1/n.
$$
\end{lemma}
\begin{proof}
We first fix $h \in [1;n]$, $m \in [n, 2n-h]$, and consider the random variable
\[
T_n = \norm{\sum_{\ell = m+1}^{m+h} \alpha_\ell (\Am_\ell - \bA)}\eqsp.
\]
Then we control its variance as  
\begin{align}
\max ( \norm{\sum_{\ell = m+1}^{m+h} \alpha_{\ell}^2\PE [(\Am_\ell - \bA) (\Am_\ell - \bA)^\top]}, \norm{\sum_{\ell = m+1}^{m+h} \alpha_{\ell}^2\PE [(\Am_\ell - \bA)^{\top} (\Am_\ell - \bA)]}) \leq \bConst{A}^2 \sum_{\ell = m+1}^{m+h} \alpha_{\ell}^2\eqsp,
\end{align}
moreover, $\norm{(\Am_\ell - \bA) (\Am_\ell - \bA)^\top} \leq \bConst{A}^2$. Applying now the matrix Bernstein inequality \cite{tropp2015introduction}, we obtain that with probability at least $1 - 1/n^3$, we have
\begin{align}
T_n \leq \bConst{A} \sqrt{2\sum_{\ell = m+1}^{m+h}\alpha_{\ell}^2} \sqrt{\log{(2n^3 d)}} + \frac{\alpha_{m+1}\bConst{A}}{3} \log{(2n^3 d)} \leq 2\bConst{A} \sqrt{\sum_{\ell = m+1}^{m+h}\alpha_{\ell}^2} \log{(2n^4)}\eqsp.
\end{align}
In the last line here we used that $d \leq n$. Rest of the proof follows by taking union bound over $h$ and $m$ together with $\norm{B}[Q]^2 \leq \qcond \norm{B}^2$ valid for any matrix $B \in \rset^{d \times d}$. 
\end{proof}

\begin{lemma}
    \label{bernstein for covariance}
    Assume \Cref{assum:iid}, \Cref{assum:noise-level}, \Cref{assum:step-size} with $\gamma = 1/2$, and \Cref{assum:step-size-bootstrap}. Then
    $$
\PP(\Omega_6) \geq 1 - 1/n.
    $$
\end{lemma}
\begin{proof}
    It is easy to check that
    $\|\funnoisew_\ell \funnoisew_{\ell}^\top - \noisecov\| \le \|\noisecov\| + \supconsteps^2$. Moreover,
    $$
    \operatorname{Var}^2: = \|\sum_{\ell=1}^n (\funnoisew_\ell \funnoisew_\ell^\top - \noisecov)^2 \| \le n \supconsteps^2 \|\noisecov\|.  
    $$
    It remains to apply the matrix Bernstein inequality together with the bound $n \geq d$ from \Cref{assum:step-size}.
\end{proof}

\subsection{Rate of Gaussian approximation in the bootstrap world}

The main result of this section is the following theorem.
\begin{theorem}
\label{CLT in the bootstrap world}
Assume \Cref{assum:iid}, \Cref{assum:noise-level}, \Cref{assum:step-size} with $\gamma = 1/2$, and \Cref{assum:step-size-bootstrap}. Then, conditionally on the event $\Omega_0$, the following error bound holds:
\begin{equation}
\begin{split}
&\sup_{B \in \Conv(\rset^{d})} \bigg| \PPb(\sqrt n (\bar{\theta}_{n}^\boot - \bar{\theta}_n) \in B) - \PPb(\xi^\boot \in B)\bigg| \lesssim \frac{d^{1/2}  \supconsteps^3}{\lambda_{\min}^{3/2}\sqrt{n}} + \frac{\qcond^{3/2}(\bConst{A}^3 \vee 1)\supconsteps^2 \log{n}}{a^{3/2} \lambda_{\min} n^{1/4}} \\
&\qquad\qquad\qquad + \frac{\qcond^{3/2}(\bConst{A}^{3} \vee 1) \supconsteps n^{1/4} \sqrt{\log{n}}}{a^{3/2} \lambda_{\min}}\exp\bigl\{-\frac{a}{4}\sum_{j=1}^{n}\alpha_{j}\bigr\} \norm{\theta_0-\thetas}\eqsp,
\end{split}
\end{equation}
where $\xi^\boot \sim \mathcal N(0, \bA^{-1} \noisecov^\boot\bA^{-\top})$ and $ \noisecov^\boot = n^{-1} \sum_{\ell = 1}^n \funnoisew_\ell \funnoisew_\ell^{\top}$.
\end{theorem}
\begin{proof}
Since both terms in the right-hand side of the error bound of \Cref{lem:last_moment_bound} scales linearly with $\sqrt{\qcond}$, for simplicity we do not trace it in the subsequent bounds (i.e. assume $\qcond = 1$), and then keep the required scaling with $\qcond$ only in the final bounds. Recall first that the quantities $\theta_{k}^\boot$ and $\theta_k$ are defined in \eqref{eq:lsa_bootstrap}. We start from the following decomposition:
\begin{equation}
\theta_{k}^\boot - \theta_k = (\Id - \alpha_k \bA) (\theta_{k-1}^\boot - \theta_{k-1}) - \alpha_k (w_k - 1)  \funnoisew_{k} - \alpha_k (\funcAw_k - \bA) (\theta_{k-1}^\boot - \theta_{k-1}) - \alpha_k (w_k-1) \funcAw_k(\theta_{k-1}^\boot - \thetas)\eqsp.
\end{equation}
Taking average for $k$ from $n+1$ to $2n$, we get after multiplying by $\sqrt{n}$ that 
\begin{multline}
\label{eq:error_decomposition_boot2}
\sqrt{n}\bA(\bar{\theta}_{n}^\boot - \bar{\theta}_n) = -\underbrace{\frac{1}{\sqrt{n}}\sum_{k=n+1}^{2n} (w_k-1)\funnoisew_{k}}_{W^\boot} +  \underbrace{\frac{1}{\sqrt{n}}\frac{\theta_{n}^\boot-\theta_n}{\alpha_{n}}}_{D_1^\boot} - \underbrace{\frac{1}{\sqrt{n}}\frac{\theta_{2n}^\boot-\theta_{2n}}{\alpha_{2n}}}_{D_2^\boot}   \\
-\underbrace{\frac{1}{\sqrt{n}}\sum_{k=n+1}^{2n}(w_k-1) \funcAw_k (\theta_{k-1}^\boot - \thetas)}_{D_3^\boot}
+\underbrace{\frac{1}{\sqrt{n}}\sum_{k=n+1}^{2n}\bigl(\theta_{k-1}^\boot - \theta_{k-1}\bigr)\left(\frac{1}{\alpha_k} - \frac{1}{\alpha_{k-1}}\right)}_{D_4^\boot} \\
 - \underbrace{\frac{1}{\sqrt{n}}\sum_{k=n+1}^{2n} (\funcAw_k - \bA) (\theta_{k-1}^\boot - \theta_{k-1})}_{D_5^\boot}\eqsp.
\end{multline}
The formula \eqref{eq:error_decomposition_boot2} resembles the key representation $T^\boot := \sqrt{n}\bA(\bar{\theta}_{n}^\boot - \bar{\theta}_n)  = W^\boot + D^\boot$, where
\begin{equation}
\label{eq:D_boot_def}
D^\boot = D_1^\boot + \ldots + D_5^\boot\eqsp,
\end{equation}
and $D_1^\boot - D_5^\boot$ are defined in \eqref{eq:error_decomposition_boot2}. Now we aim to apply the result of \cite{shao2022berry}:
\begin{multline}
\label{eq:shao_zhang_bound_bootstrap}
\sup_{B \in \Conv(\rset^d)} | \PP^{\boot}(T^\boot \in B) - \PPb(\xi^\boot \in B)| 
\le 259 d^{1/2} \Upsilon + 2 \PEb[\|W^\boot\| \|D^\boot\|] \\
+ 2 \sum_{\ell=n+1}^{2n} \PEb[\|\xi_\ell\| \|D^\boot - D^{(\boot,\ell)}\|]\eqsp,
\end{multline}
where $\xi_\ell = \frac{1}{\sqrt{n}}(w_{\ell}-1) \funnoisew_{\ell}$. We finish the proof by the application of the formula \eqref{eq:shao_zhang_bound_bootstrap}. In order to bound the quantities $\PEb{\norm{D^\boot}^2}$ and $\PEb{\norm{D^{(\boot,i)}}^2}$, we apply the respective results of \Cref{th:theo_1_iid_boot} and \Cref{th:theo_iid_boot_differences}, respectively. Namely, applying the Cauchy-Schwartz inequality together with \Cref{th:theo_1_iid_boot}, we get that on the event $\Omega_0$ it holds
\begin{equation}
\label{eq:interm_term_bound_bootstrap}
\begin{split}
\PEb[\norm{D^\boot}\norm{W^\boot}] &\leq \bigl\{\PEb[\norm{D^\boot}^2]\bigr\}^{1/2} \{\PEb[\norm{W^\boot}^2]\}^{1/2} \lesssim \frac{\qcond^{2} (\bConst{A}^4 \vee 1)\supconsteps \sqrt{\trace{\noisecov^\boot}} \log{n}}{n^{1/4}a^{5/2}} \\
&\quad + \qcond^{3/2} \supconsteps \biggl(\frac{(\bConst{A}^3 \vee 1)n^{1/4}}{\sqrt{a}} +  \frac{(\bConst{A}^5 \vee 1)\sqrt{\log{n}}}{\sqrt{n} a}\biggr)\exp\biggl\{-\frac{c_0 a \sqrt{n}}{2}\biggr\} \norm{\theta_0 - \thetas}\eqsp.
\end{split}
\end{equation}
Similarly, applying Minkowski's inequality and \Cref{th:theo_iid_boot_differences}, we obtain that 
\begin{align}
&\PEb[\sum_{i=n}^{2n-1}\norm{\xi_{i}}\norm{D^\boot - D^{(\boot, i)}}] \leq \bigl\{\PEb[\norm{\xi_{1}}^2]\bigr\}^{1/2}\sum_{i=n}^{2n-1} \bigl\{\PEb[\norm{D^\boot - D^{(\boot, i)}}^2]\bigr\}^{1/2} \\
&\qquad \lesssim \frac{\qcond^{3/2} (\bConst{A}^3 \vee 1)\supconsteps^2 \log{n}}{a^{3/2}n^{1/4}} + \frac{\qcond \supconsteps^2\bConst{A}^2}{a^2 \sqrt{n}} + \frac{\qcond^{3/2} (\bConst{A}^{3} \vee 1) \supconsteps n^{1/4} \sqrt{\log{n}}}{a^{3/2}} \exp\bigl\{-\frac{a}{4}\sum_{j=1}^{n}\alpha_{j}\bigr\} \norm{\theta_0-\thetas}\eqsp.
\end{align}
Now it remains to combine the bounds above in \eqref{eq:shao_zhang_bound_bootstrap}.
\end{proof}

\begin{proposition}
\label{th:theo_1_iid_boot}
Assume \Cref{assum:iid}, \Cref{assum:noise-level}, \Cref{assum:step-size} with $\gamma = 1/2$, and \Cref{assum:step-size-bootstrap}. Then, conditionally on the event $\Omega_0$, the following error bound holds:
\begin{multline}
\label{eq:MSE_2nd_moment_decreasing_boot}
\bigl\{\PEb\left[\norm{D^\boot(w_1, \ldots, w_{2n}, \State_1,\ldots,\State_{2n})}^{2}\right]\bigr\}^{1/2} 
\lesssim \frac{\qcond^{2} (\bConst{A}^4 \vee 1) \supconsteps \log{n}}{n^{1/4}a^{5/2}} \\
+ \qcond^{3/2} \biggl(\frac{(\bConst{A}^3 \vee 1)n^{1/4}}{\sqrt{a}} + \frac{(\bConst{A}^5 \vee 1)\sqrt{\log{n}}}{\sqrt{n} a}\biggr)\exp\biggl\{-\frac{c_0 a \sqrt{n}}{2}\biggr\} \norm{\theta_0 - \thetas} \eqsp,
\end{multline}
where $\lesssim$ stands for inequality up to an absolute constant.
\end{proposition}
Proof of \Cref{th:theo_1_iid_boot} is provided below in \Cref{sec:proof_d_boot_bound}. The lemma below is a direct counterpart of \Cref{lem:D_i_bounds}. 
\begin{proposition}
\label{th:theo_iid_boot_differences}
Assume \Cref{assum:iid}, \Cref{assum:noise-level}, \Cref{assum:step-size} with $\gamma = 1/2$, and \Cref{assum:step-size-bootstrap}. Then, conditionally on the event $\Omega_0$, the following error bound holds:
\begin{equation}
\begin{split}
\sum_{i=n+1}^{2n}{\PEb[\norm{D^\boot - D^{(\boot, i)}}^2]}^{1/2} 
&\lesssim \frac{(\bConst{A}^3 \vee 1)\supconsteps}{a^{3/2}} n^{1/4}\log{n} + \frac{\supconsteps\bConst{A}^2}{a^2} \\
&\qquad + \frac{(\bConst{A}^{3} \vee 1) n^{3/4} \sqrt{\log{n}}}{a^{3/2}} \exp\bigl\{-\frac{a}{4}\sum_{j=1}^{n}\alpha_{j}\bigr\} \norm{\theta_0-\thetas}\eqsp,
\end{split}
\end{equation}    
\end{proposition}
Proof of \Cref{th:theo_iid_boot_differences} is provided below in \Cref{th:theo_iid_boot_differences_proof}.

\begin{lemma}
\label{lem:theta_b_k_theta_k_bound}
For any $k \geq n$ on the set $\Omega_0$ the following inequality holds:
$$
\PEb[\|\theta_{k}^\boot - \theta_k\|^2] \lesssim \frac{\alpha_{k} \supconsteps^2 \bConst{A}^2}{a^3} + \bConst{A}^{2} k \prod_{j = 1}^{k} (1 - a \alpha_j/4)^2\norm{\theta_0-\thetas}^2\eqsp. 
$$
\end{lemma}
\begin{proof}
A direct application of \Cref{lem:expansion} with $L = 0$ yields that
\begin{align}
\PEb[\|\theta_{k}^\boot - \theta_k\|^2] \lesssim \frac{\alpha_{k} \supconsteps^2 }{a} \left (1+ \frac{\bConst{A}^2}{a^2} \right ) +  \bConst{A}^{2} k \prod_{j = 1}^{k} (1 - a \alpha_j/4)^2\norm{\theta_0-\thetas}^2\eqsp. 
\end{align}
Now to complete the proof it remains to notice that $\bConst{A} \geq a$.
\end{proof}

\begin{lemma}
\label{lem:product_coupling_lemma}
For any matrix-valued sequences $(U_n)_{n\in \nset}$, $(V_n)_{n\in \nset}$ and for any $M \in \nset$, it holds that:
\begin{equation}
\prod_{k=1}^M U_k - \prod_{k=1}^M V_k = \sum_{k=1}^M \{\prod_{j=k+1}^M V_j\} (U_k - V_k) \{\prod_{j=1}^{k-1} U_j\}\eqsp.
\end{equation}
\end{lemma}

\subsection{Gaussian comparison inequality}
\begin{theorem}
\label{prop: gaussian comparison}
Assume \Cref{assum:iid} and \Cref{assum:noise-level}. Then on the set $\Omega_3$ 
\begin{equation}
\sup_{B \in \Conv(\rset^{d})} |\PP(\xi \in B) - \PPb(\xi^\boot \in B) |  \leq 4 \|\noisecov ^{-1/2} \funnoisew \|_{\infty} \sqrt{\frac{d\log{n}}{n}} + \frac{4 \sqrt d (1 + \|\noisecov ^{-1/2} \funnoisew \|_{\infty}^2 )\log{n}}{n}
\end{equation}

\end{theorem}

\begin{proof}
We will use the following inequality
\begin{equation}
\label{eq: Pinsker}
\| \mathcal N(0, \Sigma_1) - \mathcal N(0,\Sigma_2)\|_{\mathsf{TV}} \le \frac{1}{2} \| \Sigma_1^{-1/2} \Sigma_2 \Sigma^{-1/2} - \Id\|_{\mathsf{Fr}}
\end{equation}
Applying \eqref{eq: Pinsker} we obtain
$$
\sup_{B \in \Conv(\rset^{d})} |\PP(\xi \in B) - \PPb(\xi^\boot \in B)\| \le \frac{\sqrt d}{2}\|\noisecov^{-1/2}  \noisecov^\boot  \noisecov^{-1/2} - \Id \|.  
$$
It remains to apply definition of $\Omega_3$.
\end{proof}

\subsection{Auxiliary technical results.}
\label{sec:error_decomspotion_perturbed}
For the analysis of the difference term $\theta_{k}^\boot - \theta_{k}$ we use the perturbation expansion technique introduced in \cite{aguech2000perturbation}, see also \cite{durmus2022finite}. Within this approach, we represent the fluctuation component of the error $\vtheta_{n}$ defined in \eqref{eq:LSA_recursion_expanded} as 
\begin{equation}
\label{eq:decomp_fluctuation}
\vtheta_{n} = \Jnalpha{n}{0}+ \Hnalpha{n}{0} \eqsp,
\end{equation}
where the latter terms are defined by the following pair of recursions
\begin{align}
\label{eq:jn0_main}
&\Jnalpha{n}{0} =\left(\Id - \alpha_{n} \bA\right) \Jnalpha{n-1}{0} - \alpha_{n} \funcnoise{\State_{n}}\eqsp, && \Jnalpha{0}{0}=0\eqsp, \\[.1cm]
\label{eq:hn0_main}
&\Hnalpha{n}{0} =\left( \Id - \alpha_{n} \funcA{\State_{n}} \right) \Hnalpha{n-1}{0} - \alpha_{n} \zmfuncA{\State_{n}} \Jnalpha{n-1}{0}\eqsp, && \Hnalpha{0}{0}=0\eqsp.
\end{align}
Moreover, it is known that for $L \geq 1$ the term $\Hnalpha{n}{0}$ can be further decomposed as follows:
\begin{equation}
\label{eq:error_decomposition_LSA}
\Hnalpha{n}{0} = \sum_{\ell=1}^{L}\Jnalpha{n}{\ell} + \Hnalpha{n}{L}\eqsp.
\end{equation}
Here the terms $\Jnalpha{n}{\ell}$ and $\Hnalpha{n}{\ell}$ are given by the following recurrences:
\begin{equation}
\label{eq:jn_allexpansion_main}
\begin{aligned}
&\Jnalpha{n}{\ell} =\left(\Id - \alpha_{n} \bA\right) \Jnalpha{n-1}{\ell} - \alpha_{n} \zmfuncA{\State_{n}} \Jnalpha{n-1}{\ell-1}\eqsp,
&& \Jnalpha{0}{\ell}=0  \eqsp, \\
& \Hnalpha{n}{\ell} =\left( \Id - \alpha_{n} \funcA{\State_{n}} \right) \Hnalpha{n-1}{\ell} - \alpha_{n} \zmfuncA{\State_{n}} \Jnalpha{n-1}{\ell} \eqsp, && \Hnalpha{0}{\ell}=0 \eqsp.
\end{aligned}
\end{equation}
The expansion depth $L$ here controls the desired approximation accuracy. Informally, one can show that $\PE^{1/p}[\norm{\Jnalpha{n}{\ell}}^{p}] \lesssim \alpha_{n}^{(\ell+1)/2}$, and similarly $\PE^{1/p}[\norm{\Hnalpha{n}{\ell}}^{p}] \lesssim \alpha_{n}^{(\ell+1)/2}$. Using the outlined expansion, we prove the following lemma:
\begin{lemma}
\label{lem:expansion}
Assume \Cref{assum:iid}, \Cref{assum:noise-level}, \Cref{assum:step-size} with $\gamma = 1/2$, and \Cref{assum:step-size-bootstrap}. Then for any $k \geq n$ and $L \in \nset$ the following decomposition holds:
\begin{equation}
\label{eq: decomposotion of error bewtween b and r worlds}
\theta_{k}^\boot - \theta_{k} = J_k^{\boot, 0} + \sum_{j=1}^L J_k^{\boot, j}  + H_k^{\boot, L}, 
\end{equation}
where 
\begin{align}
   J_k^{\boot, 0} &= - \sum_{\ell=n+1}^k \alpha_\ell (w_\ell - 1)\ProdB_{\ell+1:k} \tilde \funnoisew_{\ell}, \\ 
   J_k^{\boot, j} &= - \sum_{\ell=n+1}^k \alpha_\ell (w_\ell - 1)\ProdB_{\ell+1:k} A_\ell J_{\ell-1}^{\boot, j-1}, \quad j \in [1, L] \\
   \label{eq:H_k_boot_def}
   H_k^{\boot, L} & = -\sum_{\ell=n+1}^k \alpha_\ell (w_\ell - 1)\ProdB^{\boot}_{\ell+1:k} A_\ell J_{\ell-1}^{\boot, L}\eqsp,
\end{align}
and the quantities $\tilde \funnoisew_{\ell}$ are defined as 
$$
\tilde \funnoisew_{\ell} =  \funcAw_{\ell} (\theta_{\ell-1} - \thetas) + \funnoisew_{\ell}\eqsp.
$$ 
Moreover, on the event $\Omega_0$,
\begin{align}
\label{eq:J_k_boot_bound}
\PEb[\|J_k^{\boot, j}\|^2 ] &\lesssim \frac{\alpha_{k}^{j+1} \supconsteps^2 \bConst{A}^{2j}}{a^{j+1}} + \bConst{A}^{2j+2} \prod_{j = 1}^{k} (1 - a \alpha_j/4)^2\norm{\theta_0-\thetas}^2\eqsp, \quad j \in [0, L] \\
\label{eq:H_k_boot_bound}
\PEb[\|H_k^{\boot, L}\|^2 ] &\lesssim \frac{\alpha_{k}^{L+1}\bConst{A}^{2(L+1)}\supconsteps^2}{a^{L+3}} + \bConst{A}^{2(L+1)} k \prod_{j = 1}^{k} (1 - a \alpha_j/4)^2 \norm{\theta_0-\thetas}^2\eqsp.
\end{align}
\end{lemma}
\begin{proof}
We start from the decomposition 
\begin{equation}
\label{eq:one_step_expand}
\theta_{k}^\boot - \theta_{k} = (\Id - \alpha_{k} w_k \funcAw_k)(\theta_{k-1}^\boot - \theta_{k-1})  - \alpha_{k} (w_k-1) \tilde \funnoisew_{k}.
\end{equation}
Expanding the recurrence above till $k = n$, and using the fact that $\theta_{n}^\boot = \theta_{n}$, we get running the recurrence \eqref{eq:one_step_expand}, that 
\[
\theta_{k}^\boot - \theta_{k} = -\sum_{\ell=n+1}^k \alpha_\ell (w_\ell - 1)\ProdB^{\boot}_{\ell+1:k} \tilde \funnoisew_{\ell}\eqsp.
\]
Hence, proceeding as in \eqref{eq:jn0_main}, we obtain the representation
\begin{align}
\label{eq:jn0_bootstrap}
&\Jnalpha{k}{\boot,0} =\left(\Id - \alpha_{k} \funcAw_k \right) \Jnalpha{k-1}{\boot,0} - \alpha_{k}(w_k - 1)\tilde{\funnoisew_{k}}\eqsp, && \Jnalpha{0}{\boot,0}=0\eqsp, \\[.1cm]
\label{eq:hn0_main}
&\Hnalpha{k}{\boot,0} =\left( \Id - \alpha_{k} 
 w_{k} \funcAw_k \right) \Hnalpha{k-1}{\boot,0} - \alpha_{k} (w_k - 1) \funcAw_k \Jnalpha{k-1}{\boot,0}\eqsp, && \Hnalpha{0}{\boot,0}=0\eqsp.
\end{align}
It is easy to check that $\Jnalpha{k}{\boot,0} + \Hnalpha{k}{\boot,0} = \theta_{k}^\boot - \theta_{k}$. Similarly, with further expansion of $\Hnalpha{k}{\boot,0}$ along the lines of \eqref{eq:jn_allexpansion_main}, we arrive at the decomposition \eqref{eq: decomposotion of error bewtween b and r worlds}. Since $w_k$ for $k = n+1,\ldots,2n$ are i.i.d., we get using the definition of the events $\Omega_{1}$ and $\Omega_{2}$, that on the event $\Omega_0$:
\begin{equation}
\label{eq:eps_tilde_bound}
\begin{split}
\norm{\tilde \funnoisew_{\ell}}^2 
&\lesssim \supconsteps^2 + \bConst{A}^2 \exp\bigl\{-a\sum_{j=1}^{\ell-1}\alpha_{j}\bigr\} \norm{\theta_0 - \thetas}^2  + \frac{\alpha_{\ell} \supconsteps^2 \log^2{n}}{a} \\
&\lesssim \supconsteps^2 + \bConst{A}^2 \prod_{j=1}^{\ell-1}(1-a\alpha_{j}/2)^2 \norm{\theta_0 - \thetas}^2 + \frac{\alpha_{\ell} \supconsteps^2 \log^2{n}}{a} \\
&\lesssim \supconsteps^2 + \bConst{A}^2 \prod_{j=1}^{\ell-1}(1-a\alpha_{j}/2)^2 \norm{\theta_0 - \thetas}^2\eqsp,
\end{split}
\end{equation}
where for the last bound we have additionally used that $\alpha_{\ell} \log^{2}{n} / a \leq 1$ for $\ell \geq n$. The latter bound is guaranteed by \Cref{assum:step-size-bootstrap}. Hence, using the bound \eqref{eq:eps_tilde_bound} together with the definition of $J_k^{\boot, 0}$, we obtain that 
\begin{align}
\PEb[\|J_k^{\boot, 0}\|^2 ] & = \sum_{\ell=n+1}^k \alpha_\ell^2 \| \ProdB_{\ell+1:k} \tilde \funnoisew_{\ell}\|^2 = 
\sum_{\ell=n+1}^{k} \alpha_{\ell}^2 \norm{\ProdB_{\ell+1:k}\bigl(\funcAw_{\ell} (\theta_{\ell-1} - \thetas) + \funnoisew_{\ell}\bigr)}^2 \\
&\lesssim \supconsteps^2 \sum_{\ell=n+1}^k \alpha_\ell^2 \prod_{j = \ell + 1}^k (1 - a \alpha_j/4)^2 + \bConst{A}^2 \sum_{\ell=n+1}^k \alpha_\ell^2 \prod_{j = 1}^{k} (1 - a \alpha_j/4)^2\norm{\theta_0 - \thetas}^2 \\
&\qquad\qquad\qquad + \frac{\supconsteps^2 \bConst{A}^2 \log^2 n}{a} \sum_{\ell=n+1}^k \alpha_{\ell}^{3} \prod_{j = \ell + 1}^k (1 - a \alpha_j/4)^2 \\
& \lesssim \frac{\supconsteps^2 \alpha_{k}}{a} + \bConst{A}^2 \log{\left(\frac{k}{n}\right)} \prod_{j = 1}^{k} (1 - a \alpha_j/4)^2 \norm{\theta_0 - \thetas}^2 + \frac{\alpha_{k}^2 \supconsteps^2 \bConst{A}^2 \log^2{n}}{a^2} \\
& \lesssim \frac{\supconsteps^2 \alpha_{k}}{a} + \bConst{A}^2 \prod_{j = 1}^{k} (1 - a \alpha_j/4)^2 \norm{\theta_0 - \thetas}^2\eqsp,
\end{align}
where we additionally used the fact that $k \in [n;2n]$ and $n$ satisfies \Cref{assum:step-size-bootstrap}. Assume now that the bound on $J_k^{\boot, j-1}$ has a form
\[
\PEb[\|J_k^{\boot,j-1}\|^2 ] \lesssim \frac{\supconsteps^2 \alpha_{k}^{j}}{a^{j}} + \bConst{A}^{2j} \prod_{\ell = 1}^{k} (1 - a \alpha_{\ell}/4)^2 \norm{\theta_0 - \thetas}^2\eqsp.
\]
Then, using the martingale property of $J_k^{\boot, j}$, we write that 
\begin{align}
\PEb[\|J_k^{\boot, j}\|^2 ] 
&= \sum_{\ell=n+1}^k \alpha_\ell^2 \PEb[\| \ProdB_{\ell+1:k} A_\ell J_{\ell-1}^{\boot, j-1}\|^2] \\
&\lesssim \sum_{\ell=n+1}^k \frac{\alpha_{\ell}^{j+2} \supconsteps^2 \bConst{A}^{2j}}{a^{j}} \prod_{j = \ell + 1}^k (1 - a \alpha_j/4)^2 + \bConst{A}^{2j+2}\sum_{\ell=n+1}^{k}\alpha_{\ell}^2 \prod_{j = 1}^{k} (1 - a \alpha_j/4)^2\norm{\theta_0-\thetas}^2 \\
&\lesssim \frac{\alpha_{k}^{j+1} \supconsteps^2 \bConst{A}^{2j}}{a^{j+1}} + \bConst{A}^{2j+2} \prod_{j = 1}^{k} (1 - a \alpha_j/4)^2\norm{\theta_0-\thetas}^2\eqsp,
\end{align}
and thus the bound \eqref{eq:J_k_boot_bound} is proved. Moreover, using \eqref{eq:H_k_boot_def} and Minkowski's inequality, we obtain that 
\begin{align}
(\PEb[\|H_k^{\boot, L} \|^2] )^{1/2} 
&\leq \bConst{A} \sum_{\ell=n+1}^k \alpha _\ell (\PEb[\|\ProdB^{\boot}_{\ell+1:k} \|^2] )^{1/2}  (\PEb[\|J_{\ell-1}^{\boot, L} \|^2] )^{1/2} \\
&\lesssim \bConst{A} \sum_{\ell=n+1}^{k}\frac{\alpha_{\ell}^{(L+3)/2} \supconsteps \bConst{A}^{L}}{a^{(L+1)/2}} \prod_{j = \ell + 1}^k (1 - a \alpha_j/4) \\
& \qquad \qquad \qquad + \bConst{A}^{L+1} \sum_{\ell=n+1}^{k}\alpha_{\ell} \prod_{j = 1}^{k} (1 - a \alpha_j/4)^2\norm{\theta_0-\thetas} \\
&\lesssim \frac{\alpha_{k}^{(L+1)/2}\bConst{A}^{L+1}\supconsteps}{a^{(L+3)/2}} + \bConst{A}^{L+1} \sqrt{k} \prod_{j = 1}^{k} (1 - a \alpha_j/4)^2 \norm{\theta_0-\thetas}\eqsp,
\end{align}
and \eqref{eq:H_k_boot_bound} follows.
\end{proof}

\subsection{Proof of \Cref{th:theo_1_iid_boot}}
\label{sec:proof_d_boot_bound}
Recall that the quantity $D^\boot$ is defined in \eqref{eq:D_boot_def}. Since $\theta_n^\boot = \theta_n$, we conclude that $D_1^\boot = 0$. To estimate other terms we will use the main error decomposition outlined in \Cref{lem:expansion}, that is, the expansion
$$
\theta_{k}^\boot - \theta_{k} = \sum_{\ell = 0}^L J_k^{\boot, \ell}   + H_k^{\boot, L}, 
$$
applied with different $L \geq 0$. To bound $D_2^\boot$ we take $L = 0$ and obtain
\begin{align}
\PEb[\|D_2^\boot\|^2] 
& \lesssim \PEb[\|J_{2n}^{\boot, j}\|^2] + \PEb[\|H_{2n}^{\boot, L}\|^2] \lesssim \frac{\supconsteps^2}{n \alpha_{2n} a}\left(1 + \frac{\bConst{A}^2}{a^2}\right) + \frac{\bConst{A}^2}{\alpha_{2n}^2} \prod_{j = 1}^{2n} (1 - a \alpha_j/4)^2 \norm{\theta_0-\thetas}^2 \\
&\lesssim \frac{\supconsteps^2}{a \sqrt{n}}\left(1 + \frac{\bConst{A}^2}{a^2}\right) + n \bConst{A}^2 \prod_{j = 1}^{2n} (1 - a \alpha_j/4)^2 \norm{\theta_0-\thetas}^2\eqsp. 
\end{align}
To estimate $D_3^\boot$ we note that
$$
D_{3}^\boot  =  \frac{1}{\sqrt{n}}\sum_{k=n+1}^{2n}(w_k-1) \funcAw_k (\theta_{k-1}^\boot - \thetas) = D_{3,1}^\boot + D_{3,2}^\boot\eqsp,
$$
where we have set, respectively, 
\begin{align}
D_{3,1}^\boot & = \frac{1}{\sqrt{n}}\sum_{k=n+1}^{2n}(w_k-1) \funcAw_k (\theta_{k-1}^\boot - \theta_{k-1}), \\
D_{3,2}^\boot&=\frac{1}{\sqrt{n}}\sum_{k=n+1}^{2n}(w_k-1) \funcAw_k (\theta_{k-1} - \thetas)\eqsp.
\end{align}
It follows from \Cref{lem:theta_b_k_theta_k_bound} that on the event $\Omega_0$ it holds
\begin{align}
\PEb[\norm{D_{3,1}^\boot}^2] &\leq \frac{\bConst{A}^2}{n} \sum_{k=n+1}^{2n} \PEb[\norm{\theta_{k-1}^\boot - \theta_{k-1}}^2] \\
&\lesssim \frac{\bConst{A}^2}{n}\sum_{k=n+1}^{2n}\frac{\alpha_{k} \supconsteps^2 }{a} \left(1+ \frac{\bConst{A}^2}{a^2}\right) + \bConst{A}^{4}\sum_{k=n+1}^{2n}\prod_{j = 1}^{k} (1 - a \alpha_j/4)^2\norm{\theta_0-\thetas}^2 \\
&\lesssim \frac{\bConst{A}^2 \supconsteps^2}{a\sqrt{n}}\left(1+ \frac{\bConst{A}^2}{a^2}\right) + \frac{\bConst{A}^{4} \sqrt{n}}{a}\exp\biggl\{-c_0 a \sqrt{n} \biggr\} \norm{\theta_0 - \thetas}^2
\end{align}
Moreover, on the set $\Omega_0$ it holds (since $\Omega_0 \subseteq \Omega_1$), that
\begin{align}
\PEb[\|D_{3,2}^\boot \|^2] 
&= \frac{1}{n}\sum_{k=n+1}^{2n}\norm{\funcAw_k (\theta_{k-1} - \thetas)}^2 \\
&\lesssim \frac{\bConst{A}^2}{n}\sum_{k=n+1}^{2n}\left(\exp\bigl\{- a \sum_{\ell=1}^{k}\alpha_\ell\bigr\}\norm{\theta_0 - \thetas}^2 + \frac{\alpha_{k} \supconsteps^2 \log^2 n}{a}\right) \\
&\lesssim \frac{\bConst{A}^2}{n a \alpha_{2n}} \exp\bigl\{- a \sum_{\ell=1}^{n}\alpha_\ell\bigr\}\norm{\theta_0 - \thetas}^2 + \frac{\bConst{A}^2 \supconsteps^2 \log^2 n}{n a}\sum_{k=n+1}^{2n}\alpha_{k} \\
&\lesssim \frac{\bConst{A}^2}{ a \sqrt{n}} \exp\biggl\{-c_0 a \sqrt{n} \biggr\} \norm{\theta_0 - \thetas}^2 + \frac{\bConst{A}^2 \supconsteps^2 \log^2 n}{a\sqrt{n}}\eqsp.
\end{align}
Combining the above bounds, we get 
\begin{align}
\PEb[\norm{D_{3}^\boot}^2] 
&\lesssim \PEb[\|D_{3,1}^\boot \|^2] + \PEb[\|D_{3,2}^\boot \|^2] \\ 
&\lesssim \frac{\bConst{A}^2 \supconsteps^2 \log^{2}{n}}{a\sqrt{n}}\left(1+ \frac{\bConst{A}^2}{a^2}\right) +  \frac{\bConst{A}^{4} \sqrt{n}}{a}\exp\biggl\{-c_0 a \sqrt{n} \biggr\} \norm{\theta_0 - \thetas}^2\eqsp.
\end{align}
Now we proceed with the term $D_4^\boot$. Applying Minkowski's inequality, we get 
\begin{align}
\{\PEb[\norm{D_{4}^\boot}^2]\}^{1/2} 
&\leq \frac{1}{\sqrt{n}} \sum_{k=n+1}^{2n}\left(\frac{1}{\alpha_k} - \frac{1}{\alpha_{k-1}}\right)\{\PEb[\norm{\theta_{k-1}^\boot - \theta_{k-1}}^2]\}^{1/2} \\
&\lesssim \frac{1}{\sqrt{n}} \sum_{k=n+1}^{2n} \frac{1}{\sqrt{k}}\left(\frac{\sqrt{\alpha_{k}} \supconsteps }{\sqrt{a}} \left(1+ \frac{\bConst{A}}{a} \right) +  \bConst{A} \sqrt{k} \prod_{j = 1}^{k} (1 - a \alpha_j/4) \norm{\theta_0-\thetas}\right) \\
&\lesssim \frac{\supconsteps}{n^{1/4}\sqrt{a}}\left(1+ \frac{\bConst{A}}{a} \right) + \frac{\bConst{A}}{\sqrt{n}}\sum_{k=n+1}^{2n}\prod_{j = 1}^{k} (1 - a \alpha_j/4) \norm{\theta_0-\thetas} \\
&\lesssim \frac{\supconsteps}{n^{1/4}\sqrt{a}}\left(1+ \frac{\bConst{A}}{a}\right) + \bConst{A}\exp\biggl\{-\frac{c_0 a \sqrt{n}}{2}\biggr\}\norm{\theta_0 - \thetas}\eqsp.
\end{align}
It remains to upper bound the term $D_5^\boot$. Using the decomposition, suggested by \Cref{lem:expansion} with $L = 2$, we get that  
\begin{align}
D_5^\boot 
&= \frac{1}{\sqrt{n}}\sum_{k=n+1}^{2n} (\funcAw_k - \bA) (\theta_{k-1}^\boot - \theta_{k-1}) = \underbrace{\frac{1}{\sqrt{n}}\sum_{k=n+1}^{2n} (\funcAw_k - \bA) J_{k-1}^{\boot, 0}}_{D_{5,1}^\boot} \\
&\qquad +  \underbrace{\frac{1}{\sqrt{n}}\sum_{k=n+1}^{2n} (\funcAw_k - \bA) J_{k-1}^{\boot, 1}}_{D_{5,2}^\boot} + \underbrace{\frac{1}{\sqrt{n}}\sum_{k=n+1}^{2n} (\funcAw_k - \bA) J_{k-1}^{\boot, 2}}_{D_{5,3}^\boot} + \underbrace{\frac{1}{\sqrt{n}}\sum_{k=n+1}^{2n} (\funcAw_k - \bA) H_{k-1}^{\boot, 2}}_{D_{5,4}^\boot}\eqsp.
\end{align}
Here we have to consider expansion until $H^{\boot, 2}$, since dealing with the latter term (outlined as $D_{5,4}^\boot$ in the above expansion) is possible only with Minkowski's inequality. Now we consider the summands $D_{5,1}^\boot - D_{5,4}^\boot$ separately. Consider first the term $D_{5,1}^\boot$. Changing the summation order, we obtain
\begin{align}
D_{5,1}^\boot  
&= -\frac{1}{\sqrt{n}}\sum_{k=n+1}^{2n} (\funcAw_k - \bA)\sum_{\ell=n+1}^{k-1} \alpha_\ell (w_\ell - 1)\ProdB_{\ell+1:k-1} \tilde \funnoisew_{\ell} \\
&= -\frac{1}{\sqrt{n}} \sum_{\ell = n+1}^{2n-1} \alpha_\ell (w_\ell - 1)  \bigg( \sum_{k = \ell+1}^{2n} (\funcAw_k - \bA) \ProdB_{\ell+1:k-1} \bigg )\tilde \funnoisew_{\ell}\eqsp.
\end{align}
Then on the event $\Omega_0$ we get, since $\Omega_0 \subseteq \Omega_4$, and using that $n$ satisfies \Cref{assum:step-size-bootstrap}, 
\begin{equation}
\label{eq:aux_norm_martingale_increment_bound}
\norm{\sum_{k = \ell+1}^{2n} (\funcAw_k - \bA) \ProdB_{\ell+1:k-1}}^2 \lesssim \frac{\log{n}}{a \alpha_{\ell}}\eqsp.
\end{equation}
Combining the above bound together with the one provided by \eqref{eq:eps_tilde_bound}, we obtain that 
\begin{align}
\PEb[\|D_{5,1}^\boot \|^2] 
&\lesssim \frac{\supconsteps^2}{n} \sum_{\ell = n+1}^{2n-1}\frac{\alpha_{\ell}\log{n}}{a} + \frac{\bConst{A}^2}{n}\sum_{\ell=n+1}^{2n-1}\frac{\alpha_{\ell}\log{n}}{a}\prod_{j=1}^{\ell-1}\bigl(1 - a \alpha_{j}/2\bigr)^2 \norm{\theta_0 - \thetas}^2 \\
&\lesssim \frac{\supconsteps^2 \log{n}}{\sqrt{n} a} + \frac{\bConst{A}^2 \log{n}}{a^2 n} \exp\bigl\{ -a\sum_{j=1}^{n}\alpha_{j} \bigr\} \norm{\theta_0 - \thetas}^2\eqsp.
\end{align}
Similarly, for the term $D_{5,2}^\boot$ we get, changing the order of summation, that 
$$
D_{5,2}^\boot  = \frac{1}{\sqrt{n}} \sum_{\ell = n+1}^{2n-1} \alpha_\ell (w_\ell - 1)  \bigg( \sum_{k = \ell+1}^{2n} (\funcAw_k - \bA) \ProdB_{\ell+1:k-1} \bigg )A_\ell J_{\ell-1}^{\boot, 0}\eqsp.
$$
Hence, using the bound \eqref{eq:aux_norm_martingale_increment_bound} together with \eqref{eq:J_k_boot_bound}, we get
\begin{align}
\PEb[\|D_{5,2}^\boot \|^2] 
&\lesssim \frac{1}{n}\sum_{\ell=n+1}^{2n-1} \frac{\alpha_{\ell} \log{n}}{a} \bConst{A}^2 \biggl(\frac{\alpha_{\ell} \supconsteps^2}{a} + \bConst{A}^2 \prod_{j = 1}^{\ell-1} (1 - a \alpha_j/4)^2 \norm{\theta_0-\thetas}^2 \biggr) \\
&\lesssim \frac{\bConst{A}^2 \supconsteps^2 \log{n}}{n a^2}\sum_{\ell=n+1}^{2n-1}\alpha_{\ell}^2 + \frac{\bConst{A}^4 \log{n}}{n a} \sum_{\ell=n+1}^{2n-1}\alpha_{\ell} \prod_{j = 1}^{\ell-1} (1 - a \alpha_j/4)^2 \norm{\theta_0-\thetas}^2 \\
&\lesssim \frac{\bConst{A}^2 \supconsteps^2 \log{n}}{n a^2} + \frac{\bConst{A}^4 \log{n}}{n a} \exp\bigl\{ -(a/2)\sum_{j=1}^{n}\alpha_{j} \bigr\} \norm{\theta_0 - \thetas}^2\eqsp.
\end{align}
We proceed with $D_{5,3}^\boot$. We change the summation order and proceed exactly as with $D_{5,2}^\boot$. Indeed, 
$$
D_{5,3}^\boot  = \frac{1}{\sqrt{n}} \sum_{\ell = n+1}^{2n-1} \alpha_\ell (w_\ell - 1)  \bigg( \sum_{k = \ell+1}^{2n} (\funcAw_k - \bA) \ProdB_{\ell+1:k-1} \bigg )A_\ell J_{\ell-1}^{\boot, 1}\eqsp,
$$
and
\begin{align}
\PEb[\|D_{5,3}^\boot \|^2] 
&\lesssim \frac{1}{n}\sum_{\ell=n+1}^{2n-1} \frac{\alpha_{\ell} \log{n}}{a} \bConst{A}^2 \biggl(\frac{\alpha_{\ell}^2 \supconsteps^2 \bConst{A}^2}{a^2} + \bConst{A}^4 \prod_{j = 1}^{\ell-1} (1 - a \alpha_j/4)^2 \norm{\theta_0-\thetas}^2 \biggr) \\ 
&\lesssim \frac{\bConst{A}^4 \supconsteps^2 \log{n}}{n^{3/2}a^3} + \frac{\bConst{A}^6\log{n}}{na}  \exp\bigl\{ -(a/2)\sum_{j=1}^{n}\alpha_{j} \bigr\} \norm{\theta_0 - \thetas}^2\eqsp.
\end{align}
It remains to upper bound $D_{5,4}^\boot$. Proceeding as above, we change the summation order, and obtain
$$
D_{5,4}^\boot  = \frac{1}{\sqrt{n}} \sum_{\ell = n+1}^{2n-1} \alpha_\ell (w_\ell - 1)  \bigg( \sum_{k = \ell+1}^{2n} (\funcAw_k - \bA) \ProdB^\boot_{\ell+1:k-1} \bigg )A_\ell J_{\ell-1}^{\boot, 2}\eqsp. 
$$
Applying Minkowski's inequality, we get
\begin{align}
(\PEb[\|D_{5,4}^\boot \|^2])^{1/2} 
&\lesssim \frac{\bConst{A}^2}{\sqrt n}  \sum_{\ell = n+1}^{2n-1} \alpha_\ell \sum_{k = \ell+1}^{2n} (\PEb[\|\ProdB^\boot_{\ell+1:k-1} \|^2])^{1/2} (\PEb[\norm{J_{\ell-1}^{\boot, 2}}^2])^{1/2} \\
&\lesssim \frac{\bConst{A}^2}{\sqrt n} \sum_{\ell = n+1}^{2n-1} \alpha_\ell^{5/2} \sum_{k = \ell+1}^{2n} \exp\bigl\{-\frac{a}{4}\sum_{j=\ell+1}^{k-1}\alpha_{j} \bigr\}\frac{\supconsteps \bConst{A}^{2}}{a^{3/2}} \\
&\qquad\qquad\qquad\qquad + \frac{\bConst{A}^5}{\sqrt n} \sum_{\ell = n+1}^{2n-1} \alpha_{\ell} \sum_{k = \ell+1}^{2n-1} \exp\bigl\{-\frac{a}{4}\sum_{j=1}^{k-1}\alpha_{j} \bigr\}\norm{\theta_0 - \thetas} \\
&\lesssim \frac{\bConst{A}^4 \supconsteps}{\sqrt{n}a^{5/2}}\sum_{\ell = n+1}^{2n-1}\alpha_{\ell}^{3/2} + \frac{\bConst{A}^5}{\sqrt{n} a} \exp\bigl\{-\frac{a}{4}\sum_{j=1}^{n}\alpha_{j} \bigr\}\norm{\theta_0 - \thetas} \\
&\lesssim \frac{\bConst{A}^4 \supconsteps}{n^{1/4} a^{5/2}} + \frac{\bConst{A}^5}{\sqrt{n} a} \exp\bigl\{-\frac{a}{4}\sum_{j=1}^{n}\alpha_{j} \bigr\}\norm{\theta_0 - \thetas}\eqsp.
\end{align}
Now the result follows from the representation \eqref{eq:D_boot_def} and combinations of the above bounds for $D_{1}^\boot$ -- $D_{5}^\boot$.

\subsection{Proof of \Cref{th:theo_iid_boot_differences}}
\label{th:theo_iid_boot_differences_proof}
Consider the sequences of weights 
\begin{equation}
\label{eq:weights_one_pose_difference}
(w_1,\ldots,w_{i-1},w_{i},w_{i+1},\ldots,w_{2n}) \text{ and } (w_1,\ldots,w_{i-1},w_{i}^{\prime},w_{i+1},\ldots,w_{2n})\eqsp,
\end{equation}
which differs only in position $i$, $n+1 \leq i \leq 2n$, with $w_{i}^{\prime}$ being an independent copy of $w_{i}$. Consider the associated SA processes 
\begin{equation}
\label{eq:coupled_processes boot}
\begin{split}
\theta_{k}^\boot &= \theta_{k-1}^\boot - \alpha_{k} w_k \{ \funcA{Z_k} \theta_{k-1} - \funcb{Z_k} \}\eqsp, \quad k \geq n+1, \quad \theta_{n}^\boot = \theta_{n} \in \rset^{d} \\
\theta^{(\boot, i)}_{k} &= \theta^{(\boot, i)}_{k-1} - \alpha_{k} w_k^{(i)} \{ \funcA{Z_k} \theta^{(\boot, i)}_{k-1} - \funcb{Z_k} \}\eqsp, \quad k \geq n+1\eqsp, \quad \theta^{(\boot, i)}_{n} = \theta_{n} \in \rset^{d}\eqsp, 
\end{split}
\end{equation}
where $w_k^{(i)} = w_k$ for $k \neq i$ and $w_i^{(i)} = w_{i}^{\prime}$. Respective random variables $D^\boot$ and $D^{(\boot, i)}$ are based on the first and second sequences from \eqref{eq:weights_one_pose_difference}, respectively, and are constructed according to the equation \eqref{eq:error_decomposition_boot2}. From the above representations we easily observe that $\theta_{k}^\boot = \theta^{(\boot, i)}_{k}$ for $k < i$, moreover, 
\begin{align}
\label{eq:bound_coupled_pair boot}
\theta_{i}^\boot - \theta^{(\boot, i)}_{i} 
&= -\alpha_{i} (w_i - w_i^{\prime})\bigl\{ \funcA{\State_i})\theta_{i-1}^\boot - \funcb{\State_i}\bigr\} \\
&= -\alpha_{i} (w_i - w_i^{\prime})\bigl\{ \funcA{\State_i})(\theta_{i-1}^\boot - \theta_{i-1}) - \funcb{\State_i}\bigr\} \\
&= -\alpha_{i} (w_i - w_i^{\prime})\bigl\{ \funcA{\State_i})(\theta_{i-1}^\boot - \theta_{i-1}) + \tilde\funnoisew_{i}\bigr\}\eqsp.
\end{align}
where $\funnoisew_i= \funcnoise{\State_i}$ and $\funnoisew_{i}^{\prime}= \funcnoise{\State_i'}$. From the above representation we get, applying \Cref{lem:theta_b_k_theta_k_bound} and \eqref{eq:eps_tilde_bound}, that
\begin{align}
\label{eq:theta_i_diff_bound boot}
\{\PEb[\norm{\theta_{i}^\boot - \theta^{(\boot, i)}_{i}}^2]\}^{1/2} &\lesssim \alpha_{i} \bConst{A} \{\PEb[\norm{\theta_i^\boot - \theta_i}^2]\}^{1/2} + \alpha_i \bConst{A} \| \tilde \funnoisew_i\| \\
& \lesssim \alpha_i \bConst{A} \supconsteps + \frac{\alpha_{i}^{3/2} \supconsteps \bConst{A}^2}{a^{3/2}} + (\bConst{A}^{2} \vee 1) \sqrt{i} \prod_{j = 1}^{i} (1 - a \alpha_j/4)\norm{\theta_0-\thetas} \\
& \lesssim \frac{\alpha_{i}\supconsteps\bConst{A}^2}{a} + (\bConst{A}^{2} \vee 1) \sqrt{i} \prod_{j = 1}^{i} (1 - a \alpha_j/4)\norm{\theta_0-\thetas}\eqsp,
\end{align}
where for the last line we have additionally assumed that $\alpha_{i} \lesssim a$ for $i \geq n$. Moreover, for any $j > i$ one observes, expanding \eqref{eq:coupled_processes boot}, that
\begin{equation}
\label{eq:sync_coupled_chains boot}
\theta_{j}^\boot - \theta^{(\boot, i)}_{j} = \biggl\{\prod_{k=i+1}^{j}(\Id - \alpha_{k} w_k\funcA{\State_k})\biggr\}(\theta_{i}^\boot - \theta^{(\boot, i)}_{i}) = \ProdB^\boot_{i+1:j}\,(\theta_{i}^\boot - \theta^{(\boot, i)}_{i})\eqsp.
\end{equation}
Thus, similarly to \eqref{eq:D_D_i_bound}, we obtain that 
\begin{equation}
\label{eq:D_D_i_bound_bootstrap}
\{\PEb[\norm{D^{\boot} - D^{(\boot, i)}}^2]\}^{1/2} \leq \sum_{j=1}^{5}\{\PEb[\norm{D_{j}^{\boot} - D_{j}^{(\boot, i)}}^2]\}^{1/2}\eqsp,
\end{equation}
and bound the respective differences separately. By the construction of the process above, we note that $D_{1}^{\boot} = D_{1}^{(\boot, i)}$. Proceeding further, and using the equation \eqref{eq:error_decomposition_boot2}, we obtain that
\begin{align}
\{\PEb[\norm{D_2^{\boot} - D_2^{(\boot, i)}}^2]\}^{1/2} 
&= \frac{1}{\sqrt{n}\alpha_{2n}} \bigl\{\PEb[\norm{\theta_{2n}^\boot - \theta^{(\boot, i)}_{2n}}^2]\bigr\}^{1/2} \\
&\leq \frac{1}{\sqrt{n}\alpha_{2n}} \bigl\{\PEb[\norm{\ProdB^\boot_{i+1:2n}}^2]\bigr\}^{1/2} \bigl\{\PEb[\norm{\theta_{i}^\boot - \theta^{(\boot, i)}_{i}}^2]\bigr\}^{1/2} \\
&\lesssim \frac{\alpha_{i} \supconsteps \bConst{A}^2}{\sqrt{n}\alpha_{2n} a} \exp\bigl\{-\frac{a}{4}\sum_{j=i+1}^{2n}\alpha_{j}\bigr\} + \frac{(\bConst{A}^2 \vee 1) \sqrt{i}}{\sqrt{n}\alpha_{2n}} \exp\bigl\{-\frac{a}{4}\sum_{j=1}^{2n}\alpha_{j}\bigr\}\norm{\theta_0 - \thetas}\eqsp.
\end{align}
Thus, taking sum for $i$ from $n+1$ to $2n$, and applying \Cref{lem:summ_alpha_k}, we get that 
\begin{align}
\sum_{i=n+1}^{2n}\{\PEb[\norm{D_2^{\boot} - D_2^{(\boot, i)}}^2]\}^{1/2} &\lesssim \sum_{i=n+1}^{2n} \frac{\alpha_{i} \supconsteps \bConst{A}^2}{\sqrt{n}\alpha_{2n} a} \exp\bigl\{-\frac{a}{4}\sum_{j=i+1}^{2n}\alpha_{j}\bigr\} \\
&\qquad\qquad +\sum_{i=n+1}^{2n} \frac{(\bConst{A}^2 \vee 1) \sqrt{i}}{\sqrt{n}\alpha_{2n}} \exp\bigl\{-\frac{a}{4}\sum_{j=1}^{2n}\alpha_{j}\bigr\}\norm{\theta_0 - \thetas} \\
&\lesssim \frac{\supconsteps \bConst{A}^2}{\sqrt{n}\alpha_{2n} a^2} + \frac{(\bConst{A}^2 \vee 1) n}{\alpha_{2n}} \exp\bigl\{-\frac{a}{4}\sum_{j=1}^{2n}\alpha_{j}\bigr\}\norm{\theta_0 - \thetas} \\
& \lesssim \frac{\supconsteps \bConst{A}^2}{a^2} + (\bConst{A}^2 \vee 1) n^{3/2} \exp\bigl\{-\frac{a}{4}\sum_{j=1}^{2n}\alpha_{j}\bigr\}\norm{\theta_0 - \thetas} \\
& \label{eq:D_2_boot_final_diff}
\lesssim \frac{\supconsteps \bConst{A}^2}{a^2} + (\bConst{A}^2 \vee 1) n^{3/4} \exp\bigl\{-\frac{a}{4}\sum_{j=1}^{n}\alpha_{j}\bigr\}\norm{\theta_0 - \thetas}\eqsp.
\end{align}
Here in the last line above we used a particular form $\alpha_{k} = c_{0}/\sqrt{k}$, and relied on the bound
\[
n^{3/4}\exp\bigl\{-\frac{a}{4}\sum_{j=n+1}^{2n}\alpha_{j}\bigr\} \leq 1\eqsp,
\]
which is guaranteed by the lower bound on the trajectory length $n$ of the form
\[
\frac{\sqrt{n}}{\log{n}} \geq \frac{3}{2(\sqrt{2}-1) a c_{0}}\eqsp.
\]
The latter condition is guaranteed by \Cref{assum:step-size-bootstrap}. Now we proceed with $D_3^{\boot} - D_3^{(\boot, i)}$. Using its definition in \eqref{eq:error_decomposition_boot2}, we get
\begin{align}
D_3^{\boot} - D_3^{(\boot, i)} = \frac{1}{\sqrt{n}}(w_i - w_i^{\prime}) \funcAw_i (\theta_{i-1}^\boot - \thetas) + \frac{1}{\sqrt{n}}\sum_{k=i+1}^{2n}(w_k-1) \funcAw_k (\theta_{k-1}^\boot - \theta_{k-1}^{(\boot,i)})\eqsp. 
\end{align}
Since the latter term is a martingale-difference, we obtain that 
\begin{align}
\PEb[\norm{D_3^{\boot} - D_3^{(\boot, i)}}^2] 
&\lesssim \frac{\bConst{A}^2}{n}\PEb[\norm{\theta_{i-1}^\boot - \thetas}^2] + \frac{\bConst{A}^2}{n}\sum_{k=i+1}^{2n}\PEb[\norm{\theta_{k-1}^\boot - \theta_{k-1}^{(\boot,i)}}^2] \\
&\lesssim \frac{\bConst{A}^2}{n}\PEb[\norm{\theta_{i-1}^\boot - \theta_{i-1}}^2] + \frac{\bConst{A}^2}{n}\norm{\theta_{i-1} - \thetas}^2 + \frac{\bConst{A}^2}{n}\sum_{k=i+1}^{2n}\PEb[\norm{\theta_{k-1}^\boot - \theta_{k-1}^{(\boot,i)}}^2] \\
&\lesssim \frac{\bConst{A}^2}{n}\PEb[\norm{\theta_{i-1}^\boot - \theta_{i-1}}^2] + \frac{\bConst{A}^2}{n}\norm{\theta_{i-1} - \thetas}^2 + \frac{\bConst{A}^2}{n} \sum_{k=i+1}^{2n} \PEb[\norm{\ProdB^\boot_{i+1:k-1}}^2\norm{\theta_{i}^\boot - \theta^{(\boot, i)}_{i}}^2]\eqsp.
\end{align}
Hence we obtain, using \eqref{eq:theta_i_diff_bound boot} together with the definition of $\Omega_{1}$, that 
\begin{equation}
\label{eq:D_3_boot_i_bound}
\begin{split}
&\PEb[\norm{D_3^{\boot} - D_3^{(\boot, i)}}^2] 
\lesssim \frac{\alpha_{i-1}\supconsteps^2\bConst{A}^4}{n a^3} + \frac{\alpha_{i-1} \supconsteps^2 \bConst{A}^2 \log^2{n}}{a n} + \frac{\bConst{A}^2}{n} \exp\bigl\{-a \sum_{\ell=1}^{i}\alpha_\ell\bigr\} \norm{\theta_0 - \thetas}^2 \\
& \qquad + \underbrace{\frac{\bConst{A}^2}{n} \left(\frac{\alpha_{i}^2\supconsteps^2\bConst{A}^4}{a^2} + (\bConst{A}^{4} \vee 1) i \prod_{j = 1}^{i} (1 - a \alpha_j/4)^2\norm{\theta_0-\thetas}^2\right)\sum_{k=i+1}^{2n}\exp\bigl\{-\frac{a}{2}\sum_{j=i+1}^{k-1}\alpha_{j}\bigr\}}_{T_1}\eqsp.
\end{split}
\end{equation}
Considering the latter term in the sum, we obtain 
\begin{align}
T_1 &\lesssim \frac{\alpha_i^2 \supconsteps^2 \bConst{A}^6}{n \alpha_{2n} a^2} \sum_{k=i+1}^{2n}\alpha_{k} \exp\bigl\{-\frac{a}{2}\sum_{j=i+1}^{k-1}\alpha_{j}\bigr\} + \frac{(\bConst{A}^6\vee 1)i}{n}\sum_{k=i+1}^{2n}\exp\bigl\{-\frac{a}{2}\sum_{j=1}^{k-1}\alpha_{j}\bigr\} \norm{\theta_0-\thetas}^2 \\
&\lesssim \frac{\alpha_{i} \supconsteps^2 \bConst{A}^6}{n a^3} + \frac{(\bConst{A}^6\vee 1) i}{n \alpha_{2n}}\exp\bigl\{-\frac{a}{2}\sum_{j=1}^{i}\alpha_{j}\bigr\} \norm{\theta_0-\thetas}^2\eqsp.
\end{align}
Thus, summing the equations \eqref{eq:D_3_boot_i_bound} for $i$ from $n+1$ to $2n$, we obtain that 
\begin{align}
\sum_{i=n+1}^{2n}\{\PEb[\norm{D_3^{\boot} - D_3^{(\boot, i)}}^2]\}^{1/2} 
&\lesssim \sum_{i=n+1}^{2n}\left(\frac{\sqrt{\alpha_{i-1}}\supconsteps\bConst{A}^2}{\sqrt{n}a^{3/2}} + \frac{\sqrt{\alpha_{i-1}}\supconsteps\bConst{A}\log{n}}{\sqrt{a n}} + \frac{\sqrt{\alpha_i} \supconsteps \bConst{A}^3}{\sqrt{n}a^{3/2}}\right) \\
&\qquad + \frac{\bConst{A}^3 \vee 1}{\sqrt{n}}\sum_{i=n+1}^{2n}\exp\bigl\{-\frac{a}{2} \sum_{\ell=1}^{i}\alpha_\ell\bigr\} \norm{\theta_0 - \thetas} \\
& \label{eq:D_3_boot_final_diff}
\lesssim \frac{(\bConst{A}^3 \vee 1)\supconsteps}{a^{3/2}} n^{1/4}\log{n} + \frac{\bConst{A}^3 \vee 1}{a} \exp\bigl\{-\frac{a}{2} \sum_{\ell=1}^{n}\alpha_\ell\bigr\}\norm{\theta_0 - \thetas}\eqsp.
\end{align}
Using now the definition of $D_{4}^{\boot}$ in \eqref{eq:error_decomposition_boot2} and Minkowski's inequality, we write 
\begin{align}
\{\PEb[\norm{D_4^{\boot} - D_4^{(\boot, i)}}^2]\}^{1/2} 
&\leq \frac{1}{\sqrt{n}} \{\PEb[\norm{\theta_{i}^\boot - \theta^{(\boot, i)}_{i}}^2]\}^{1/2} \sum_{k=i+1}^{2n}\left(\frac{1}{\alpha_k} - \frac{1}{\alpha_{k-1}}\right)\exp\bigl\{-\frac{a}{4}\sum_{j=i+1}^{k-1}\alpha_{j}\bigr\} \\
&\lesssim \frac{\alpha_{i}\supconsteps\bConst{A}^2}{\sqrt{n} a} \sum_{k=i+1}^{2n}\alpha_{k}\exp\bigl\{-\frac{a}{4}\sum_{j=i+1}^{k-1}\alpha_{j}\bigr\} \\
&\qquad + \frac{(\bConst{A}^{2} \vee 1)\sqrt{i}}{\sqrt{n}} \sum_{k=i+1}^{2n} \alpha_{k} \exp\bigl\{-\frac{a}{4}\sum_{j=1}^{k-1}\alpha_{j}\bigr\}\norm{\theta_0-\thetas} \\
&\lesssim \frac{\alpha_{i}\supconsteps\bConst{A}^2}{\sqrt{n} a^2} + \frac{(\bConst{A}^{2} \vee 1)\sqrt{i}}{\sqrt{n} a} \exp\bigl\{-\frac{a}{4}\sum_{j=1}^{i}\alpha_{j}\bigr\}\norm{\theta_0-\thetas}\eqsp.
\end{align}
Thus, taking sum for $i$ from $n+1$ to $2n$, we get 
\begin{align}
\sum_{i=n+1}^{2n}\{\PEb[\norm{D_4^{\boot} - D_4^{(\boot, i)}}^2]\}^{1/2} 
&\lesssim \sum_{i=n+1}^{2n} \frac{\alpha_{i}\supconsteps\bConst{A}^2}{\sqrt{n} a^2} + \sum_{i=n+1}^{2n}\frac{(\bConst{A}^{2} \vee 1)\sqrt{i}}{\sqrt{n} a} \exp\bigl\{-\frac{a}{4}\sum_{j=1}^{i}\alpha_{j}\bigr\}\norm{\theta_0-\thetas} \\
& \label{eq:D_4_boot_final_diff}
\lesssim \frac{\supconsteps\bConst{A}^2}{a^2} + \frac{(\bConst{A}^{2} \vee 1)\sqrt{n}}{a^2}\exp\bigl\{-\frac{a}{4}\sum_{j=1}^{n}\alpha_{i}\bigr\} \norm{\theta_0-\thetas}\eqsp.
\end{align}
Similarly, with the definition of $D_{5}^{\boot}$ in \eqref{eq:error_decomposition_boot2}, we write
\begin{align}
D_5^{\boot} - D_5^{(\boot, i)} 
&= \frac{1}{\sqrt{n}}\sum_{k=i+1}^{2n} (\funcAw_k - \bA) (\theta_{k-1}^\boot - \theta_{k-1}^{(\boot,i)}) = \frac{1}{\sqrt{n}} \left\{\sum_{k=i+1}^{2n} (\funcAw_k - \bA) \ProdB^\boot_{i+1:k-1}\right\} (\theta_{i}^\boot - \theta^{(\boot, i)}_{i}) \\
&= \underbrace{\frac{1}{\sqrt{n}} \left\{\sum_{k=i+1}^{2n} (\funcAw_k - \bA)\ProdB_{i+1:k-1}\right\} (\theta_{i}^\boot - \theta^{(\boot, i)}_{i})}_{T_2} \\
&\qquad\qquad + \underbrace{\frac{1}{\sqrt{n}} \left\{\sum_{k=i+1}^{2n} (\funcAw_k - \bA)(\ProdB^\boot_{i+1:k-1} - \ProdB_{i+1:k-1})\right\} (\theta_{i}^\boot - \theta^{(\boot, i)}_{i})}_{T_3}\eqsp.
\end{align}
Now we bound the terms $T_2$ and $T_3$ separately. Indeed, for the term $T_2$ we get, applying the definition of the set $\Omega_4$, that
\begin{align}
\PEb[\norm{T_2}^2] 
&\lesssim \frac{1}{n}\left(\frac{\bConst{A}^2\log{n}}{a \alpha_{i}} + \bConst{A}^2\log^{2}{n}\right) \PEb[\norm{\theta_{i}^\boot - \theta^{(\boot, i)}_{i}}^2] \\
&\lesssim \frac{\bConst{A}^2\log{n}}{n a \alpha_{i}} \PEb[\norm{\theta_{i}^\boot - \theta^{(\boot, i)}_{i}}^2]\eqsp.
\end{align}
In the above bounds we have used that $\alpha_{\ell} \leq \frac{1}{a \log{n}}$. For the term $T_3$ we get, applying \Cref{lem:product_coupling_lemma}, that for any vector $v \in \rset^{d}$, 
\begin{align}
\sum_{k=i+1}^{2n} (\funcAw_k - \bA)(\ProdB^\boot_{i+1:k-1} - \ProdB_{i+1:k-1}) v 
&= \sum_{k=i+1}^{2n} \sum_{\ell=i+1}^{k-1} (\funcAw_k - \bA) \ProdB_{\ell+1:k-1} \alpha_{\ell}(w_{\ell}-1)\funcAw_{\ell} \ProdB^\boot_{i+1:\ell-1}\eqsp v \\
&= \sum_{\ell=i+1}^{2n-1} \alpha_{\ell}(w_{\ell}-1) \biggl\{\sum_{k=\ell+1}^{2n}  (\funcAw_k - \bA) \ProdB_{\ell+1:k-1} \biggr\} \funcAw_{\ell} \ProdB^\boot_{i+1:\ell-1}\eqsp v\eqsp.
\end{align}
From the above representation we obtain, using the definition of the set $\Omega_4$, that 
\begin{align}
\PEb[\norm{T_3}^2] 
&\lesssim \frac{\bConst{A}^2}{n}\sum_{\ell=i+1}^{2n-1}\alpha_{\ell}^2\left(\frac{\bConst{A}^2\log{n}}{a \alpha_{\ell}} + \bConst{A}^2\log^{2}{n}\right)\exp\biggl\{-\frac{a}{4}\sum_{j=i+1}^{\ell-1}\alpha_{j}\biggr\} \PEb[\norm{\theta_{i}^\boot - \theta^{(\boot, i)}_{i}}^2] \\
&\lesssim \frac{\bConst{A}^4}{n} \sum_{\ell=i+1}^{2n-1}\frac{\alpha_{\ell}\log{n}}{a} \exp\biggl\{-\frac{a}{4}\sum_{j=i+1}^{\ell-1}\alpha_{j}\biggr\} \PEb[\norm{\theta_{i}^\boot - \theta^{(\boot, i)}_{i}}^2] \\
&\lesssim \frac{\bConst{A}^4 \log{n}}{n a^2}\eqsp \PEb[\norm{\theta_{i}^\boot - \theta^{(\boot, i)}_{i}}^2]\eqsp. 
\end{align}
Combining the above bounds, we obtain that 
\begin{align}
\PEb[\norm{D_5^{\boot} - D_5^{(\boot, i)}}^2] &\lesssim 
\frac{\bConst{A}^2\log{n}}{na}\left(\frac{1}{\alpha_{i}} + \frac{\bConst{A}^2}{a}\right)\PEb[\norm{\theta_{i}^\boot - \theta^{(\boot, i)}_{i}}^2] \\
&\lesssim 
\frac{\bConst{A}^2\log{n}}{na \alpha_{i}} \PEb[\norm{\theta_{i}^\boot - \theta^{(\boot, i)}_{i}}^2]\eqsp,
\end{align}
where we have additionally used that $\alpha_{i} \leq a/\bConst{A}^2$. Thus, using the upper bound \eqref{eq:theta_i_diff_bound boot}, we obtain that
\begin{align}
&\sum_{i=n+1}^{2n}\{\PEb[\norm{D_5^{\boot} - D_5^{(\boot, i)}}^2]\}^{1/2} \\ 
&\qquad \lesssim \sum_{i=n+1}^{2n}\frac{\bConst{A}\sqrt{\log{n}}}{\sqrt{\alpha_{i} a n}} \frac{\alpha_{i}\supconsteps\bConst{A}^2}{a} + \sum_{i=n+1}^{2n} \frac{\bConst{A}\sqrt{\log{n}}}{\sqrt{\alpha_{i} a n}}(\bConst{A}^{2} \vee 1) \sqrt{i} \prod_{j = 1}^{i} (1 - a \alpha_j/4)\norm{\theta_0-\thetas} \\
&\qquad \lesssim \frac{\bConst{A}^3 \supconsteps}{a^{3/2}}n^{1/4}\sqrt{\log{n}} \\
&\qquad \qquad \qquad \qquad + \frac{(\bConst{A}^{3} \vee 1)\sqrt{\log{n}}}{a^{1/2}\alpha_{2n}^{3/2}} \exp\bigl\{-\frac{a}{4}\sum_{j=1}^{n}\alpha_{j}\bigr\}\biggl\{\sum_{i=n+1}^{2n}\alpha_{i} \prod_{j = n+1}^{i}(1 - a \alpha_j/4)\biggr\}\norm{\theta_0-\thetas} \\
& \label{eq:D_5_boot_final_diff}
\qquad \lesssim \frac{\bConst{A}^3 \supconsteps}{a^{3/2}}n^{1/4}\sqrt{\log{n}} + \frac{(\bConst{A}^{3} \vee 1) n^{3/4} \sqrt{\log{n}}}{a^{3/2}} \exp\bigl\{-\frac{a}{4}\sum_{j=1}^{n}\alpha_{j}\bigr\} \norm{\theta_0-\thetas}\eqsp. 
\end{align}
Now it remains to combine the bounds outlined above in \eqref{eq:D_2_boot_final_diff}, \eqref{eq:D_3_boot_final_diff}, \eqref{eq:D_4_boot_final_diff}, and \eqref{eq:D_5_boot_final_diff}, and the statement follows.

\section{Proof of stability of random matrix product}
\label{appendix:tehnical}
\subsection{Proof of \Cref{prop:hurwitz_stability}}
\label{proof:hurwitz_stability}
The fact that there exists a unique matrix $Q$, such that the following Lyapunov equation holds:
\begin{equation}
\label{eq:Lyapunov_equation}
\bA^\top Q + Q \bA = P\eqsp,
\end{equation}
follows directly from \cite[Lemma~$9.1$, p. 140]{poznyak:control}. In order to show the second part of the statement, we note that for any non-zero vector $x \in \rset^{d}$, we have
\begin{align}
\frac{x^{\top}(\Id - \alpha \bA)^{\top}Q(\Id - \alpha \bA)x}{x^{\top} Q x} 
&= 1 - \alpha \frac{x^{\top}(\bA^{\top}Q + Q\bA)x}{x^{\top} Q x} + \alpha^2 \frac{x^{\top} \bA^{\top} Q \bA x}{x^{\top} Q x} \\
&= 1 - \alpha \frac{x^{\top} P x}{x^{\top} Q x} + \alpha^2\, \frac{x^{\top} \bA^{\top} Q \bA x}{x^{\top} Q x} \\
&\leq 1 - \alpha \frac{\lambda_{\min}(P)}{\normop{Q}} + \alpha^2\, \frac{\norm{\bA}[Q]^2}{\lambda_{\min}(Q)} \\
&\leq 1 - \alpha a\eqsp,
\end{align}
where we set 
\[
a = \frac{1}{2} \frac{\lambda_{\min}(P)}{\lambda_{\max}(Q)}\eqsp,
\]
and used the fact that $\alpha \leq \alpha_{\infty}$, where $\alpha_{\infty}$ is defined in \eqref{eq:alpha_infty_def}.

\subsection{Proofs for auxiliary results on products of random matrix}
In order to bound the moment $\PE[\norm{\theta_k - \thetas}^p]$, we first prove a stability results on the products of random matrices $\ProdB_{m:k}$ arising in the LSA recursion. Towards this aim we first introduce some notations and definitions. For a matrix $\MatB \in \rset^{d \times d}$ we denote by $(\sigma_\ell(\MatB))_{\ell=1}^d$ its singular values. For $\qexponent \geq 1$, the Shatten $\qexponent$-norm of $B$ is denoted by $\norm{\MatB}[\qexponent] = \{\sum_{\ell=1}^d \sigma_\ell^\qexponent (\MatB)\}^{1/\qexponent}$. For $\qexponent, \ppexponent \geq 1$ and a random matrix $\X$ we write $\norm{\X}[\qexponent,\ppexponent] = \{ \PE[\norm{\X}[\qexponent]^\ppexponent] \}^{1/\ppexponent}$. Our proof technique is based on the stability results arising in \cite{huang2020matrix}, see also \cite{durmus2022finite}.

\begin{lemma}[Proposition~15 in \cite{durmus2022finite}]
\label{th:general_expectation}
Let $\sequence{\Y}[\ell][\nset]$ be an independent sequence and $P$ be a positive definite matrix. Assume that for each $\ell \in \nset$ there exist $m_\ell \in (0,1)$  and $\sigma_{\ell} > 0$ such that \(\norm{\PE[\Y_\ell]}[P]^2  \leq 1 - m_\ell\) and \(\norm{\Y_\ell - \PE[\Y_\ell]}[P] \leq \sigma_{\ell}\) almost surely.  Define $\Zbf_k = \prod_{\ell = 0}^k \Y_\ell= \Y_k \Zbf_{k-1}$, for $k \geq 1$ and starting from $\Zbf_0$. Then, for any $2 \le q \leq p$ and $k \geq 1$,
\begin{equation} 
\label{eq:gen_expectation}
\norm{\Zbf_k}[p,q]^2 \leq \kappa_P \prod_{\ell=1}^{k} (1- m_\ell + (p-1)\sigma_{\ell}^2) \norm{P^{1/2}\Zbf_0 P^{-1/2}}[p, q]^2 \eqsp,
\end{equation}
where we recall that $\kappa_P = \lambda_{\sf min}^{-1}( P )\lambda_{\sf max}( P )$.
\end{lemma}
Now we aim to bound $\ProdB_{m:k}$ defined in \eqref{eq:prod_rand_matr} using \Cref{th:general_expectation}. We identify the latter with
$\prod_{\ell=m}^{k}\Y_\ell$, where $\Y_\ell = \Id - \alpha_\ell \Am_\ell, \ell \geq 1$, and $\Y_0 = \Id$. Applying the bound \eqref{eq:contractin_q_norm}, we get $\norm{\PE[\Y_\ell]}[Q]^2 = \norm{\Id - \alpha_{\ell} \bA }[Q]^2 \leq 1 - a \alpha_{\ell}$. Further, assumption \Cref{assum:noise-level} implies that almost surely,
\[
\norm{\Y_\ell - \PE[\Y_\ell]}[Q] =  \alpha_\ell \norm{ \Am_\ell- \bA}[Q] \leq   \alpha_\ell \sqrt{\qcond} \bConst{A}  = b_{Q} \alpha_\ell \eqsp.
\]
Therefore, \eqref{eq:gen_expectation} holds with $m_\ell = a \alpha_\ell$ and  $\sigma_{\ell} = b_Q \alpha_\ell$. As $\norm{\Id}[p] = d^{1/p}$, we obtain the following corollary.

\begin{corollary}
\label{cor:norm_Gamma_m_n}
Assume \Cref{assum:iid} and \Cref{assum:noise-level}. Then, for any $\alpha_\ell \in [0, \alpha_{\infty}]$, $2 \le q \le p$, and $1 \leq m \leq k$, it holds
\begin{equation}
\label{eq:concentration iid}
\PE^{1/q}\left[ \normop{\ProdB_{m:k}}^{q} \right]  
\leq  \norm{\ProdB_{m:k}}[p,q] 
\leq \sqrt{\qcond} d^{1/p} \prod_{\ell=m}^{k}(1 - a \alpha_\ell + (p-1) b_Q^2 \alpha_\ell^2) \eqsp,
\end{equation}
where $\alpha_\infty$ was defined in \eqref{eq:alpha_infty_def}, and $b_{Q} =  \sqrt{\qcond} \bConst{A}$.
\end{corollary}
\begin{corollary}
\label{cor:exp_bound_decay}
Assume \Cref{assum:iid}, \Cref{assum:noise-level}, and \Cref{assum:step-size}. Then for any $2 \leq q \leq \log{n}$, and any $k \geq n$, $1 \leq m \leq k$, it holds that 
\begin{equation}
\label{eq:concentration_iid_cor}
\PE^{1/q}\left[ \normop{\ProdB_{m:k}}^{q} \right]  
\leq  \sqrt{\qcond} \rme \exp\left\{-(a/2) \sum_{\ell=m}^{k}\alpha_\ell\right\} \eqsp,
\end{equation}
where $\alpha_\infty$ is defined in \eqref{eq:alpha_infty_def}. Moreover, 
\begin{equation}
\label{eq:concentration_iid_product}
\PE^{1/q}\left[ \normop{\ProdB_{m:k}}^{q} \right]  
\leq \sqrt{\qcond} \rme \prod_{\ell=m}^{k}\bigl(1 - \frac{a \alpha_{\ell}}{4}\bigr)
\end{equation}
\end{corollary}
\begin{proof}
We first apply the result of \Cref{cor:norm_Gamma_m_n}. Indeed, for $k \geq n$, and any $2 \leq q \leq p$, it holds, setting $b_{Q} =  \sqrt{\qcond} \bConst{A}$, that 
\begin{align}
\PE^{1/q}\left[ \normop{\ProdB_{m:k}}^{q} \right] 
&\leq \sqrt{\qcond} d^{1/p} \prod_{\ell=m}^{k} (1 - a \alpha_\ell + (p-1) b_Q^2 \alpha_\ell^2) \\
&\leq \sqrt{\qcond} d^{1/p} \exp\left\{-a \sum_{\ell=m}^{k}\alpha_\ell + (p-1) b_Q^2 \sum_{\ell=m}^{k} \alpha_\ell^2  \right\}\eqsp.
\end{align}
Note that, setting $p = \log{n}$, and provided that $n$ satisfies \eqref{eq:sample_size_bound}, we easily obtain that, for $\ell \geq n/2$, 
\begin{equation}
\label{eq:step_size_linear_quadraic_relation}
(\log{n}) b_Q^2 \alpha_\ell^2 \leq a \alpha_\ell / 2\eqsp.
\end{equation}
Hence, for $m \geq n/2$, we have  
\begin{equation}
\label{eq:bound_large_m}
\PE^{1/q}\left[ \normop{\ProdB_{m:k}}^{q} \right] \leq \sqrt{\qcond} \rme \exp\bigl\{-\frac{a}{2}\sum_{\ell=m}^{k}\alpha_{\ell}\bigr\}\eqsp,
\end{equation}
and the statement follows. Suppose now that $m < n/2$. In such a case we have, applying \eqref{eq:step_size_linear_quadraic_relation}, that 
\begin{align}
\PE^{1/q}\left[ \normop{\ProdB_{m:k}}^{q} \right] 
&\leq \sqrt{\qcond} \rme \exp\left\{-a \sum_{\ell=m}^{k}\alpha_\ell + (\log{n}) b_Q^2 \sum_{\ell=m}^{k} \alpha_\ell^2  \right\} \\
&\leq \sqrt{\qcond} \rme \exp\left\{-a \sum_{\ell=m}^{n}\alpha_\ell + (\log{n}) b_Q^2 \sum_{\ell=m}^{n} \alpha_\ell^2  \right\} \exp\left\{-(a/2)\sum_{\ell=n+1}^{k} \alpha_{\ell}\right\} \label{eq:gamma_m_k_bound_sqrt_n}\eqsp,
\end{align}
and we need to bound the first term in the product. We first consider $\alpha_{\ell} = c_{0} \ell^{-1/2}$, and use the inequalities 
\begin{equation}
\label{eq:upper_bound_harmonic}
\sum_{\ell=m}^{n}\frac{1}{\ell} \leq \left( 1 + \int_{m}^{n}\frac{dx}{x} \right) \wedge \left(\int_{m-1}^{n}\frac{dx}{x}\right) = \left(1 + \log{\frac{n}{m}}\right) \wedge \left(\log{\frac{n}{m-1}}\right)\eqsp,
\end{equation}
and 
\begin{equation}
\label{eq:lower_bound_sqrt}
\sum_{\ell=m}^{n}\frac{1}{\sqrt{\ell}} \geq \int_{m}^{n}\frac{dx}{\sqrt{x}} = 2(\sqrt{n} - \sqrt{m})\eqsp.
\end{equation}
Thus, it is enough to satisfy the constraint
\[
(\log{n}) b_{Q}^{2} c_{0}^2 (1 + \log{n} - \log{m}) \leq a c_{0} (\sqrt{n} - \sqrt{m})\eqsp.
\]
Since $m < n/2$, it is enough to ensure that   
\[
(1 + \log{n})(\log{n}) b_{Q}^{2} c_{0}^2 \leq a c_{0} (\sqrt{n} - \sqrt{n/2})\eqsp,
\]
or, equivalently,
\[
\frac{\sqrt{n}}{(1+\log{n})\log{n}} \geq \frac{c_{0} b_{Q}^{2}}{a(1-1/\sqrt{2})}\eqsp,
\]
which is granted by \Cref{assum:step-size}. Combining the above bounds in \eqref{eq:gamma_m_k_bound_sqrt_n}, we obtain that the lemma's statement \eqref{eq:concentration_iid_cor}  holds for the step size $\alpha_{\ell} = c_0 / \ell^{1/2}$. Similarly, for $\alpha_{\ell} = c_0 / \ell^{\gamma}$ with $\gamma \in (1/2;1)$, we get for $m \geq n/2$ that 
\[
\PE^{1/q}\left[ \normop{\ProdB_{m:k}}^{q} \right] \leq \sqrt{\qcond} \rme \exp\bigl\{-\frac{a}{2}\sum_{\ell=m}^{k}\alpha_{\ell}\bigr\}\eqsp,
\]
since the relation \eqref{eq:step_size_linear_quadraic_relation} holds. Similarly, for $m < n/2$, the desired upper bound would follow from the inequality
\[
\sum_{\ell=m}^{n}\frac{1}{\ell^{2\gamma}} \leq \int_{m-1}^{n}\frac{dx}{x^{2\gamma}} = \frac{(m-1)^{1-2\gamma} - n^{1-2\gamma}}{2\gamma - 1} \leq \frac{1}{2\gamma - 1}\eqsp,
\]
together with an inequality 
\[
\frac{(\log{n}) b_{Q}^{2} c_{0}^2}{2\gamma - 1} \leq  (a/2) c_{0} (n^{1-\gamma} - (n/2)^{1-\gamma})\eqsp.
\]
The latter inequality can be re-written as 
\[
\frac{n^{1-\gamma}}{\log{n}} \geq \frac{2c_{0}b_{Q}^{2}}{a(2\gamma-1)(1-(1/2)^{1-\gamma}}\eqsp,
\]
which is also granted by \Cref{assum:step-size}. Combining the above inequalities implies that \eqref{eq:concentration_iid_cor} holds for $\alpha_{\ell} = c_0 / \ell^{\gamma}$. The bound \eqref{eq:concentration_iid_product} can be immediately obtained from \eqref{eq:concentration_iid_cor} using the fact that $\rme^{-x} \leq 1 - x/2$ for $x \in [0;1]$.
\end{proof}

\begin{corollary}
\label{cor:matr_product_as_bound}
Under conditions of \Cref{cor:exp_bound_decay} it holds with $\PP$ -- probability at least $1 - 1/n^2$ that
$$
\normop{\ProdB_{m:k}} 
\leq \sqrt{\qcond} \rme^2 \exp\left\{-(a/2) \sum_{\ell=m}^{k}\alpha_\ell\right\} \eqsp,
$$
and
$$
 \normop{\ProdB_{m:k}}
\leq \sqrt{\qcond} \rme^2 \prod_{\ell=m}^{k}\bigl(1 - \frac{a \alpha_{\ell}}{4}\bigr)
$$
\end{corollary}
\begin{proof}
It is sufficient to choose $q = 2\log n$ and use Markov's inequality together with the union bound. 
\end{proof}

\begin{proposition}
\label{prop:product_random_matrix_bootstrap}
Assume \Cref{assum:iid}, \Cref{assum:noise-level}, \Cref{assum:step-size} with $\gamma = 1/2$, and \Cref{assum:step-size-bootstrap}. Then on the set $\Omega_5$ defined in \eqref{eq:omegas_definition}, it holds for any $n \leq m \leq k \leq 2n$, that 
\begin{equation}
\label{eq:stability_bound_matrix_products_bootstrap}
\bigl\{\PEb[\norm{\ProdB^\boot_{m+1:k}}^2] \bigr\}^{1/2} \leq \qcond^{3/2} \rme^{9/8} \exp\left\{-\frac{a}{4}\sum_{\ell=m+1}^{k}\alpha_{\ell} \right\}\eqsp.
\end{equation}
\end{proposition}
\begin{proof}
Our proof relies on the auxiliary result of \Cref{lem:product_ramdom_matrix_aux} below together with the blocking technique. Indeed, let us represent 
\[
k - m = N h + r\eqsp,
\]
where $r < h$ and $h = h(n)$ is a block size defined in \eqref{eq:block_size_constraint}. Then we obtain, using the independence of bootstrap weights $w_{m+1},\ldots,w_{k}$, that 
\begin{align}
\{\PEb[\norm{\ProdB^\boot_{m+1:k}}^2]\}^{1/2} 
&\leq \sqrt{\qcond} \{\PEb[\norm{\ProdB^\boot_{m+1:k}}[Q]^2]\}^{1/2} \\
&= \sqrt{\qcond} \prod_{j=1}^{N}\bigl\{\PEb[\norm{\ProdB^\boot_{m+1+(j-1)h:m+jh}}[Q]^2]\bigr\}^{1/2} \bigl\{\PEb[\norm{\ProdB^\boot_{m+1+Nh:k}}[Q]^2]\bigr\}^{1/2} \\
&\leq \sqrt{\qcond} \exp\left\{-\frac{a}{4}\sum_{\ell=m+1}^{k}\alpha_{\ell} \right\} \bigl\{\PEb[\norm{\ProdB^\boot_{m+1+Nh:k}}[Q]^2]\bigr\}^{1/2} \exp\left\{\frac{a}{4}\sum_{\ell=m+1+Nh:k}^{k}\alpha_{\ell} \right\}\eqsp.
\end{align}
In the last inequality we applied \Cref{lem:product_ramdom_matrix_aux} to each of the blocks of length $h$ in the first bound. It remains to upper bound the residual terms. Since the remainder block has length less then $h$, we have due to \eqref{eq:sum_steps_bound_stability} (which holds according to \Cref{assum:step-size-bootstrap}), that
\[
\exp\left\{\frac{a}{4}\sum_{\ell=m+1+Nh:k}^{k}\alpha_{\ell} \right\} \leq \exp\left\{\frac{\alpha_{\infty} a}{4}\right\} \leq \rme^{1/8}\eqsp,
\]
where the last inequality is due to \Cref{prop:hurwitz_stability}. Next,
\begin{align}
\bigl\{\PEb[\norm{\ProdB^\boot_{m+1+Nh:k}}[Q]^2]\bigr\}^{1/2} 
&\leq \qcond \prod_{\ell=m+1+Nh:k}^{k}\{\PEb[\norm{(\Id - \alpha_{\ell} w_{\ell} \Am_{\ell})}^2]\}^{1/2} \\
&\leq \qcond \prod_{\ell=m+1+Nh:k}^{k} \{\PEb[(1+\alpha_{\ell} |w_{\ell}| \bConst{A})^2]\}^{1/2} \\
&\leq \qcond \prod_{\ell=m+1+Nh:k}^{k} \{\PEb[1 + 2\alpha_{\ell}|w_{\ell}| \bConst{A} + \alpha_{\ell}^2 w_{\ell}^2\bConst{A}^2]\}^{1/2}\eqsp.
\end{align}
Since 
\[
\PE[|w_{\ell}|] \leq \sqrt{\PE[w_{\ell^2}]} \leq \sqrt{(\PE[w_{\ell}])^2 + \var{w_{\ell}}} = \sqrt{2}\eqsp,
\]
we get from previous bound 
\begin{align}
\bigl\{\PEb[\norm{\ProdB^\boot_{m+1+Nh:k}}[Q]^2]\bigr\}^{1/2} 
&\leq \qcond \prod_{\ell=m+1+Nh:k}^{k} (1 + 2\sqrt{2}\alpha_{\ell} \bConst{A} + 2\alpha_{\ell}^2 \bConst{A}^2)^{1/2} \\
&\leq \qcond \exp\left\{\sqrt{2} \bConst{A} \sum_{\ell=m+1+Nh:k}^{k}\alpha_{\ell}\right\} \leq \qcond \rme^{\sqrt{2}\bConst{A} c_{0} h / \sqrt{n}} \leq \qcond \rme\eqsp,
\end{align}
where in the last line we additionally used \eqref{eq:sample_size_bound_part_2}.
\end{proof}

\begin{lemma}
\label{lem:product_ramdom_matrix_aux}
Assume \Cref{assum:iid}, \Cref{assum:noise-level}, \Cref{assum:step-size} with $\gamma = 1/2$, and \Cref{assum:step-size-bootstrap}. On the set $\Omega_5$ defined in \eqref{eq:omegas_definition}, it holds for $h = h(n)$ defined in \eqref{eq:block_size_constraint} and any $m \in [n;2n-h]$, that 
$$
\bigl\{\PEb[\norm{\ProdB^\boot_{m+1:m+h}}[Q]^2] \bigr\}^{1/2} \leq \exp\left\{-\frac{a}{4}\sum_{\ell=m+1}^{m+h}\alpha_{\ell} \right\} \eqsp.
$$
\end{lemma}
\begin{proof}
Recall that we use the notation $\PEb[\cdot] = \PE[\cdot | \mathcal{Z}^{2n}]$, where $\mathcal{Z}^{2n} = (\State_1,\ldots,\State_{2n})$ are the random variables used in the construction of the iterates $\{\theta_k\}_{1 \leq k \leq n}$ in \eqref{eq:lsa}. 

Let $h \in \nset$ be a block length, which value will be determined later, and consider a product 
\begin{equation}
\label{eq:decomp_Gamma_proof_main}
\ProdB^\boot_{m+1:m+h} = \prod_{\ell=m+1}^{m+h} (\Id - \alpha_{\ell} w_{\ell} \Am_{\ell})\eqsp.
\end{equation}
Expanding the product of matrices \eqref{eq:decomp_Gamma_proof_main}, we obtain 
\begin{equation} 
\label{eq:split_main}
\ProdB^\boot_{m:m+h} = \Id - \sum_{\ell = m+1}^{m+h} \alpha_\ell \Am_\ell   - \Mat{S}  + \Mat{R} = \Id - \sum_{\ell = m+1}^{m+h} \alpha_\ell \bA - \sum_{\ell = m+1}^{m+h} \alpha_\ell (\Am_\ell - \bA) - \Mat{S} + \Mat{R}\eqsp,
\end{equation}
where $\Mat{S} =  \sum_{\ell = m+1}^{m+h} \alpha_\ell (w_\ell - 1) \Am_\ell$ is a linear statistics in $\{w_{\ell}\}_{\ell=m+1}^{m+h}$, and the remainder $\Mat{R}$ collects the higher-order terms in the products
\begin{equation} 
\label{eq:RlRlbar_def}
\Mat{R} = \sum_{r=2}^{h}(-1)^{r}  \sum_{(i_1,\dots,i_r)\in\msi_r^\ell}\prod_{u=1}^{r} \alpha_{i_u} w_{i_u} \Am_{i_u}\eqsp.
\end{equation}
with $\msi_r^{\ell} = \{(i_1,\ldots,i_r) \in \{m+1,\ldots,m+h\}^r\, : \, i_1 < \cdots < i_r \}$. We first consider the contracting part in matrix $Q$-norm. Indeed, applying \eqref{eq:contractin_q_norm}, we obtain that 
\[
\norm{\Id - \sum_{\ell = m+1}^{m+h} \alpha_\ell \bA}[Q]^2 \leq 1 - a \sum_{\ell = m+1}^{m+h} \alpha_\ell\eqsp,
\]
provided that $h$ is set in such a manner that $\sum_{\ell = m+1}^{m+h} \alpha_\ell \leq \alpha_{\infty}$, where $\alpha_{\infty}$ is defined in \eqref{eq:alpha_infty_def}. Hence, we get from the above inequality that for any $u \in \rset^{d}$, it holds that
\[
\norm{\Id - \sum_{\ell = m+1}^{m+h} \alpha_\ell \bA}[Q] \leq 1 - (a/2) \sum_{\ell = m+1}^{m+h} \alpha_\ell\eqsp.
\]
Now we need to estimate the remainders in the representation \eqref{eq:split_main}. 
On the set $\Omega_5$, it holds that
\begin{equation}
\label{eq:lem_hoeffding_rd}
\norm{\sum_{\ell = m+1}^{m+h} \alpha_\ell (\Am_\ell - \bA)}[Q] \leq 2\bConst{A} \sqrt{\qcond} \sqrt{\sum_{\ell=m+1}^{m+h}\alpha_\ell^{2}} \log(2 n^4)\eqsp.
\end{equation}
Moreover, it is straightforward to check that
\[
\PEb[\norm{ \Mat{S}}[Q]^2] \leq \bConst{A}^2 \qcond \sum_{\ell = m+1}^{m+h} \alpha_\ell^2\eqsp.
\]
In order to bound the remainder term $\Mat{R}$, we note that 
\begin{align}
\PEb[\norm{\Mat{R}}[Q]] 
&\leq \sum_{r=2}^{h}\binom{h}{r} \alpha_{m+1}^{r} (2\bConst{A})^{r} \qcond^{r/2} \leq \alpha_{m+1}^{2} (2\bConst{A})^{2} \qcond \sum_{r=0}^{h-2} \binom{h}{r+2} \alpha_{m+1}^{r} (2\bConst{A})^{r} \qcond^{r/2} \\
&\leq \frac{\alpha_{m+1}^{2} h^2 (2\bConst{A})^{2} \qcond}{2} \exp\bigl\{2\alpha_{m+1} \bConst{A} \qcond^{1/2}\bigr\} \\
&\leq \frac{\alpha_{m+1}^{2} h^2 (2\bConst{A})^{2} \qcond \rme}{2}\eqsp.
\end{align}
To complete the proof it remains to set the parameter $h$ in such a way that we can guarantee
\begin{equation}
\label{eq:init_step_size_constr}
\bConst{A} \sqrt{\qcond}\sqrt{\sum_{\ell = m+1}^{m+h} \alpha_\ell^2} \biggl(1 + 2\log(2n^4)\biggr) + \frac{\alpha_{m+1}^{2} h^2 \bConst{A}^{2} \qcond \rme}{2} \leq \frac{a}{4}\sum_{\ell=m+1}^{m+h}\alpha_{\ell}\eqsp,
\end{equation}
keeping at the same time the constraint 
\begin{equation}
\label{eq:step_size_contraction_constraint}
\sum_{\ell = m+1}^{m+h} \alpha_\ell \leq \alpha_{\infty}\eqsp.
\end{equation}
Recall that $\alpha_{\ell} = c_{0}/\sqrt{\ell}$. Thus, using the bounds \eqref{eq:upper_bound_harmonic} and \eqref{eq:lower_bound_sqrt}, we obtain that 
\begin{equation}
\label{eq:sum_steps_bound_stability}
\frac{a}{4}\sum_{\ell=m+1}^{m+h}\alpha_{\ell} \geq \frac{a c_{0}}{2}(\sqrt{m+h} - \sqrt{m+1}) \geq \frac{a c_{0}}{2}(\sqrt{m+h} - \sqrt{m})\eqsp,
\end{equation}
and 
\begin{equation}
\label{eq:sum_steps_squared_bound_stability}
\sum_{\ell = m+1}^{m+h} \alpha_\ell^2 = \sum_{\ell = m+1}^{m+h}\frac{c_{0}^2}{\ell} \leq c_{0}^2(\log{(m+h)}-\log{m})\eqsp.
\end{equation}
Hence, taking into account \eqref{eq:sum_steps_bound_stability} and \eqref{eq:sum_steps_squared_bound_stability}, and $\frac{1}{m+1} \leq \frac{1}{m}$, the inequality \eqref{eq:init_step_size_constr} would follow from the bound  
\begin{align}
\label{eq:one_more_bound_appendix}
\bConst{A}\sqrt{\qcond}\sqrt{\log(m+h) - \log(m)} \biggl(1 + 2\log(2n^4)\biggr) + \frac{c_{0}h^2 \bConst{A}^{2} \qcond \rme}{2 m} \leq \frac{a}{2}(\sqrt{m+h}-\sqrt{m})\eqsp.
\end{align}
Since $\log{(1+x)} \leq x$ for $x \geq 0$ and $c_{0} \bConst{A}^2 \qcond \rme \leq 1$, the latter inequality is satisfied if 
\begin{equation}
\label{eq:block_size_h_bound}
\bConst{A}\sqrt{\qcond} \frac{\sqrt{h}}{\sqrt{m}}\biggl(1 + 2\log(2n^4)\biggr) + \frac{h^2}{2 m} \leq \frac{a}{2}(\sqrt{m+h}-\sqrt{m})\eqsp.
\end{equation}
Now we use one more lower bound
\[
\sqrt{m+h}-\sqrt{m} = \sqrt{m}(\sqrt{1+h/m} - 1) \geq \frac{\sqrt{m}(\sqrt{2}-1) h}{m} = \frac{(\sqrt{2}-1)h}{\sqrt{m}}\eqsp,
\]
which follows from an elementary inequality $\sqrt{1+x} \geq 1+(\sqrt{2}-1)x$, valid for $0 \leq x \leq 1$. Hence, \eqref{eq:one_more_bound_appendix} would from the inequality 
\begin{align}
\label{eq:one_more_bound_appendix_new}
\bConst{A}\sqrt{\qcond} \frac{\sqrt{h}}{\sqrt{m}}\biggl(1 + 2\log(2n^4)\biggr) + \frac{h^2}{2 m} \leq \frac{a(\sqrt{2}-1)h}{2\sqrt{m}}\eqsp.
\end{align}
Setting $h$ is such a manner that 
\begin{equation}
\label{eq:block_size_appendix_constraint_1}
\frac{h}{\sqrt{m}} \leq \frac{a(\sqrt{2}-1)}{2}\eqsp,
\end{equation}
inequality \eqref{eq:one_more_bound_appendix_new} would follow from 
\begin{equation}
\label{eq:one_more_bound_appendix_final}
\bConst{A}\sqrt{\qcond} \frac{\sqrt{h}}{\sqrt{m}}\biggl(1 + 2\log(2n^4)\biggr) \leq \frac{a(\sqrt{2}-1)h}{4\sqrt{m}}\eqsp.
\end{equation}
The latter inequality is satisfied, if the block size $h$ satisfies
\begin{equation}
\label{eq:block_size_constraint_1}
h \geq \biggl(\frac{4\bConst{A}\qcond^{1/2}}{(\sqrt{2}-1)a}\biggr)^{2}(1+2\log{(2n^4)})^2\eqsp.
\end{equation}
Thus, setting $h(n)$ as in \eqref{eq:block_size_constraint}, all previous inequalities will be fulfilled, provided that
\begin{equation}
\label{eq:lock_const_appendix_fin}
\begin{cases}
\frac{h(n)}{\sqrt{n}} &\leq \frac{a(\sqrt{2}-1)}{2} \\
\frac{c_{0} h(n)}{\sqrt{n}} &\leq \alpha_{\infty}\eqsp.
\end{cases}
\end{equation}
Here last inequality follows from \eqref{eq:step_size_contraction_constraint} and the following simple bounds, where we use that $m \geq n$ and $\sqrt{1+x} \leq 1+x/2$:
\begin{align}
\sum_{\ell = m+1}^{m+h} \alpha_\ell \leq \sum_{\ell = n+1}^{n+h} \alpha_\ell = c_{0} \sum_{\ell = n+1}^{n+h} \frac{1}{\sqrt{\ell}} \leq c_{0} \int_{n}^{n+h}\frac{dx}{\sqrt{x}} = 2 c_{0} (\sqrt{n+h}-\sqrt{n}) \leq \frac{c_{0} h}{\sqrt{n}}\eqsp.
\end{align}
Now \eqref{eq:split_main} implies that 
\[
\bigl\{\PEb[\norm{\ProdB^\boot_{m+1:m+h}}[Q]^2] \bigr\}^{1/2} \leq 1 - (a/4) \sum_{\ell = m+1}^{m+h} \alpha_\ell\eqsp,
\]
and the statement follows from an elementary inequality $1 + x \leq \rme^{x}$.
\end{proof}


\section{Applications to the TD learning}
\label{appendix:td_learning}
Recall that the temporal difference learning algorithm in the LSA's setting can be written as
\begin{equation}
\label{eq:LSA_procedure_TD_appendix}
\theta_{k} = \theta_{k-1} - \alpha_{k} (\funcAw_{k} \theta_{k-1} - \funcbw_{k})\eqsp,
\end{equation}
where $\funcAw_{k}$ and $\funcbw_{k}$ are given by
\begin{equation}
\label{eq:matr_A_def_appendix}
\begin{split}
\funcAw_{k} &= \varphi(s_k)\{\varphi(s_k) - \gamma \varphi(s'_k)\}^{\top}\eqsp, \\
\funcbw_{k} &= \varphi(s_k) r(s_k,a_k)\eqsp.
\end{split}
\end{equation}
Recall that our aim is to estimate the agent's \emph{value function}
\[
\textstyle
V^{\pi}(s) = \PE[\sum_{k=0}^{\infty}\gamma^{k}r(s_k,a_k)|s_0 = s]\eqsp,
\]
where $a_{k} \sim \pi(\cdot | s_k)$, and $s_{k+1} \sim \PMDP(\cdot | s_{k}, a_{k})$, for any $k \in \nset$. We define the transition kernel under policy $\pi$
\begin{equation}
\label{eq:transition_matrix_P_pi_appendix}
\textstyle \PMDP_{\pi}(B | s) = \int_{\A} \PMDP(B | s, a)\pi(\rmd a|s)\eqsp,
\end{equation}
which corresponds to the $1$-step transition probability from state $s$ to a set $B \in \borel{\S}$. We denote by $\mu$ the invariant distribution over the state space $\S$ induced by the transition kernel $\PMDP_{\pi}(\cdot | s)$ in \eqref{eq:transition_matrix_P_pi_appendix}. In this case the TD learning updates \eqref{eq:LSA_procedure_TD_appendix} correspond to the approximate solution of the deterministic system $\bA \thetas = \barb$, where we have set, respectively,
\begin{align}
\label{eq:system_matrix}
\bA &= \PE_{s \sim \mu, s' \sim \PMDP_{\pi}(\cdot|s)} [\varphi(s)\{\varphi(s)-\gamma \varphi(s')\}^{\top}] \\
\barb &= \PE_{s \sim \mu, a \sim \pi(\cdot|s)}[\varphi(s) r(s,a)]\eqsp.
\end{align}

\subsection{Proof of \Cref{prop:assumption_check_TD}}
We first need to check that the matrix $\bA + \bA^{\top}$, where $\bA$ is defined in \eqref{eq:system_matrix}, is positive-definite. In order to show this fact we closely follow the exposition of \cite[Lemma~18]{samsonov2023finite} and \cite[Lemma~5]{patil2023finite}. Define a random matrix $\funcAw$ as an independent copy of $\funcAw_{k}$ from \eqref{eq:matr_A_def_appendix}, that is, 
\[
\funcAw = \varphi(s)\{\varphi(s) - \gamma \varphi(s')\}^{\top}\eqsp,
\]
where $s \sim \mu$, and $s' \sim \PMDP_{\pi}(\cdot|s)$. With the definition of $\funcAw$, we get that
\begin{align}
\funcAw + \funcAw^{\top}
&= \varphi(s)\{\varphi(s) - \gamma \varphi(s')\}^{\top} + \{\varphi(s) - \gamma \varphi(s')\}\varphi(s)^{\top} \\
&= 2 \varphi(s)\varphi(s)^{\top} - \gamma \{\varphi(s)\varphi(s')^{\top} + \varphi(s')\varphi(s)^{\top}\} \label{eq:a_at_bound} \\
&\succeq (2-\gamma) \varphi(s)\varphi(s)^{\top} - \gamma \varphi(s')\varphi(s')^{\top} \eqsp,
\end{align}
where we used an elementary inequality $u v^{\top} + v u^{\top} \preceq (uu^{\top}+vv^{\top})$ valid for any $u,v \in \rset^{d}$. Hence, with the definition of $\covfeat$ in \eqref{eq:covfeat_matr_def}, we get 
\begin{equation}
\label{eq:a_plus_a_top_bound}
\bA + \bA^{\top} = \PE[\funcAw + \funcAw^{\top}] \succeq 2(1-\gamma)\covfeat\eqsp.
\end{equation}
Hence, $\bA + \bA^{\top}$ is positive-definite, and we can set $P = \bA + \bA^{\top}$ in the right-hand side of the Lyapunov equation \eqref{eq:Lyapunov_equation}. Obviously, $Q = \Id$ is a solution to the corresponding Lyapunov equation 
\[
\bA^\top Q + Q \bA = \bA + \bA^{\top}\eqsp.
\]
Moreover, applying \cite[Lemma~18]{samsonov2023finite}, we obtain
\begin{equation}
\label{eq:a_top_a_bound}
\bA^{\top}\bA \preceq \PE[\funcAw^{\top}\funcAw] \preceq (1+\gamma)^2 \covfeat\eqsp.
\end{equation}
Hence, we get for $\alpha \leq (1-\gamma)/(1+\gamma)^2$, and applying \eqref{eq:a_plus_a_top_bound} and \eqref{eq:a_top_a_bound}, that 
\begin{align}
(\Id - \alpha \bA)^{\top} (\Id - \alpha \bA) 
&= \Id - \alpha(\bA^{\top} + \bA) + \alpha^2 \bA^{\top} \bA \\
&\preceq \Id - 2\alpha(1-\gamma)\covfeat + \alpha^2(1+\gamma)^2\covfeat \\
&\preceq \Id - \alpha(1-\gamma)\covfeat \\
&\preceq (1 - \alpha (1-\gamma)\lambda_{\min}(\covfeat))\Id\eqsp.
\end{align}
Hence, the bound \eqref{eq:contractin_q_norm} holds with $a = (1-\gamma)\lambda_{\min}(\covfeat)$ and $\alpha_{\infty} = (1-\gamma)/(1+\gamma)^2$.

\section{Experimental details for the TD learning}
\label{appendix:numeric_details}
Here we provide some details on numerical experiments. Code to run experiments is provided in \url{https://github.com/svsamsonov/BootstrapLSA}. For the considered Garnet problem we choose the policy $\pi$ in the following way. For any $a \in \seta$, we set 
\[
\pi(a|s) = \frac{U_{a}^{(s)}}{\sum_{i=1}^{|\seta|}U_{i}^{(s)}}\eqsp,
\]
where the $U_{i}^{(s)}$ are independent random variables following uniform distribution $\mathcal{U}[0,1]$. Here we assume that each action $a \in \seta$ can be selected at any state $s \in \{1,\ldots,N_{s}\}$. We generate an instance of Garnet problem with mentioned parameters, and find analytically the true parameter $\thetas$. In order to estimate the supremum 
\begin{equation}
\label{eq:approx_supremum_experiment_appendix}
\textstyle 
\Delta_{n} := \sup_{x \in \rset} \bigl|\P(\sqrt{n}\norm{\bar{\theta}_{n} - \thetas} \leq x) - \P(\norm{\Sigma_{\infty}^{1/2}\eta} \leq x)\bigr|\eqsp,
\end{equation}
$\eta \sim \mathcal{N}(0,\Id_{N_s})$, and show that this supremum scales as $n^{-1/4}$ when $\gamma = 1/2$ and admits slower decay for other powers of $\gamma$. We first approximate true probability $\P(\norm{\Sigma_{\infty}^{1/2}\eta} \leq x)$ by the corresponding empirical probabilities based on sample of size $M \gg n$. We fix $M = 5 \cdot 10^{7}$. We choose trajectory lengths 
\[
n \in \{1600,3200,6400,12800,25600,51200,102400,204800,409600,819200, 1638400\}\eqsp,
\]
fix the length of burn-in period $n_0 = 102400$, and generate $N = 6553600$ independent trajectories starting in the fixed point $\theta_0 \in \rset^{N_s}$. We set the learning rate schedule as $\alpha_{k} = c_{0}/k^{\gamma}$ and try different values $\gamma \in \{0.5,0.65,0.7\}$, and $c_0 = 4.0$. Unfortunately, even the chosen order of trajectory length $n$ seems to be insufficient in order to significantly distinguish, for example, between $\gamma = 0.5$ and $\gamma = 0.65$. However, learning rate schedule with faster decay performs worse in terms of $\Delta_{n}$. Note that the current experiment is already rather computationally intense
for artificial problem and takes about 12 hours of compute on a Core i9 - 10920x processor with 12
cores with 3.7 GHz.

\newpage
\section*{NeurIPS Paper Checklist}
\begin{enumerate}

\item {\bf Claims}
    \item[] Question: Do the main claims made in the abstract and introduction accurately reflect the paper's contributions and scope?
    \item[] Answer: \answerYes{} 
    \item[] Justification: Main results are, respectively, the ones of \Cref{th:shao2022_berry} and \Cref{th:bootstrap_validity}, their statements are complete and supported by the proofs in the Appendix section.

\item {\bf Limitations}
    \item[] Question: Does the paper discuss the limitations of the work performed by the authors?
    \item[] Answer: \answerYes{} 
    \item[] Justification: We discuss the limitations of our setting related to the LSA problem, and not more general non-linear stochastic optimisation problems. We highlight the potential generalizations of \Cref{th:shao2022_berry} to the non-linear setting and discuss why generalizing \Cref{th:bootstrap_validity} might be more challenging. We also discuss the limitation related to i.i.d. observations.

\item {\bf Theory Assumptions and Proofs}
    \item[] Question: For each theoretical result, does the paper provide the full set of assumptions and a complete (and correct) proof?
    \item[] Answer: \answerYes{} 
    \item[] Justification: All our theoretical results are provided with references to assumptions, that are stated in \Cref{sec:independent_case} and \Cref{sec:bootstrap}. All results are given with proofs, that are correctly referenced for each theorem and corollary.

    \item {\bf Experimental Result Reproducibility}
    \item[] Question: Does the paper fully disclose all the information needed to reproduce the main experimental results of the paper to the extent that it affects the main claims and/or conclusions of the paper (regardless of whether the code and data are provided or not)?
    \item[] Answer: \answerYes{} 
    \item[] Justification: Numerical results are stated with a complete description of the environments that are used, as well as the precise sets of hyperparameters that we used. The code (in Python) is provided as supplementary with the paper, making it easy for one to reproduce our numerical experiments. At the same time, tracing the second-order terms in the normal approximation is computationally involved and can take sufficiently large amount of time.

\item {\bf Open access to data and code}
    \item[] Question: Does the paper provide open access to the data and code, with sufficient instructions to faithfully reproduce the main experimental results, as described in supplemental material?
    \item[] Answer: \answerYes{} 
    \item[] Justification: All code is open source, link to a github repository is included.

\item {\bf Experimental Setting/Details}
    \item[] Question: Does the paper specify all the training and test details (e.g., data splits, hyperparameters, how they were chosen, type of optimizer, etc.) necessary to understand the results?
    \item[] Answer: \answerYes{} 
    \item[] Justification: The algorithm used in the numerical experiments are exactly the algorithms described in the paper. The Garnet environements are given with the parameters used for generation, and with reference to the original problem. 

\item {\bf Experiment Statistical Significance}
    \item[] Question: Does the paper report error bars suitably and correctly defined or other appropriate information about the statistical significance of the experiments?
    \item[] Answer: \answerNo{} 
    \item[] Justification: Unfortunately, error bars for computing the second-order terms in normal approximation are quite computationally intense, moreover, tracing the terms of order $n^{1/4}$ requires quick increase of trajectory length $n$.

\item {\bf Experiments Compute Resources}
    \item[] Question: For each experiment, does the paper provide sufficient information on the computer resources (type of compute workers, memory, time of execution) needed to reproduce the experiments?
    \item[] Answer: \answerYes{} 
    \item[] Justification: All necessary information to reproduce experiments is provided in \Cref{appendix:numeric_details}.
    
\item {\bf Code Of Ethics}
    \item[] Question: Does the research conducted in the paper conform, in every respect, with the NeurIPS Code of Ethics \url{https://neurips.cc/public/EthicsGuidelines}?
    \item[] Answer: \answerYes{} 
    \item[] Justification: This paper is of purely theoretical nature, and the proposed methods do not deal with sensitive attributes that could induce unfairness or privacy issues.

\item {\bf Broader Impacts}
    \item[] Question: Does the paper discuss both potential positive societal impacts and negative societal impacts of the work performed?
    \item[] Answer: \answerNA{} 
    \item[] Justification: This paper is of purely theoretical nature. We do not foresee any societal harm from the proof of non-asymptotic bootstrap validity and normal approximation bounds in Kolmogorov distance.
    
\item {\bf Safeguards}
    \item[] Question: Does the paper describe safeguards that have been put in place for responsible release of data or models that have a high risk for misuse (e.g., pretrained language models, image generators, or scraped datasets)?
    \item[] Answer: \answerNA{} 
    \item[] Justification: Not applicable.

\item {\bf Licenses for existing assets}
    \item[] Question: Are the creators or original owners of assets (e.g., code, data, models), used in the paper, properly credited and are the license and terms of use explicitly mentioned and properly respected?
    \item[] Answer: \answerNA{} 
    \item[] Justification: Not applicable: no existing assets are used.

\item {\bf New Assets}
    \item[] Question: Are new assets introduced in the paper well documented and is the documentation provided alongside the assets?
    \item[] Answer: \answerNA{} 
    \item[] Justification: Not applicable: paper does not release new assets.

\item {\bf Crowdsourcing and Research with Human Subjects}
    \item[] Question: For crowdsourcing experiments and research with human subjects, does the paper include the full text of instructions given to participants and screenshots, if applicable, as well as details about compensation (if any)? 
    \item[] Answer: \answerNA{} 
    \item[] Justification: Not applicable: paper does not involve crowdsourcing nor research on human subjects. 

\item {\bf Institutional Review Board (IRB) Approvals or Equivalent for Research with Human Subjects}
    \item[] Question: Does the paper describe potential risks incurred by study participants, whether such risks were disclosed to the subjects, and whether Institutional Review Board (IRB) approvals (or an equivalent approval/review based on the requirements of your country or institution) were obtained?
    \item[] Answer: \answerNA{} 
    \item[] Justification: Not applicable: paper does not involve crowdsourcing nor research on human subjects. 

\end{enumerate}

\end{document}